\documentclass[11pt]{article}
\usepackage{amsfonts,amsbsy,amssymb,amsmath,amsthm}
\usepackage[toc,title,page]{appendix}
\usepackage{a4}
\usepackage{graphicx}
\usepackage[english]{babel}
\usepackage{datetime}
\usepackage[colorinlistoftodos]{todonotes}
\usepackage[numbers]{natbib}
\usepackage{varioref}
\usepackage{hyperref}
\usepackage{cleveref} 

\usepackage{algorithm}
\usepackage{algorithmic}
\usepackage[ruled,vlined, algo2e]{algorithm2e}
\usepackage{caption}
\usepackage{subcaption}

\usepackage{comment}
\usepackage{cancel}
\usepackage{chngcntr}

\newcommand{\<}{\langle}
\renewcommand{\>}{\rangle}

\newcommand{\R}{\mathbb{R}}
\newcommand{\E}{\mathbb{E}}

\renewcommand{\P}{\mathbb{P}}
\renewcommand{\L}{\mathcal{L}}
\newcommand{\what}{\widehat}

\newcommand{\pif}{\Pi_{K_0}(f^*)}

\newcommand{\supp}{\text{supp}}
\newcommand{\ptest}{\what{p}_{test}}
\newcommand{\qtest}{\what{q}_{test}}

\usepackage{bbm}

\newtheorem{theorem}{Theorem}[section]
\newtheorem{definition}[theorem]{Definition}

\newtheorem{lemma}[theorem]{Lemma}
\newtheorem{remark}[theorem]{Remark}
\newtheorem{assumption}[theorem]{Assumption}
\newtheorem{proposition}[theorem]{Proposition}
\newtheorem{corollary}[theorem]{Corollary}

\usepackage{xcolor}

\title{Training Guarantees of Neural Network Classification Two-Sample Tests by Kernel Analysis}

\author{Varun Khurana\footnote{Department of Applied Mathematics, Brown University, Providence, RI 02906.} 
\and Xiuyuan Cheng\footnote{Department of Mathematics, Duke University, Durham, NC, 27708.} 
\and Alexander Cloninger\footnote{Department of Mathematics and Halicio{\u g}lu Data Science Institute, University of California, San Diego, San Diego, CA, 92093.}
}

\date{}

\begin{document}

\maketitle

\begin{abstract}
We construct and analyze a neural network two-sample test to determine whether two datasets came from the same distribution (null hypothesis) or not (alternative hypothesis).  We perform time-analysis on a neural tangent kernel (NTK) two-sample test. In particular, we derive the theoretical minimum training time needed to ensure the NTK two-sample test detects a deviation-level between the datasets.  Similarly, we derive the theoretical maximum training time before the NTK two-sample test detects a deviation-level. By approximating the neural network dynamics with the NTK dynamics, we extend this time-analysis to the realistic neural network two-sample test generated from time-varying training dynamics and finite training samples.  A similar extension is done for the neural network two-sample test generated from time-varying training dynamics but trained on the population.  To give statistical guarantees, we show that the statistical power associated with the neural network two-sample test goes to 1 as the neural network training samples and test evaluation samples go to infinity.  Additionally, we prove that the training times needed to detect the same deviation-level in the null and alternative hypothesis scenarios are well-separated.  Finally, we run some experiments showcasing a two-layer neural network two-sample test on a hard two-sample test problem and plot a heatmap of the statistical power of the two-sample test in relation to training time and network complexity.
\end{abstract}

\section{Introduction}

The ability to compare whether two datasets $\what{P} \sim p$ and $\what{Q} \sim q$ came from the same data-generating process (i.e. checking if $p = q$ or $p \neq q$) is a problem studied for many years.  Traditionally, the methods to answer this question are called two-sample tests. As a non-exhaustive list of applications, two-sample testing is widely used in testing drug efficacy \cite{Dara2022}, studying behavioral differences in psychology \cite{bradley2020role}, pollution impact studied in environmental science research \cite{bucci2020known}, and market research impact studies \cite{burns2020marketing}. The most basic method to compare distributions is by comparing means with a $t$-test, proportions with a $z$-test, variances with Levene's test, medians with a Mann-Whitney U test, or overall distributions with a Kolmogorov-Smirnov test.  The advent of complex, high-dimensional data in fields like genomics, finance, and social media analytics has exposed limitations in these traditional methods, particularly in terms of handling non-linearity, complex interactions, and the curse of dimensionality.  The flexibility and scalability of neural networks make them particularly suited to tackle the challenges posed by modern datasets, suggesting their potential to revolutionize two-sample testing.  

This paper is not the first to explore the idea of using neural networks or classifiers for two-sample testing.  In particular, \cite{lopez2017revisiting} showed properties and analyzed the performance of the so-called Classifier Two-Sample Test (C2ST) and specifically showcasing theoretically what the statistical power of such two-sample tests. 
To go further in the neural network direction, \cite{xiuNTKMMD} expanded \cite{grettonKernel}'s work and used the neural tangent kernel (NTK) for the kernel involved in a maximum mean discrepancy (MMD) problem.  Yet their analysis still did not relate the NTK MMD performance to the behavior of neural network classification two-sample tests.  Moreover, \cite{Cheng_2022} introduced a neural network-based two sample test statistic using the classification logit and showed theoretical guarantees for test power for sub-exponential density problems.  One may be hesitant to use a neural network for two-sample tests since with a big enough neural network and long enough training time, a neural network could find a separation for data coming from the same distribution.  Our approach alleviates this hesitation since we train the neural network on a small time-scale and ensure our network is initialized to output 0 for all values.  We also conduct time analysis on two levels.  First, we analyze the time needed for achieving a desired level of deviation or detection in the two-sample test.  Second, we provide time approximations between different training regimes, which extends the analysis of the time needed for detection to different training regimes.

\subsection{Main Contributions}
Our main contributions to the field are the following:
\begin{enumerate}
    \item We perform some time analysis on the neural tangent kernel (NTK) derived from our neural network and show that the time it takes for the neural network two-sample test to learn does \textit{not} depend on the entire spectrum of the NTK but rather only a subset of the spectrum on which the labels or witness function $f^* = \frac{p - q}{p+q}$ non-trivially projects onto.  This behavior is a result of averaging behavior of the neural network two-sample test.

    \item We approximate the population-level neural network dynamics and finite-sample neural network dynamics with the population-level NTK dynamics.  This allows the time analysis performed on the NTK dynamics to transfer to the other two training regimes.  Additionally, we notice here that there is a balancing act of not training the neural network too long so that the the approximations hold but long enough to detect differences in the datasets.  This balancing act is further informed by the complexity of the neural network considered in relation to the difficulty of the two-sample test problem.
\end{enumerate}

Our main result essentially shows that as long as $p$ and $q$ are ``separated enough", our neural network two-sample test can detect the difference before the same detection would take place if $p = q$.  This idea comes about as test power analysis and test time analysis.  In particular, we can summarize the main result of value as the following two informal theorems.  Here, let $\what{T}$ denote the neural network two-sample test.

\begin{theorem}[Informal Test Power Analysis]
    For particular type 1 error
    \begin{align*}
        \P( \what{T} > \tau \vert H_0) = \alpha,
    \end{align*}
    where $\alpha$ is the power level, the statistical power goes to 1, i.e. we have
    \begin{align*}
        \P \Big( \what{T} > \tau \big\vert \substack{ f^* \text{nontrivially projects onto first $k$} \\ \text{eigenfunctions of zero-time NTK $K_0$}} \Big) \to 1,
    \end{align*}
    as the training and test samples sizes go to $\infty$.
\end{theorem}

\begin{theorem}[Informal Test Time Analysis]
    Assume that $f^*$ nontrivially projects onto only the first $k$ eigenfunctions of the zero-time NTK $K_0$.  Consider a desired detection level $\epsilon > 0$ and time separation level $C \epsilon \geq \gamma > 0$. Further assume that the projection of $f^*$ onto the first $k$ eigenfunctions has a ``large enough norm/energy."  Then with high probability,
    \begin{align*}
        t^+(\epsilon) - t^-(\epsilon) \geq \gamma > 0,
    \end{align*}
    where $t^-(\epsilon)$ and $t^+(\epsilon)$ are the minimum times needed for the neural network two-sample test to detect a deviation $\epsilon$ when $f^*$ projects onto the first $k$ eigenfuctions and the null hypothesis, respectively.
\end{theorem}

In the formal versions of these theorem (as shown in \Cref{cor:genMomPower_uhat} and \Cref{cor:genMomRes_uhat}, respectively), the detection level $\epsilon > 0$ possible is perturbed by a time-approximation error between the actual neural network two-sample test and the zero-time NTK two-sample test.  This adds a small amount of complexity to the informal theorem above and there are lower bound conditions on $f^*$ to ensure detectability.  We also discuss (in a subsequent remark to \Cref{cor:genMomRes_uhat}) which detection level is the most trustworthy. A visual for this graph is given in \Cref{fig:twoSampleGraph}.

\begin{figure}[h!]
    \centering
    \includegraphics[scale=0.5]{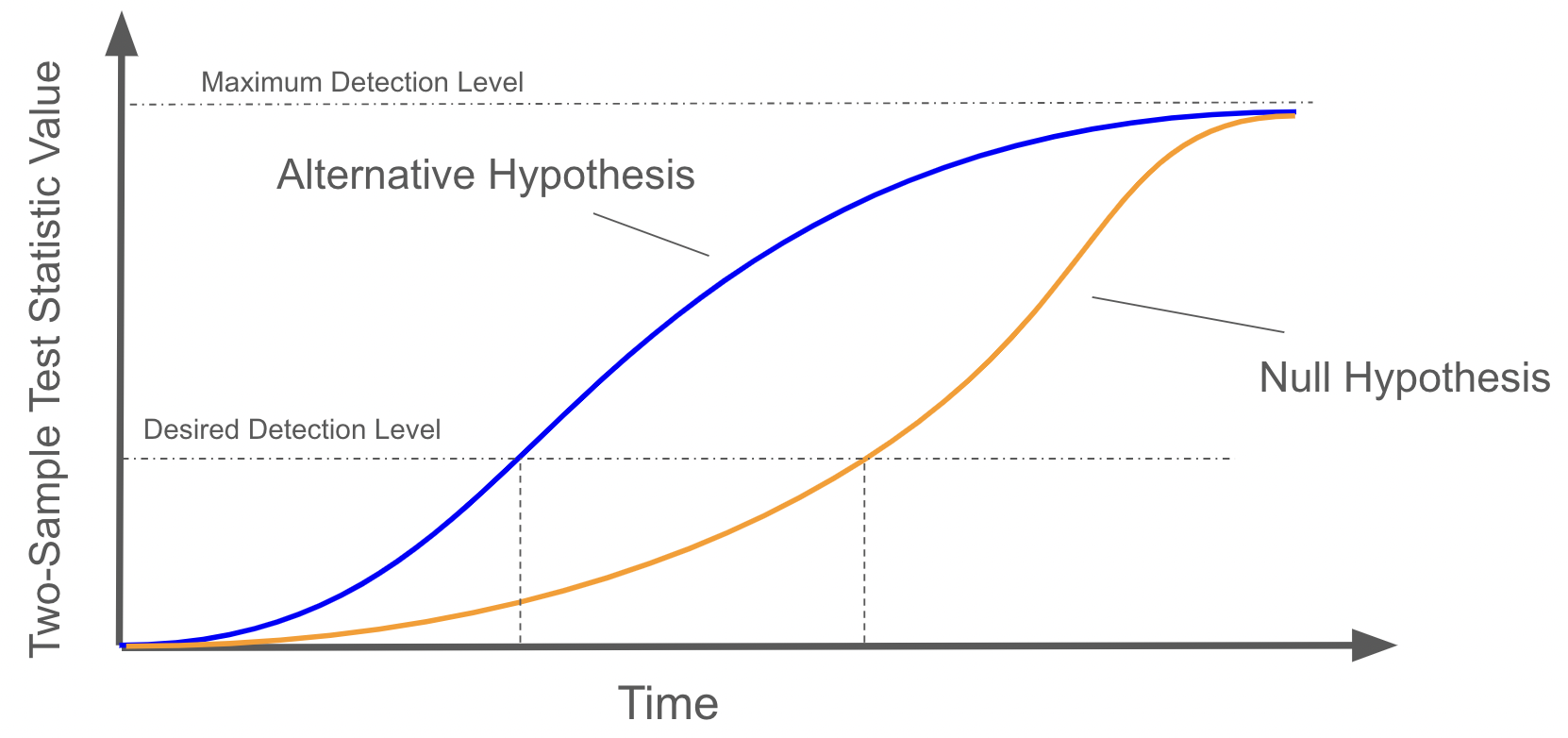}
    \caption{Visual for detection levels $t^+(\epsilon)$ and $t^-(\epsilon)$ being well-separated.}
    \label{fig:twoSampleGraph}
\end{figure}

The efficacy of neural network-based two-sample tests hinges on the dynamics of training and network complexity. A key insight is that as soon as gradient descent begins, the two-sample test becomes significantly more effective. To illustrate this, we work out a toy example in \Cref{sec:motivatingEx} where we distinguish between two multivariate normal distributions with identical covariance matrices but different means. Utilizing a simple linear neural network, we show that \textit{even} a single gradient descent step can markedly improve the test's ability to detect differences in distributions. This improvement is rooted in how the network parameters adjust to reflect the underlying statistical properties of the data, thereby enhancing the test's sensitivity. In the broader analysis done in later chapters, we show the theory that balances small training time scales with network complexity to ensure our two-sample test works well.

\subsection{Structure of the paper}

We review some papers in \Cref{sec:prevWorks} that study the same topic that we explore in this paper.  We discuss specifics of how the training time-scale interplays with the network complexity in \Cref{sec:balance}.  In \Cref{sec:background}, we introduce the main notation, a motivating example, and concepts that we will need for the rest of the paper.

In \Cref{sec:trainReg}, we describe in detail three training dynamics that we consider for the two-sample test.  First, we consider the case of realistic (finite-sample time-varying) dynamics; second, we consider population (population-level time-varying) dynamics; and finally, we consider the zero-time neural tangent kernel (NTK) auxiliary dynamics, where the analysis is easier to understand. The population dynamics offer useful insight into the limit of the realistic dynamics as the number of samples increases.  We emphasize that the results for the realistic and population dynamics scenarios are similar, but the realistic dynamics require more care to work out the analogous results as the population dynamics. For the zero-time NTK auxiliary dynamics, we conduct a full analysis in \Cref{subsec:pop0}. By solving for the actual solution of the zero-time NTK dynamics, we are able to find an exact form for the two-sample test in this regime by using the spectrum of the NTK.  The exact form of the two-sample test further allows us to conduct some time analysis.  The time analysis is done to show guarantees for when the null hypothesis is correct or when the alternative hypothesis is correct.  For the alternative hypothesis, we can estimate the minimum time needed for sensing an error level $\epsilon > 0$.  For the null hypothesis statement, we can estimate the maximum time needed before we are able to sense past an error level $\epsilon > 0$.  Using proof techniques similar to \cite{repasky2023neural}, these time-analysis results are adapted later to the other training settings by using approximation and estimation between the different settings.

\Cref{sec:approx} estimates both the realistic and population dynamics with the zero-time NTK dynamics.  The approximation guarantees between the realistic and zero-time NTK dynamics hold up to a factor of $t^{5/2}$ and depends on how many data points are sampled from both $p$ and $q$ whilst the population dynamics are approximated by the zero-time NTK dynamics up to a factor of $t^{3/2}$, where $t$ denotes time.  We additionally approximate the two-sample statistics in the realistic and population dynamics scenarios with the zero-time NTK two-sample statistic.  Finally, we show the two-sample statistic in the realistic and populations dynamics case are monotonically increasing under mild assumptions.

In \Cref{sec:timepower}, we extend the time analysis done for the zero-time NTK dynamics situation to the realistic dynamics (\Cref{sec:timeReal}) and population dynamics (\Cref{sec:timePop}) cases.  These extensions depends heavily on the results of \Cref{sec:approx}.  For the reader wishing to delve into the proofs, the actual proofs of the realistic dynamics and population dynamics cases are written in sequential progression in \Cref{apdx:real} and \Cref{apdx:pop}.  The process of extending the time analysis requires using the approximations from \Cref{sec:approx} together with the time analysis theorems in \Cref{subsec:pop0}.  In particular, if the zero-time NTK dynamics two-sample test are able to detect $f^*$ faster than the approximation guarantees of \Cref{sec:approx}, then all the time analysis for the NTK auxiliary dynamics also holds for the realistic and population two-sample tests.

The results shown in \Cref{sec:timeReal} and \Cref{sec:timePop} showcase that our neural network two-sample test is more useful in identifying when the alternative hypothesis is correct.  Moreover, we see a sort of balancing act of training on short time scales and increasing the complexity of the neural network.  These specifics are discussed in more detail in \Cref{sec:balance}.  Finally, in \Cref{sec:experiments} we show empirical evidence of the statistical power of the neural network two-sample test on a hard two-sample test problem.

\section{Previous Works and Our Work}\label{sec:prevWorks}

\subsection{Previous Works}

\paragraph{Classification and two-sample testing.}  As mentioned, classifier two-sample tests are not new \cite{grettonKernel, 
lopez2017revisiting,friedman2004}.  These approaches demonstrated that a well-trained classifier could effectively identify distributional discrepancies. More traditional two-sample methods use methods such as kernel two-sample test and maximum mean discrepancy (MMD) \cite{grettonKernel, cheng2020two, schrab2023mmd, kubler2022witness} but adhere to the classification two-sample test spirit. More recently, there has been growing interest in developing two-sample tests based on neural networks, leveraging the power of deep learning to address some of the limitations of classical two-sample tests \cite{liu2020learning, Cheng_2022, kubler2022automl, chwialkowski2015fasttwosampletestinganalytic}. \cite{binkowski2021demystifyingmmdgans} applied MMD to generative adversarial netowrks (GANs) for imporoved distribution comparison.  Kernel MMD was also extended to the neural tangent kernel (NTK) for two-sample testing by \cite{xiuNTKMMD}. Variations of a neural network-based two-sample test are present in \cite{kirchler2020twosample} and \cite{Cheng_2022}.  The analysis done in this paper goes further by using small time approximations between the NTK-based kernel machines and the actual neural network training dynamics.  To accomplish this, we use very similar proof techniques to \cite{repasky2023neural}, but rather than using a loss rescaling to get into the lazy training regime \cite{chizat2020lazy}, we are able to use small time approximations for the NTK.

\paragraph{Neural Tangent Kernels.}  We heavily use the neural tangent kernel (NTK) in this paper for tractable analysis which is extended to realistic dynamics for the actual trained neural network. The concept of NTKs comes from the seminal work of \cite{jacot2020neuraltangentkernelconvergence}, who demonstrated that training dynamics of a neural network can be approximated by a kernel method, where the kernel is the NTK.  Many studies have extended and refined NTK theory since the seminal work.  \cite{arora2019exactcomputationinfinitelywide} extended the NTK analysis to convolutional neural networks and showed how NTKs evolve during training.  Other works \cite{Lee_2020, du2019gradientdescentprovablyoptimizes, allenzhu2019convergencetheorydeeplearning, mei2021generalizationerrorrandomfeatures, chizat2020lazy} have studied the hierarchical structure of NTKs, convergence guarantees for training over-parameterized neural networks, and implicit regularization effects of NTKs.  Many of these works simplify the NTK with either a large-width assumption or a small-time assumption where the NTK dynamics and neural network dynamics are similar.  For this paper, our neural network must be large enough to capture the target function $f^*$ but small enough to ensure small-time approximation holds between the NTK and the true neural network.

\subsection{Difference of Our Work from Previous NTK Literature}\label{sec:balance}

Our NTK analysis introduces a novel perspective by balancing time scales and network complexity, distinguishing it from traditional NTK analysis. Specifically, we focus on the small-time regime to assess the two-sample test, which aims to identify when the alternative hypothesis is correct. Our approach relates the training time to the complexity of the zero-time NTK and the projection of the weighted difference of densities $\frac{p-q}{p+q}(x)$ onto the zero-time NTK's eigenfunctions. We analyze both population-level and finite-sample dynamics in relation to the zero-time NTK, ensuring our approximations hold within this short-time regime. This delicate balance depends on the problem's complexity and the neural network's complexity.

Consider the scenario where the densities $p$ and $q$ are significantly different. If the zero-time NTK's larger eigenvalue functions correspond to low-frequency eigenfunctions, the two-sample test can produce positive results quickly since $\frac{p-q}{p+q}$ will likely project onto low-frequency eigenfunctions. Conversely, when $p$ and $q$ are similar, $\frac{p-q}{p+q}$ projects onto higher frequency eigenfunctions, requiring more time or a larger neural network for detection.

Moreover, balancing short-time and long-time scales necessitates managing neural network complexity. The complexity must be sufficient to capture $\frac{p-q}{p+q}$ on the NTK's eigenbasis, but the estimation accuracy of the zero-time NTK decreases with larger networks. Fortunately, the approximation depends on the product of neural network complexity and time, which allows for good detection even with small time scales, counteracting approximation losses.

Our empirical results span neural networks with various parameter-to-sample ratios, from severely under-parameterized to highly over-parameterized regimes, demonstrating the robustness of our NTK analysis in different settings.

\section{Notation and Background}\label{sec:background}

We will study a neural network two-sample test, which will test whether two datasets came from the same distribution or not. In particular, assume that we are given datasets $X = \{ x_i \}_{i=1}^{n_p} \subseteq \R^d$ and $Z = \{ z_j \}_{j=1}^{n_q} \subseteq \R^d$.  We will endow samples from $X$ to have labels $1$ whilst samples from $Z$ will have labels $-1$.  To give some more structure to our problem, we will moreover assume that the datasets $X$ and $Z$ are sampled from distributions $p(x)dx$ and $q(x)dx$, respectively, where $p$ and $q$ are associated density functions.  From $X$ and $Z$, note that we can construct finite-sample empirical measures 
\begin{align*}
    \what{p}(x) dx = \frac{1}{n_p} \sum_{i=1}^{n_p} \delta_{x_i}(x) dx \\
    \what{q}(x) dx = \frac{1}{n_q} \sum_{j=1}^{n_q} \delta_{z_j}(x) dx
\end{align*}
respectively.  In the same fashion, we can assume that we are given \textbf{independent} test samples from each of $p$ and $q$ to generate $X_{test} = \{x_i^*\}_{i=1}^{m_p}$ and $Z_{test} = \{z_j\}_{j=1}^{m_q}$ as well as corresponding \textit{test} empirical measures $\ptest(x) dx$ and $\qtest(x) dx$.  These test sets will be used when considering the finite-sample two-sample test on test data.  We now introduce the following notation
\begin{align*}
    \Vert f \Vert_{L^2(p+q)} = \bigg( \int_{\R^d} \vert f(x) \vert^2 (p(x) + q(x))dx \bigg)^{1/2} \\
    \Vert f \Vert_{L^2(\what{p}+\what{q})} = \bigg( \int_{\R^d} \vert f(x) \vert^2 (\what{p}(x) + \what{q}(x))dx \bigg)^{1/2}.
\end{align*}

Assume that our neural network architecture has associated parameters space $\Theta \subseteq \R^{M_\Theta}$ so that our neural network is given as $f : \R^d \times \Theta \to \R$ and will be trained on an $\ell_2$ loss function against the labels as shown here
\begin{align*}
    \what{L}(\theta) &= \frac{1}{2} \bigg( \frac{1}{n_p} \sum_{i=1}^{n_p} \big( f(x_i, \theta) - 1 \big)^2 + \frac{1}{n_q} \sum_{j=1}^{n_q} \big( f(z_j, \theta) + 1 \big)^2 \bigg) \\
    &= \frac{1}{2} \bigg( \int_{\R^d} \big( f(x,\theta) - 1 \big)^2 \what{p}(x) dx + \int_{\R^d} \big( f(x,\theta) + 1 \big)^2 \what{q}(x) dx \bigg) .
\end{align*}
As a precursor to the more concrete notation introduced in \Cref{sec:trainReg}, we will use the general rule of thumb of distinguishing mathematical objects in different training dynamics by:
\begin{enumerate}
    \item \textit{Realistic} (finite-sample time-varying) mathematical objects are adorned with hats, such as $\what{u}$.
    
    \item \textit{Population} (population-level time-varying) mathematical objects are not adorned with any specific notation, such as $u$.

    \item \textit{Auxillary NTK} (population-level zero-time NTK) mathematical objects are adorned with bars, such as $\overline{u}$.
\end{enumerate}

The notation for these different training dynamics used in \Cref{sec:trainReg} is represented in \Cref{tab:notation}.

\begin{table}[h!]
    \begin{tabular}{|c|c|c|c|}
    \hline Notation & \textbf{Realistic} & \textbf{Population} & \textbf{Auxillary NTK} \\
    \hline
    \textbf{Params} & $\what{\theta}(t)$ & $\theta(t)$ & $\theta_0 = \theta(0)$ \\
    \hline \textbf{Output}  & $\what{u}(x,t)=f(x,\what{\theta}(t))$ & $u(x,t)=f(x,\theta(t))$ & $\overline{u}(x,t)$ \\
    \hline 
    \textbf{Statistic} & $\what{T}_{train}(t), \what{T}_{test}(t), \what{T}_{pop}(t)$ & $T(t)$ & $\overline{T}(t)$ \\ \hline
    \textbf{Kernel} & $\what{K}_t(x,x')$ & $K_t(x,x')$ & $K_0(x,x')$ \\ \hline \textbf{Error} & $\substack{\what{e}_p(x,t) = \what{u}(x,t) - 1 \\ \what{e}_q(x,t) = \what{u}(x,t) + 1}$ &  $e = u - f^*$ & $\overline{e} = \overline{u} - \Pi_{K_0}(f^*)$ \\ \hline
    \end{tabular}
    \caption{Overview of notation used for three training dynamics.}
    \label{tab:notation}
\end{table}

\subsection{Motivating Example}\label{sec:motivatingEx}

For our motivating example two-sample test scenario, we consider when our probability distributions of interest are two multivariate normals with the same covariance matrix but different means.  In particular, with a fixed covariance matrix $\Sigma$, we let $p \sim N( \mu_1, \Sigma)$ and $q \sim N(\mu_2, \Sigma)$ with labels $1$ and $-1$ respectively.  Assume that we work with a linear neural network given by 
\begin{align*}
    f(x; a, W, b) = \frac{1}{M_\Theta} a^\top \Big( W x + b \Big) 
\end{align*}
where $x\in \R^{d}, a \in \R^{M_\Theta}$, $W \in \R^{M_\Theta \times d}$, and $b \in \R^{M_\Theta}$.  For ease assume that $M_\Theta$ is even, then for initialization, let $b = 0$ and opt to make $a_i = 1$ for $i \leq M_\Theta/2$ and $a_i = -1$ otherwise.  For $W$, we will generate a random matrix $\widetilde{W} \in \R^{(M_\Theta / 2) \times d}$ and let $W = \begin{bmatrix} \widetilde{W} \\ - \widetilde{W} \end{bmatrix}$.  These choices will ensure that $f(x) = 0$ for all $x$.

Recall the gradient of $f$ with respect to its parameters is given by
\begin{align*}
    \frac{\partial f}{ \partial a }(x) = \frac{1}{M_\Theta} (W x + b) \\
    \frac{\partial f}{ \partial W}(x) = \frac{1}{M_\Theta} a  x^\top  \\
    \frac{\partial f}{ \partial b}(x) = \frac{1}{M_\Theta} a.
\end{align*}
We will show that just one population-level gradient descent step with this setup will allow the two-sample test to detect the difference in distributions with high probability.  In particular, recall that with our initialization
\begin{align*}
    \frac{\partial \what{L}(\theta_0) }{ \partial f }(x) = \begin{cases}
        -1 & x \sim p \\
        1 & x \sim q
    \end{cases}.
\end{align*}
Now with learning rate $\eta$, one gradient descent step gives us
\begin{align*}
    a^{(1)} &= a - \eta \int \frac{\partial f }{ \partial a }(x) \frac{\partial \what{L} }{ \partial f }(x) d (p + q)(x) \\
    &= a - \eta \int \frac{1}{M_\Theta}  W x d(q - p)(x) = a - \frac{\eta}{M_\Theta} W ( \mu_2 - \mu_1) \\
    W^{(1)} &= W - \eta \int \frac{\partial f }{ \partial W }(x) \frac{\partial \what{L} }{ \partial f }(x) d (p + q)(x) \\
    &= W  - \eta \int \frac{1}{M_\Theta}  a x^\top d(q - p)(x) = W - \frac{\eta}{M_\Theta} a ( \mu_2 - \mu_1)^\top \\
    b^{(1)} &= b - \eta \int \frac{\partial f }{ \partial b }(x) \frac{\partial \what{L} }{ \partial f }(x) d (p + q)(x) \\
    &= b  - \eta \int \frac{1}{M_\Theta}  a d(q - p)(x) = b - 0 = 0.
\end{align*}
This means that after the first gradient descent step, we have
\begin{align*}
    f(x; a^{(1)}, W^{(1)}, 0) &= \bigg( a - \frac{\eta}{M_\Theta} W ( \mu_2 - \mu_1) \bigg)^\top \Bigg( \bigg( W - \frac{\eta}{M_\Theta} a ( \mu_2 - \mu_1)^\top \bigg) x \Bigg) \\
    &= \frac{\eta}{M_\Theta} \bigg( \Vert a \Vert^2 \< \mu_1 - \mu_2, x\> + \< W (\mu_1 - \mu_2), W x \> \\
    &+ \frac{\eta}{M_\Theta} \< a, W (\mu_2 - \mu_1) \> \<\mu_2 - \mu_1, x \> \bigg).
\end{align*}
Now notice that if we consider the two-sample test
\begin{align*}
    \int f(x; a^{(1)}, W^{(1)}, 0) d(p-q)(x) &= \frac{\eta}{M_\Theta} \bigg( \Vert a \Vert^2 \Vert \mu_1 - \mu_2 \Vert^2  + \Vert W (\mu_1 - \mu_2) \Vert \\
    &+ \frac{\eta}{M_\Theta} \< a, W (\mu_1 - \mu_2) \> \Vert \mu_1 - \mu_2 \Vert \bigg).
\end{align*}
In essence, if $\eta$ is small enough the two-sample test will be positive and the farther $\mu_1$ is away from $\mu_2$, the easier it becomes to detect.

In the case that we have $W$ fixed and $M_\Theta$ is large, we can see that $a$ is trying to learn $W(\mu_1 - \mu_2)$.  From a qualitative point of view, we only really need one row $w$ of $W$ to form a hyperplane that separates $\mu_1$ and $\mu_2$, assuming that $0, \mu_1, \mu_2$ do not all fall on the same line (and $\mu_1$ and $\mu_2$ are not on opposite sides of 0).  Moreover, the larger we pick $M_\Theta$, the random matrix $W$ gets a greater probability of producing such a row $w$.  Moreover, producing such a $w$ becomes increasingly more likely when we center the data so that the origin is between the two means.

\subsection{Relating Finite-sample and Population-level Loss}

We now revert to the more general case of neural networks considered and recall the form of $\what{L}(\theta)$. Notice that when $n_p, n_q \to \infty$, we get a population-level loss given by
\begin{align*}
    L(\theta) &= \frac{1}{2} \bigg( \int_{\R^d} \big( f(x,\theta) - 1 \big)^2 p(x) dx + \int_{\R^d} \big( f(x,\theta) + 1 \big)^2 q(x) dx \bigg) \\
    &=  \frac{1}{2} \bigg( \int_{\R^d} \big( f(x,\theta)^2 p(x) - 2 f(x,\theta) p(x) + p(x) \\
    &+ f(x,\theta)^2 q(x) + 2 f(x,\theta) q(x) + q(x) \big) dx \bigg) \\
    &= \frac{1}{2} \bigg( \int_{\R^d} \big( f(x,\theta)^2  - 2 f(x,\theta) \underbrace{\frac{p(x)-q(x)}{p(x) + q(x)}}_{f^*((x)} + 1 \big) (p(x) + q(x))dx \bigg).
\end{align*}
Here we can notice that
\begin{align*}
    \Vert f(\cdot, \theta) - f^*(\cdot) \Vert_{L^2(p+q)}^2 = \int_{\R^d} \big( f(x,\theta)^2 - 2 f(x,\theta) f^*(x) + (f^* (x))^2 \big) (p(x) + q(x)) dx.
\end{align*}
If we add the constant
\begin{align*}
    C = \frac{1}{2} \int_{\R^d} 4 \frac{p(x) q(x) }{p(x) + q(x)} dx,
\end{align*}
we get that
\begin{align*}
    L(\theta) = \frac{1}{2} \Vert f(\cdot, \theta) - f^*(\cdot) \Vert_{L^2(p+q)}^2 + C.
\end{align*}
To see this, notice that
\begin{align*}
    \frac{1}{2} \int_{\R^d} \Big( ( f^*(x) )^2 (p(x) + q(x)) +  \frac{4 p(x) q(x)}{p(x)+ q(x)} \Big) dx \\
    = \frac{1}{2} \int_{\R^d} \Big( \frac{(p(x) - q(x))^2}{(p(x) + q(x))^2} (p(x) + q(x)) +  \frac{4 p(x) q(x)}{p(x)+ q(x)} \Big) dx \\
    = \frac{1}{2} \int_{\R^d} \Big( \frac{p(x)^2  - 2p(x) q(x) + q(x)^2 + 4 p(x) q(x) }{p(x) + q(x)}\Big) dx \\
    = \frac{1}{2} \int_{\R^d} \Big( \frac{( p(x) + q(x))^2 }{p(x) + q(x)}\Big) dx = \frac{1}{2} \int_{\R^d} p(x) + q(x) dx.
\end{align*}
This means that minimizing $L(\theta)$ is the same as minimizing $\Vert f - f^* \Vert_{L^2(p+q)}^2$ as the constant doesn't depend on $\theta$.  Importantly, this means that our target function in the population-level training regimes will be
\begin{align*}
    f^*(x) := \frac{p(x) - q(x)}{p(x)+q(x)}.
\end{align*}

\subsection{Two-Sample Test}

Given probability densities $p$ and $q$, the two-sample test assesses whether to accept the null hypothesis $H_0$ or reject it for $H_1$, where
\begin{align*}
    H_0 : p = q , \hspace{3cm} H_1 : p \neq q.
\end{align*}
In words, our test is constructed using the average output of the neural network on measure $p$ minus the average output of the neural network on measure $q$.  This will give either population-level two-sample tests or finite-sample two-sample tests on the datasets $X_{test}$ and $Z_{test}$.  In particular, for the population-level statistic, we can define
\begin{align*}
    \mu_{p}(\theta) = \int_{\R^d} f(x,\theta) dp(x), \hspace{0.3cm} \mu_{q}(\theta) = \int_{\R^d} f(x,\theta) dq(x) \\
    T(\theta; p, q) = (\mu_{p}(\theta) - \mu_{q}(\theta)) = \int_{\R^d} f(x,\theta) d(p-q)(x);
\end{align*}
whereas, for the finite-sample statistic on \textit{test} data, we can define
\begin{align*}
    \mu_{\ptest}(\theta) = \int_{\R^d} f(x,\theta) d\ptest(x), \hspace{0.3cm} \mu_{\qtest}(\theta) = \int_{\R^d} f(x,\theta) d\qtest(x) \\
    T(\theta; \ptest,\qtest) = (\mu_{\ptest}(\theta) - \mu_{\qtest}(\theta)) = \int_{\R^d} f(x,\theta) d(\ptest-\qtest)(x).
\end{align*}
Here, we define the neural network two-sample test for a neural network $f(\cdot, \theta)$ by $T(\theta; \ptest, \qtest)$.  Given a test threshold $\tau > 0$, we reject the null hypothesis if $\vert T(\theta; \ptest, \qtest) \vert > \tau$.  Moreover, we control the false discovery of the null by finding the smallest $\tau$ such that $\text{Pr}[ \vert T(\theta; \ptest, \qtest) \vert  > \tau \vert H_0 ] \leq \alpha $, where $0 < \alpha < 1$ is the significance level.  To find $\tau$, we use a permutation test.

In \Cref{sec:trainReg}, we will consider different training regimes and each of these training regimes will have different notions of the two-sample test, which change by what the output of the neural network is and which probability measures the two-sample test statistic is computed on.  Particularly, the training regime with the zero-time NTK will end up using not the neural network by the function that is trained under zero-time NTK dynamics.  The specific notation regarding the two-sample test will be discussed there.

\section{The Three Training Dynamics}\label{sec:trainReg}

We will consider the following three different training dynamics for our neural network.  For all of scenarios, however, we assume the following.
\begin{assumption}
    The neural network is initialized with parameters $\theta_0$ such that $f(x,\theta_0) = 0$ for all $x \in \R^d$.
\end{assumption}

For the benefit of the reader, \Cref{fig:roadmap} provides a roadmap of how the test time and power theorems are proven and the dependencies inherent for the results.  Importantly, \Cref{fig:roadmap} explains that most of the approximated results for the realistic and population dynamics depend on the exact results of the zero-time NTK auxiliary dynamics.  The longest and most technical result of the roadmap is \Cref{prop:uhatubarBND_informal}.

\begin{figure}[t]
    \centering
    \includegraphics[scale=0.42]{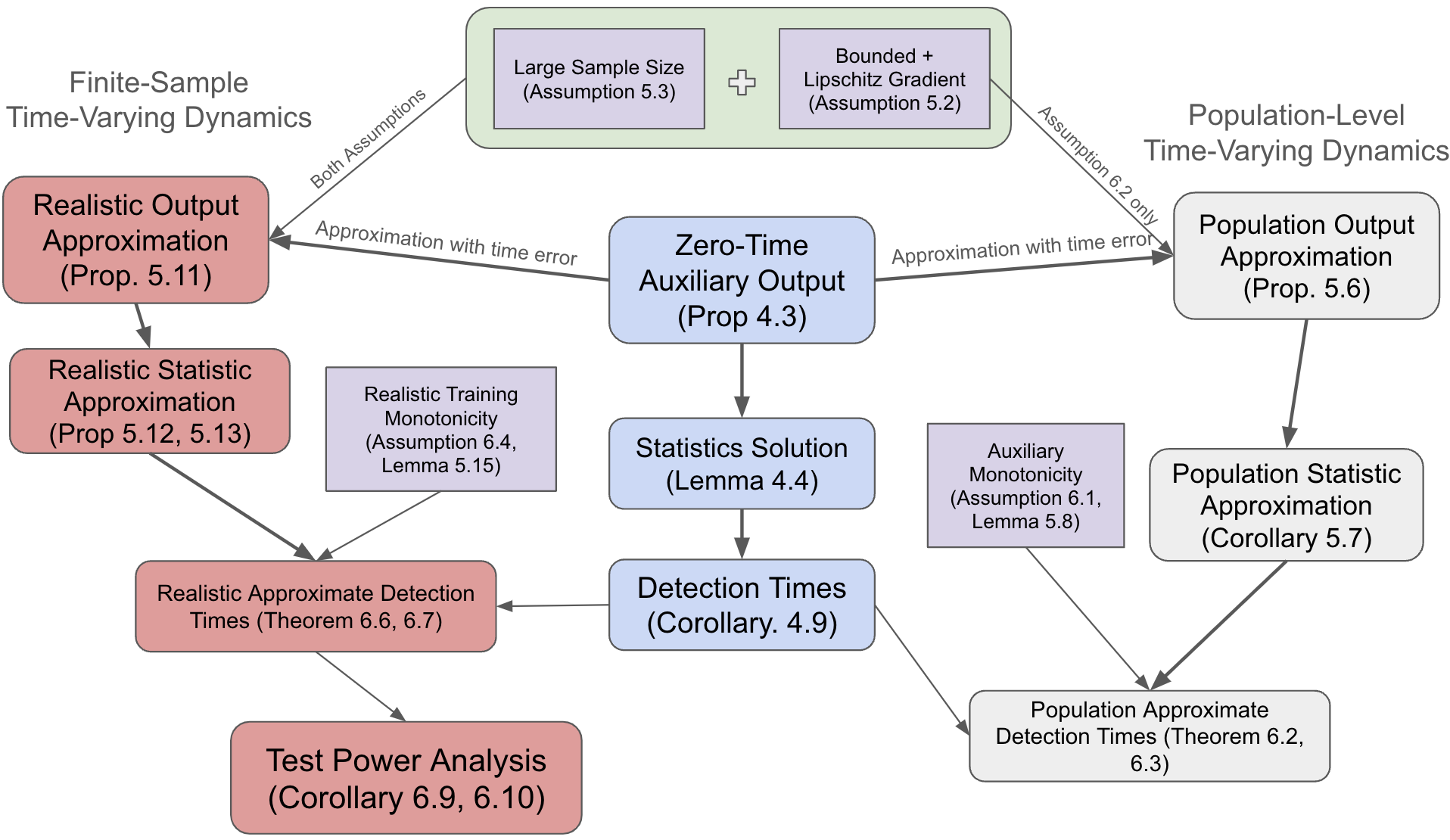}
    \caption{Roadmap of Theory and Results. Blue denotes zero-time NTK auxiliary dynamics, gray denotes population dynamics, and red denotes realistic dynamics.}
    \label{fig:roadmap}
\end{figure}

\subsection{Realistic Dynamics}

These finite-sample time-varying dynamics are the most realistic as these are the dynamics that will arise in practice.  Here, we tend to denote all the associated quantities with a hat.  Since we use gradient descent to optimize, we denote the parameters trained from finite-sample data as $\what{\theta}(t)$ and will denote the associated neural network's output as $\what{u}(x,t) = f(x,\what{\theta}(t))$.  Now, let us inspect the equation used to optimize the parameters.
\begin{align*}
    \what{L}(\theta) &= \frac{1}{2} \bigg( \int_{\R^d} \big( f(x,\theta) - 1 \big)^2 \what{p}(x) dx + \int_{\R^d} \big( f(x,\theta) + 1 \big)^2 \what{q}(x) dx \bigg) \\
    - \Dot{\what{\theta}}(t) &= \partial_\theta \what{L}(\what{\theta}(t)) \\
    &= \frac{1}{2} \bigg( \int_{\R^d}  \nabla_\theta f(x,\what{\theta}(t)) \Big( \big( f(x,\what{\theta}(t)) - 1 \big) \what{p}(x) + \big( f(x,\what{\theta}(t)) + 1 \big) \what{q}(x) \Big) dx \bigg) \\
    \partial_t \what{u}(x,t) &= \< \nabla_{\theta} f(x,\what{\theta}(t)), \Dot{\what{\theta}} \>_{\Theta} = - \< \nabla_{\theta} f(x,\what{\theta}(t)), \partial_\theta \what{L}(\what{\theta}(t)) \>_{\Theta} \\
    &= - \frac{1}{2} \bigg( \int_{\R^d}  \< \nabla_\theta f(x,\what{\theta}(t)), \nabla_\theta f(x',\what{\theta}(t))\>_{\Theta} \Big( \big( f(x',\what{\theta}(t)) - 1 \big) \what{p}(x') \\
    &+ \big( f(x',\what{\theta}(t)) + 1 \big) \what{q}(x') \Big) dx' \bigg).
\end{align*}
We define the time-varying finite-sample neural tangent kernel by
\begin{align*}
    \what{K}_t(x,x') = \< \nabla_\theta f(x,\what{\theta}(t)), \nabla_\theta f(x',\what{\theta}(t))\>_{\Theta}.
\end{align*}
We can further define density-specific residuals
\begin{align*}
    \what{e}_p(x,t) = \big( f(x,\what{\theta}(t)) - 1 \big) \\
    \what{e}_q(x,t) = \big( f(x,\what{\theta}(t)) + 1 \big).
\end{align*}
This means that
\begin{align*}
    \partial_t \what{u}(x,t) &= - \frac{1}{2} \bigg( \int_{\R^d}  \what{K}_t(x,x') \Big( \big( f(x',\what{\theta}(t)) - 1 \big) \what{p}(x') + \big( f(x',\what{\theta}(t)) + 1 \big) \what{q}(x') \Big) dx' \bigg) \\
    &= - \frac{1}{2} \Big( \E_{x' \sim \what{p} } \what{K}_t(x,x') \what{e}_p(x',t) + \E_{x' \sim \what{q} }  \what{K}_t(x,x') \what{e}_q(x',t)  \Big) .
\end{align*}

In the context of the two sample test, since this training regime uses only finite training samples, we will study the two-sample test statistic's behavior evaluated on the training samples, the test samples (\textit{independent} from the training samples), and the population. We denote these different evaluated test statistics by
\begin{align*}
    \what{T}_{train}(t) :&=  T(\what{\theta}(t) ; \what{p}, \what{q} ) = \Big( \E_{x \sim \what{p}} - \E_{x \sim \what{q} } \Big) f(x,\what{\theta}(t)) = \Big( \E_{x \sim \what{p} } - \E_{x \sim \what{q} } \Big) \what{u}(x,t),\\
    \what{T}_{test}(t) :&=  T(\what{\theta}(t) ; \ptest, \qtest ) = \Big( \E_{x \sim \ptest} - \E_{x \sim \qtest } \Big) f(x,\what{\theta}(t)) \\
    &= \Big( \E_{x \sim \ptest } - \E_{x \sim \qtest } \Big) \what{u}(x,t), \\
    \what{T}_{pop}(t) :&=  T(\what{\theta}(t) ; p, q) = \Big( \E_{x \sim p} - \E_{x \sim q} \Big) f(x,\what{\theta}(t)) = \Big( \E_{x \sim p} - \E_{x \sim q} \Big) \what{u}(x,t),
\end{align*}
where $\what{T}, \what{T}_{test}, \what{T}_{pop}$ denote the evaluation on the training samples, test samples, and population, respectively.

\subsection{Population Dynamics}

These population dynamics are essentially what would happen as the number of samples in our datasets grows larger and larger, since here we train on the population yet still have time-varying dynamics.  In this case, we denote the path of the parameters as simply $\theta(t)$ and will denote the associated neural network's output as $u(x,t) = f(x,\theta(t))$. We showed earlier that the population-level loss is equivalent to simply minimizing $\L (\theta) = \frac{1}{2} \Vert f(\cdot,\theta) - f^*(\cdot) \Vert_{L^2(p+q)}$.  This means that training the population-level neural network is equivalent to running gradient descent on $\L$.  Moreover, we can define the population level error function as $e(x,t) = f(x,\theta(t)) - f^*(x)$.  Using these facts, we get the following
\begin{align*}
    - \Dot{\theta}(t) &= \partial_\theta \L( \theta(t)) = \frac{1}{2} \int_{\R^d} \nabla_\theta f(x,\theta(t)) \big( f(x,\theta(t)) - f^*(x)\big) (p+q)(x) dx \\
    \partial_t u(x,t) &= \< \nabla_\theta f(x,\theta(t)), \Dot{\theta}(t) \>_{\Theta} = -\< \nabla_\theta f(x,\theta(t)), \partial_\theta \L(\theta(t)) \>_{\Theta} \\
    &= - \frac{1}{2} \int_{\R^d} \< \nabla_\theta f(x,\theta(t)), \nabla_\theta f(x',\theta(t))\>_{\Theta} \big( f(x',\theta(t)) - f^*(x')\big) (p+q)(x') dx'.
\end{align*}
Now we define the population level time-varying neural tangent kernel
\begin{align*}
    K_t(x,x') = \< \nabla_\theta f(x,\theta(t)), \nabla_\theta f(x',\theta(t))\>_{\Theta} .
\end{align*}
This finally implies that
\begin{align*}
    \partial_t u(x,t) = -\frac{1}{2} \E_{x' \sim p+q} K_t(x,x') e(x',t) = \partial_t e(x,t).
\end{align*}
Contrary to the realistic dynamics, we only care about the population-level two-sample test statistic with these population dynamics since these dynamics trains on the population itself.  We denote the two-sample test statistic associated to these population-level time-varying dynamics by
\begin{align*}
    T(t) := T(\theta(t); p, q) = \int_{\R^d} u(x,t) d(p - q)(x),
\end{align*}
where we are integrating the neural network output over the entire densities.

\subsection{Zero-time Auxiliary Dynamics}\label{subsec:pop0}

The zero-time auxiliary dynamics are characterized by training on the population but having the dynamics driven by the zero-time NTK. With these zero-time auxiliary dynamics, we will denote related quantities with a bar so that the output of the trained function here becomes $\bar{u}(x,t)$ with $\bar{u}(x,0) = f(x,\theta_0)$.  At this point, consider the zero-time NTK of the neural network $f(\cdot, \theta_0)$ by
\begin{align*}
    K_0(x,x') = \< \nabla_\theta f(x,\theta_0), \nabla_\theta f(x',\theta_0 )\>_{\Theta}.
\end{align*}
According to \cite[Lemma 4.3]{repasky2023neural}, assuming that $\Vert \nabla_\theta f(x, \theta_0) \Vert_{\Theta}$ is squared integrable on $\R^d$ against measure $(p(x) + q(x))dx$, the zero-time kernel can act as a kernel integral operator and admits a spectral decomposition, which we can write as 
\begin{align*}
    (L_{K_0} g)(x) &= \int_{\R^d} K_0(x,x') g(x') (p+q)(x') dx' = \sum_{\ell = 1}^M \lambda_\ell \< g, u_\ell \>_{L^2(p+q)} u_\ell(x),
\end{align*}
where $\lambda_1 \geq \lambda_{2} \geq \dotsc \geq \lambda_M$.  Although we can find an extended basis for $\ell > M$ for $L^2(p+q)$, the associated eigenvalues of $K_0$ are 0 on eigenfunctions $u_\ell$ for $\ell > M$.  Since $K_0$ does not effectively have full basis for $L^2(p+q)$, the quantities that we work with will need to be projected onto the range of the operator $L_{K_0}$.  This motivates the following definition.
\begin{definition}
    Denote the projection operator onto the range of $L_{K_0}$ by $\Pi_{K_0}$.
\end{definition}
Now we can define the associated error function 
\begin{align*}
    \bar{e}(x,t) = \bar{u}(x,t) - \Pi_{K_0}( f^*) (x).
\end{align*}
Since $f^*$ is fixed, the dynamics of our model that we care about will be given by
\begin{align*}
    \partial_t \bar{u}(x,t) = - \frac{1}{2} \E_{x' \sim p+q } K_0(x,x') \bar{e}(x',t) = \partial_t \bar{e}(x,t).
\end{align*}
With the simplicity in the model dynamics, we can attain better analysis.  Using this interpretation, we know that
\begin{align*}
    \partial_t \bar{e}(\cdot,t) = - \frac{1}{2} (L_{K_0} \bar{e}(\cdot, t)) =  - \sum_{\ell = 1}^M \lambda_\ell \< \bar{e}(\cdot, t) , u_\ell \>_{L^2(p+q)} u_\ell .
\end{align*}
We formulate an ansatz of what $\bar{e}(\cdot, t)$ would be so that it satisfies this differential equation.  In particular, consider
\begin{align*}
    \bar{e}(\cdot, t) = \sum_{\ell = 1}^M e^{-t \lambda_\ell} \< u_\ell, \bar{e}(\cdot, 0) \>_{L^2(p+q)} u_\ell.
\end{align*}
The following proposition ensures that this indeed is a solution of the differential equation of interest.
\begin{proposition}[Zero-Time Auxiliary Dynamics Solution]\label{prop:diffeqSol}
    The solution
    \begin{align*}
        \bar{e}(\cdot, t) = \sum_{\ell = 1}^M e^{-t \lambda_\ell} \< u_\ell, \bar{e}(\cdot, 0) \> u_\ell
    \end{align*}
    solves the differential equation
    \begin{align*}
        \partial_t \bar{e}(\cdot,t) = - \frac{1}{2} (L_{K_0} \bar{e}(\cdot, t)) =  - \sum_{\ell = 1}^M \lambda_\ell \< \bar{e}(\cdot, t) , u_\ell \>_{L^2(p+q)} u_\ell .
    \end{align*}
\end{proposition}

The proof of \Cref{prop:diffeqSol} is given in \Cref{pf:diffeqSol}.

For these zero-time auxiliary dynamics, we can now talk about the two-sample test statistic.  In particular, recalling that $\bar{u} = \bar{e} + \Pi_{K_0} (f^*)$, we will set
\begin{align*}
    \overline{\mu_p}(t) &= \int_{\R^d} \bar{u}(x,t) dp(x) = \int_{\R^d} (\bar{e}(x,t) + \Pi_{K_0} (f^*)(x)) p(x) dx, \\
    \overline{\mu_p}(t) &= \int_{X} \bar{u}(x,t) dq(x) = \int_X (\bar{e}(x,t) + \Pi_{K_0} (f^*)(x) ) q(x) dx.
\end{align*}
Similar to the population (population-training data time-varying) dynamics, we only care about the population-level two-sample test statistic for this training regime since we use the entire population for training.  For this training regime, we will denote the associated two-sample test by
\begin{align*}
    \overline{T}(t) := \overline{\mu_p}(t) - \overline{\mu_q}(t) = \int_{\R^d} \bar{u}(x,t) d(p-q)(x) .
\end{align*}
Note that in this case, $\overline{T}(t)$ is not determined by the parameters of the neural network changing but rather by the output of the NTK trained dynamics.  With this in mind, we get the following lemma.
\begin{lemma}[Zero-Time Auxiliary Statistic]\label{lem:T0_form}
    The population-level zero-time kernel dynamics two-sample test statistic is given by
    \begin{align*}
        \overline{T}(t) &= \Vert \pif \Vert_{L^2(p+q)}^2 - \sum_{\ell \geq 1} e^{-t \lambda_\ell} \<u_\ell, \Pi_{K_0} (f^*) \>_{L^2(p+q)}^2.
    \end{align*}
\end{lemma}
The proof of \Cref{lem:T0_form} is in \Cref{pf:T0_form}.

Note that there is some time-analysis we can undertake at this point.  First, you can notice that at $t = 0$,
\begin{align*}
    \overline{T}(0) = \Vert \pif \Vert_{L^2(p+q)}^2 - \underbrace{\sum_{\ell \geq 1} \<u_\ell, \pif \>_{L^2(p+q)}^2}_{\Vert \pif \Vert_{L^2(p+q)}^2 } = 0,
\end{align*}
and for any time $t > 0$, we have $\overline{T}(t) > 0$.  We will find a theoretical minimum time $t(\epsilon)$ such that $\overline{T}(t(\epsilon)) \geq \epsilon$. Let us define a few quantities before delving into the main result.
\begin{definition}
    Let $S \subseteq \{1, \dotsc, M\}$, then we can consider how much of the norm $\Vert \pif \Vert_{L^2(p+q)}^2$ (and hence the norm of $f^*$) lies on the eigenbasis subset $V_S = \{u_\ell\}_{\ell \in S}$.  In particular, we have
    \begin{align*}
        \Vert \pif \Vert_{S}^2 = \Vert f^* \Vert_{S}^2 = \sum_{\ell \in S } \<u_\ell, f^* \>_{L^2(p+q)}^2 = x_S \Vert f^* \Vert_{L^2(p+q)}^2.
    \end{align*}
    Since $\< u_\ell, \pif \>_{L^2(p+q)} = \< u_\ell, f^*\>_{L^2(p+q)}$ for $\ell \leq M$, we can also define these quantities for $f^*$ rather than just $\pif$. Moreover, we can define the minimum and maximum eigenvalue that exist for the eigenvectors that lie in $V_S$ by defining
    \begin{align*}
        \lambda_{\min}(S) = \min_{\ell \in S} \lambda_\ell, \hspace{0.5cm} \lambda_{\max}(S) = \max_{\ell \in S} \lambda_\ell .
    \end{align*}
\end{definition}
Then we have the following theorem.
\begin{proposition}[Minimum Detection Time -- Zero-Time Auxiliary]\label{prop:H1thm}
    Let $\epsilon > 0$ and assume that there exists a finite subset $S \subset \{1, \dotsc, M\}$ such that
    \begin{align*}
        \Vert \pif \Vert_{S}^2 = \Vert f^* \Vert_{S}^2 = \sum_{\ell \in S } \<u_\ell, f^* \>_{L^2(p+q)}^2 > \epsilon
    \end{align*}
    and $\lambda_{\min}(S) > 0$.  Then 
    \begin{align*}
        t(\epsilon) \geq \lambda_{\min}(S) \log\bigg( \frac{\Vert \pif \Vert_S^2 }{ \Vert \pif \Vert_S^2 - \epsilon} \bigg) = \lambda_{\min}(S) \log\bigg( \frac{\Vert f^* \Vert_S^2 }{ \Vert f^* \Vert_S^2 - \epsilon} \bigg)
    \end{align*}
    ensures that $\overline{T}(t(\epsilon)) \geq \epsilon$.
\end{proposition}
The proof of \Cref{prop:H1thm} is in \Cref{pf:H1thm}.

\begin{remark}
    Let us analyze the function
    \begin{align*}
        g(\epsilon, S) &= \min_{S \in \mathcal{S}_1 (\epsilon)} \lambda_{\min}(S) \log\bigg( \frac{\Vert \pif \Vert_{S}^2 }{\Vert \pif \Vert_S^2 - \epsilon} \bigg).
    \end{align*}
    Notice first that the largest that $\epsilon$ can be is $\Vert \pif \Vert_{L^2(p+q)}^2$ because as $t \to \infty$, we get that $\overline{T}(t) \to \Vert \pif \Vert_{L^2(p+q)}^2$ and it is easy to see that $\overline{T}(t)$ is monotonic in $t$.  Now notice that as $\epsilon$ gets larger, we need $S$ to satisfy $\Vert f^* \Vert_{S}^2 = \Vert \pif \Vert_{S}^2 > \epsilon$ to make sure that $g(\epsilon,S)$ is well-defined.  Moreover, we need $\lambda_{\min}(S) > 0$ otherwise we find that $g(\epsilon,S) = 0$ which is the trivial bound.  In particular, we will want the fraction $\epsilon/\Vert \pif \Vert_{S}^2$ to be as small as possible to give the smallest possible non-trivial time.
\end{remark}

We have an analogous statement for when we want our test statistic to be less than $\epsilon$.  In particular, we get that
\begin{proposition}[Maximum Undetectable Time -- Zero-Time Auxiliary]\label{prop:H0thm}
    Let $\epsilon > 0$ and assume that there exists a finite subset $S \subset \mathbb{N}$ such that
    \begin{align*}
        \frac{ \Vert \pif \Vert_{L^2(p+q)}^2 - \epsilon}{\Vert \pif \Vert_{S}^2 } > 0
    \end{align*}
    and $\lambda_{\max}(S) > 0$.  Then 
    \begin{align*}
        t(\epsilon) \leq \lambda_{\max}(S) \log\Bigg\{\frac{\Vert \pif \Vert_{S}^2}{\Vert \pif \Vert_{L^2(p+q)}^2 - \epsilon} \Bigg\}
    \end{align*}
    ensures that $\overline{T}(t(\epsilon)) \leq  \epsilon$.
\end{proposition}
The proof of \Cref{prop:H0thm} is in \Cref{pf:H0thm}.  Now if we optimize over all such subsets $S$, we get the following corollary.

\begin{corollary}[Detection Times -- Zero-Time Auxiliary]\label{timeAnalysisCor}
    Let $0 < \epsilon < \Vert \pif \Vert_{L^2(p+q)}^2$.  Assume that the set
    \begin{align*}
        \mathcal{S}_1 (\epsilon) &= \{ S \subset \mathbb{N} : \Vert \pif \Vert_S^2 > \epsilon, \lambda_{\min}(S) > 0 \} \neq \emptyset \\
        \mathcal{S}_2 (\epsilon) &= \{ S \subset \mathbb{N} : ( \Vert \pif \Vert_{L^2(p+q)}^2 - \epsilon)/\Vert \pif \Vert_S^2 > 0, \lambda_{\max}(S) > 0 \} \neq \emptyset.
    \end{align*}
    Then
    \begin{align*}
        t \geq t_1^*(\epsilon) := \min_{S \in \mathcal{S}_1 (\epsilon)} \lambda_{\min}(S) \log\Bigg( \frac{\Vert \pif \Vert_{S}^2 }{\Vert \pif \Vert_S^2 - \epsilon} \Bigg)
    \end{align*}
    ensures that $\overline{T}(t) \geq \epsilon$ whilst 
    \begin{align*}
        t \leq t_2^*(\epsilon) := \max_{S \in \mathcal{S}_2 (\epsilon)} \lambda_{\max}(S) \log\Bigg(  \frac{\Vert \pif \Vert_{S}^2}{\Vert \pif \Vert_{L^2(p+q)}^2 - \epsilon} \Bigg)
    \end{align*}
    ensures that $\overline{T}(t) \leq \epsilon$.
\end{corollary}

Here let us remark what occurs in the case when our null hypothesis is correct versus when the alternative is correct.
\begin{remark}
    If $H_0$ is true (so that $p = q$), then $f^* = 0$.  We can't apply the theorem above then since the assumption is not satisfied; however, we note by inspection that $\overline{T}(t) = 0$ for all $t$.  If $\Vert \pif \Vert_{L^2(p+q)} < \delta$ for small $\delta > 0$, then note that both \Cref{prop:H1thm} and \Cref{prop:H0thm} limit $\epsilon < \delta$.  This means that $\overline{T}(t)$ can only detect small changes.  On the other hand, if we are under $H_1$ (so that $p \neq q$) and we assume that $\Vert \pif \Vert_{L^2(p+q)} > \delta$ for some larger $\delta > 0$, then $\epsilon$ can be made much larger and should be more easy to detect.
\end{remark}

\section{Realistic and Population Dynamics Approximation}\label{sec:approx}

The results in this section are crucial and foundational for building the results in \Cref{sec:timepower}.  In essence, we are trying to show that $\what{u}$ and $u$ are close to $\Bar{u}$, which can easily imply that the associated two-sample statistics $\what{T}$ and $T$ are close to $\overline{T}$.  After showing that the statistics are close the zero-time NTK statistic, we will show $\what{T}$ and $T$ are monotonically increasing in time in order to prove \Cref{thm:uhat_H1} and \Cref{thm:u_H1}.

The general process of showing $\what{u}$ and $u$ are close to $\Bar{u}$ requires
\begin{enumerate}
    \item Showing that $\what{\theta}(t)$ and $\theta(t)$ are close to $\theta(0)= \theta_0$ in time.  Along with \Cref{C1}, this is used to ensure that $K_t$ is close to $K_0$.  Note that this is all that is needed to show $u$ is close to $\Bar{u}$.

    \item Show that a neural network version of matrix Bernstein's inequality holds for the sampled kernel $\what{K}_t$.  Along with showing the parameters are close, \Cref{C1} and \Cref{C2} are used here to show that $\what{K}_t$ is close to $K_0$.
\end{enumerate}

Formally, there are a few assumptions that we will need to ensure the analysis works.  To start, let $B_R$ denote the open ball of radius $R$ with center $\theta_0$ and assume that $u(x,0) = \Bar{u}(x, 0) = f(x,\theta_0) = 0$ for all $x$.  For much of the analysis of realistic and population dynamics going forward, we will use the following lemma heavily.
\begin{lemma}
    $\L(\theta(0)) = \Vert u(\cdot, 0) - f^* \Vert_{L^2(p+q)}^2 = \Vert \Bar{u}(\cdot, 0) - f^* \Vert_{L^2(p+q)}^2 = \Vert f^* \Vert_{L^2(p+q)}^2$.
\end{lemma}
\begin{proof}
    Notice that since $u(\cdot, 0) = \Bar{u}(\cdot, 0) = f(x,\theta_0) = 0$, we have the result.
\end{proof}
For both the realistic and population dynamics, we will need to assume bounded and Lipschitz gradient as shown in the following assumption.
\begin{assumption}[Gradient Condition]\label{C1}
    There exists positive constants $R, L_1$, and $L_2$ such that 
    \begin{enumerate}
        \item (Boundedness) For any $\theta \in B_R$, $\sup_{x \in \supp(p+q)} \Vert \nabla_\theta f(x, \theta) \Vert \leq L_1$.

        \item (Lipschitz) For any $\theta_1, \theta_2 \in B_R$, $\sup_{x \in \supp(p+q)} \Vert \nabla_\theta f(x, \theta_1) - \nabla_\theta f(x, \theta_2) \Vert \leq L_2 \Vert \theta_1 - \theta_2 \Vert$.
    \end{enumerate}
\end{assumption}
For only the realistic dynamics case, we will need to assume a large enough training sample size as is specified in the following assumption in order to get a matrix Bernstein type statement for our neural network.
\begin{assumption}[Sample Size Condition]\label{C2}
    For $A > 0$ (representing a desired probability level), consider the function
    \begin{align*}
        h(n) = \sqrt{ 2L_1^2 (2L_1^2 + 3/2)  \frac{  A \log(n) + \log(2 M_\Theta)   }{n} }.
    \end{align*}
    Assume that the training sample sizes $n_p$ and $n_q$ are large enough that $h(n_p) < \frac{3}{2}$ and $h(n_q) < \frac{3}{2}$.
\end{assumption}

\subsection{Population Dynamics Approximation}\label{sec:approxPop}

As the population dynamics have no sample size associated to them, we do not need to worry about $n_p$ or $n_q$.  In this case, we use that the loss
\begin{align*}
    \L (\theta) = \frac{1}{2} \Vert f(\cdot,\theta) - f^*(\cdot) \Vert_{L^2(p+q)}
\end{align*}
is decreasing to get the following lemma.

\begin{lemma}[Population-Dynamics Parameter Bound]\label{lem:paramApprox_pop}
    Assume that $u(x,0) = \Bar{u}(x,0) = 0$, then
    \begin{align*}
         \Vert \theta(t) - \theta(0) \Vert &\leq \sqrt{t} \Vert f^* \Vert_{L^2(p+q)}.
    \end{align*}
    Moreover, if 
    \begin{align*}
        t \leq \bigg( \frac{R}{\Vert f^* \Vert_{L^2(p+q)} } \bigg)^2,
    \end{align*}
    then $\theta(t) \in B_R$.
\end{lemma}
The proof of \Cref{lem:paramApprox_pop} is contained in \Cref{pf:paramApprox_pop}.
Now with \Cref{C1}, we can further bound the operator norm  of the difference $K_t - K_0$ with the following lemma.
\begin{lemma}[Population-Dynamics NTK Bound]\label{lem:KtK0bd}
    Let $\theta(t) \in B_R$, then under \Cref{C1}, we have 
    \begin{align*}
        \Vert K_t - K_0 \Vert_{L^2(p+q)} \leq 2 L_1 L_2 \sqrt{t} \Vert f^* \Vert_{L^2(p+q)}.
    \end{align*}
\end{lemma}

The proof of \Cref{lem:KtK0bd} is contained in \Cref{pf:KtK0bd}.  Now we use this result for bounding the difference $\Vert u - \Bar{u} \Vert_{L^2(p+q)}$.  In particular, we have the following proposition.
\begin{proposition}[Population-Dynamics Output Approximation]\label{prop:uubarBND}
    Under \Cref{C1} and
    \begin{align*}
        t &\leq \bigg( \frac{R}{\Vert f^* \Vert_{L^2(p+q)} } \bigg)^2,
    \end{align*}
    we get
    \begin{align*}
        \Vert (u - \Bar{u})(\cdot, t) \Vert_{L^2(p+q)} = \Vert (e - \Bar{e})(\cdot, t) \Vert_{L^2(p+q)} \leq \frac{8}{3} L_1 L_2 \Vert f^* \Vert_{L^2(p+q)}^2  (t)^{3/2}.
    \end{align*}
\end{proposition}

The proof of \Cref{prop:uubarBND} is contained in \Cref{pf:uubarBND}.  Now let us extend our zero-time NTK two-sample test results to the population-level time-varying kernel two-sample test.  We first notice the following corollary.

\begin{corollary}[Population Statistic Bound]\label{cor:testStatU_Ubar_bnd}
    Under \Cref{C1} and
    \begin{align*}
        t \leq \bigg( \frac{R }{\Vert f^* \Vert_{L^2(p+q)} } \bigg)^2,
    \end{align*}
    we have 
    \begin{align*}
        \Big\vert T(t) - \overline{T}(t) \Big\vert \leq \sqrt{2} \frac{8}{3} L_1 L_2 \Vert f^* \Vert_{L^2(p+q)}^2 (t)^{3/2} .
    \end{align*}
\end{corollary}

The proof of \Cref{cor:testStatU_Ubar_bnd} is contained in \Cref{pf:testStatU_Ubar_bnd}.  For further time analysis in the alternative hypothesis case below, we will need to show that this two-sample test $T(t)$ is monotonically increasing.  We will be able to show this if our population-level neural network has increasing norm.  This assumption is not unsupported since we initialize as $u(x,0) = f(x, \theta(0)) = 0$ and our target function $f^*(x) = \frac{p-q}{p+q}(x)$ has non-zero norm.  Using this assumption, we get the following theorem (with proof contained in \Cref{pf:stat_Monotone}).
\begin{lemma}[Population Statistic Monotonicity]\label{lem:stat_Monotone}
    Assume that $\Vert u(x,t) \Vert_{L^2(p+q)}$ is monotonically increasing on the interval $[0,\tau]$, then $T(t)$ is monotonically increasing on $[0,\tau]$.
\end{lemma}
Note that we only need \Cref{C4:stat_Monotone} for this theorem to work.

\subsection{Realistic Approximation}\label{sec:approxReal}

Since the realistic approximation represents a neural network trained on finite samples and time-varying dynamics, we will be using finite-samples, which require some concentration inequalities and need a few extra assumptions than the population-dynamics case.  To start off, recall that for our finite-sample we have $n_p$ training samples from density $p$ and $n_q$ samples from density $q$.  Moreover, recall that our finite-sample loss function is given by
\begin{align*}
    \what{L}(\theta) &= \frac{1}{2} \bigg( \int_{\R^d} \big( f(x,\theta) - 1 \big)^2 \what{p}(x) dx + \int_{\R^d} \big( f(x,\theta) + 1 \big)^2 \what{q}(x) dx \bigg)
\end{align*}
Using that this loss decreases in time, we get the following lemma.
\begin{lemma}[Realistic Parameter Bound]\label{lem:thetahatBD}
    Assume that $\what{\theta}(0) = \theta(0)$ and that $f(x, \theta(0)) = 0$, then
    \begin{align*}
        \Vert \what{\theta}(t) - \theta(0) \Vert_{\Theta} \leq \sqrt{t} .
    \end{align*}
    Moreover, if $\theta(0) \in B_R$, then $t \leq  R^2$ ensures that $\what{\theta}(t) \in B_R$.
\end{lemma}
We prove \Cref{lem:thetahatBD} in \Cref{pf:thetahatBD}. Using \Cref{C2}, we can apply \Cref{matBern} to get the following lemma to be used later.
\begin{lemma}[Neural Net Matrix Bernstein]\label{lem:matBernNN}
    Assume that $\theta \in B_R$, then under \Cref{C1} and $\Cref{C2}$ and let $p$ be a probability density, consider the random $M_\Theta$-by-$M_\Theta$ matrix 
    \begin{align*}
        X_i = \nabla_\theta f(x_i, \theta) \nabla_\theta f(x_i, \theta)^\top - \E_{x\sim p} \nabla_\theta f(x, \theta) \nabla_\theta f(x, \theta)^\top.
    \end{align*}
    If $n$ is the number of samples from $p$, then with probability greater than $1 - n^{-A}$, we have
    \begin{align*}
          \Vert \frac{1}{n} \sum_{i=1}^n X_i \Vert \leq \sqrt{ 2L_1^2 (2L_1^2 + 3/2)  \frac{  A \log(n) + \log(2 M_\Theta)   }{n} }.
    \end{align*}
\end{lemma}

The proof of \Cref{lem:matBernNN} is in \Cref{pf:matBernNN} and is used in the following proposition.

\begin{proposition}[Realistic Output Approximation]\label{prop:uhatubarBND_informal}
    Assume that $t \leq R^2$ (so that $\theta(t) \in B_R$) as well as \Cref{C1} and \Cref{C2}, then with probability $\geq 1 - n_p^{-A} - n_q^{-A}$, we have
    \begin{align*}
        \Vert (\what{u}-\Bar{u})(\cdot,t) \Vert_{L^2(p+q)} &\leq C_1 t + C_2 t^{3/2} + C_3 t^2 + C_4 t^{5/2},
    \end{align*}
    where the dependence of the constants is given by $C_1 = C(L_1)$, $C_2 = C(L_1, L_2, f^*)$, $C_3 = C(L_1, f^*, n_p, n_q, M_\Theta, A )$, and $C_4 = C(L_1, L_2, f^*, n_p, n_q, M_\Theta, A)$.
\end{proposition}

Note that the more technical version of \Cref{prop:uhatubarBND_informal} is contained in \Cref{uhatubarBND} along with its proof.  We will use \Cref{prop:uhatubarBND_informal} to show that the finite-sample two-sample test statistic and zero-time kernel population-level two-sample test statistic are close for $\what{T}_{pop}$, $\what{T}_{train}$, and $\what{T}_{test}$ (i.e. the evaluation of the finite-sample two-sample test statistic on the population, training samples, and test samples, respectively).  For $\what{T}_{pop}$, we get the following proposition.
\begin{proposition}[Realistic Population-Evaluated Statistic]\label{prop:testStatUhat_pop_Ubar_bnd}
    Assume the conditions of $\Cref{prop:uhatubarBND_informal}$, then with probability $\geq 1 - (n_p^{-A} + n_q^{-A} )$, we get the time-approximation error function
    \begin{align*}
        \big\vert \what{T}_{pop}(t) - \overline{T}(t) \big\vert &\leq C_1 t + C_2 t^{3/2} + C_3 t^2 + C_4 t^{5/2} := \delta_{pop}(t),
    \end{align*}
    where $C_1, C_2, C_3, C_4$ are exactly the constants from \Cref{prop:uhatubarBND_informal}.  Moreover, note that this error function is monotonic.
\end{proposition}
\begin{proof}[Proof of \Cref{prop:testStatUhat_pop_Ubar_bnd}]\label{pf:testStatUhat_pop_Ubar_bnd}
    Mimic the proof of \Cref{cor:testStatU_Ubar_bnd} \textit{mutatis mutandis} applying \Cref{prop:uhatubarBND_informal}.
\end{proof}

Now for $\what{T}_{test}(t)$, the test size sample sizes come into play since 
\begin{align*}
    \what{T}_{test}(t) = T(\what{\theta}(t); \ptest, \qtest) = \big( \E_{x \sim \ptest} - \E_{x - \qtest} \big) \what{u}(x,t),
\end{align*}
with test sample sizes $m_p$ and $m_q$ for $\ptest$ and $\qtest$, respectively. We also describe the time-approximation error function in the next proposition, which will be used for the theorems and corollaries afterwards.
\begin{proposition}[Realistic Test Sample-Evaluated Statistic]\label{prop:testStatUhat_Ubar_bnd_samp}
    Assume the conditions of \Cref{prop:uhatubarBND_informal}, then with probability $\geq 1 - (m_p^{-A} + m_q^{-A})$,
    we have
    \begin{align*}
        \vert \what{T}_{test}(t) - \what{T}_{pop}(t) \vert \leq L_1^2 t \sqrt{2} \bigg( \sqrt{ \frac{A \log(m_p)}{m_p} } + \sqrt{ \frac{A \log(m_q) }{m_q} } \bigg). 
    \end{align*}
    Moreover, with probability $\geq 1 - (m_p^{-A} + m_q^{-A} + n_p^{-A} + n_q^{-A} )$, we get the time-approximation error function
    \begin{align*}
        \big\vert \what{T}_{test}(t) - \overline{T}(t) \big\vert &\leq \Tilde{C}_1 t + C_2 t^{3/2} + C_3 t^2 + C_4 t^{5/2} := \delta(t),
    \end{align*}
    where $\Tilde{C}_1 = C(L_1, A, m_p, m_q)$ and $C_2, C_3, C_4$ are exactly the constants from \Cref{prop:uhatubarBND_informal}.  Finally, note that this error function $\delta(t)$ is monotonic.
\end{proposition}
The proof of this proposition is located in \Cref{pf:testStatUhat_Ubar_bnd_samp}.

\begin{remark}\label{rem:error_hat}
We note that \Cref{prop:testStatUhat_Ubar_bnd_samp} works for $\what{T}_{train}$ if we replace $m_p$ and $m_q$ with $n_p$ and $n_q$ respectively.  Since the error function $\delta$ depends on whether we use test samples or training samples, we will regard the error function by $\delta_{test}(t)$ and $\delta_{train}(t)$ to distinguish these cases.  In  particular, the only constant that is different in $\delta_{train}$ and $\delta_{test}$ is $\Tilde{C}_1$, where we change $m_p$ and $m_q$ to $n_p$ and $n_q$, respectively.  Moreover, using the triangle inequality, we can see that
\begin{align*}
    \vert \what{T}_{test}(t) - \what{T}_{train}(t) \vert \leq L_1^2 t \sqrt{2} \bigg( \sqrt{ \frac{A \log(m_p)}{m_p} } + \sqrt{ \frac{A \log(n_p) }{n_p} } \\
    + \sqrt{ \frac{A \log(m_q)}{m_q} } + \sqrt{ \frac{A \log(n_q) }{n_q} }\bigg).
\end{align*}
Finally, we may also deduce from \Cref{prop:testStatUhat_Ubar_bnd_samp} that
\begin{align*}
    \vert \what{T}_{train}(t) - \what{T}_{pop}(t) \vert \leq L_1^2 t \sqrt{2} \bigg( \sqrt{ \frac{A \log(n_p) }{n_p} } + \sqrt{ \frac{A \log(n_q) }{n_q} }\bigg).
\end{align*}
\end{remark}

To do further time-analysis in this finite-sample training case, we will need that $\what{T}_{train}(t)$ is monotonic in time.  We will then use sampling concentration of $\what{T}_{train}$ with $\what{T}_{pop}$ and $\what{T}_{test}$ to extend the two-sample test statistic to these two different evaluation settings.

\begin{lemma}[Realistic Statistic Monotonicity]\label{lem:stat_hat_Monotone}
    Assume that there is an interval $[0,\what{\tau}]$ such that $\vert \what{u}(x,s) \vert \leq 1$ for $s \in [0,\what{\tau}]$ and $x \in \text{supp}(\what{p} + \what{q})$.  Then $\what{T}_{train}(t)$ is monotonically increasing on $[0,\what{\tau}]$.
\end{lemma}
Note that the assumption of \Cref{lem:stat_hat_Monotone} is exactly \Cref{C3:stat_hat_Monotone}. We include the proof of \Cref{lem:stat_hat_Monotone} in \Cref{pf:stat_hat_Monotone}.

\begin{remark}
    Note that the assumption $\vert \what{u}(x,s) \vert \leq 1$ definitely holds for at least small time intervals since the training dynamics are smooth and $\what{u}(x,0) = 0$.  Moreover, we crucially use the fact that the training loss is decreasing for the proof of \Cref{lem:stat_hat_Monotone}.
\end{remark}

\section{Test Statistic Time and Power Analysis}\label{sec:timepower}

We want to extend the zero-time auxiliary time analysis of \Cref{prop:H1thm} and \Cref{prop:H0thm} to the realistic and population dynamics cases.  The general method of extending the time analysis of the zero-time auxiliary analysis requires us to

\begin{enumerate}
    \item \textbf{Done in \Cref{sec:approx}:} Approximate the realistic and population dynamics outputs ($\what{u}$ and $u$, respectively) with the zero-time auxiliary output $\Bar{u}$,

    \item \textbf{Done in \Cref{sec:approx}:} Extend the approximation to the realistic two-sample statistic and population-dynamics two-sample statistic as well as show monotonicity of the statistics,

    \item \textbf{Done in this section:} Combine the two-sample approximations with the time analysis of \Cref{prop:H1thm} and \Cref{prop:H0thm}.
    
\end{enumerate}

Here, it is useful to recall that both the realistic and population dynamics need \Cref{C1} whilst only the realistic dynamics need the additional \Cref{C2}.

\subsection{Population Time Analysis}\label{sec:timePop}

Since the population dynamics do not have any samples, we cannot really perform any test power analysis in this scenario.  We can, however, use approximation theorems from \Cref{sec:approxPop} to extend the zero-time NTK time analysis theorems.

We need to counteract the time-dependent estimation error shown in \Cref{cor:testStatU_Ubar_bnd}. The next theorem, which is geared towards discovering the alternative hypothesis, necessarily assumes first that the minimum time needed to detect an error $\epsilon$ in \Cref{timeAnalysisCor} is smaller than $\epsilon$ and second that the time scale we work on is valid for detection.  Note that as the size of the neural network grows, the minimum time needed for detection decreases but $L_1$ and $L_2$ below increase; thus, there is an interplay of making sure your neural network is large but not too large.  We will need the following monotonicity assumption for the alternative hypothesis time analysis.

\begin{assumption}[Population Monotonicity Condition]\label{C4:stat_Monotone}
Assume that $\Vert u(x,t) \Vert_{L^2(p+q)}^2$ is monotonically increasing on $[0,\tau]$.
\end{assumption}

For this next theorem, recall from \Cref{timeAnalysisCor} that
\begin{align*}
    t_1^*(\epsilon) &:= \min_{S \in \mathcal{S}_1(\epsilon) } \lambda_{\min}(S) \log \Bigg( \frac{\Vert \pif \Vert_S^2 }{\Vert \pif \Vert_S^2 - \epsilon} \Bigg), \\
    t_2^*(\epsilon) &:= \max_{S \in \mathcal{S}_2 (\epsilon)} \lambda_{\max}(S) \log\Bigg(  \frac{\Vert \pif \Vert_{S}^2}{\Vert \pif \Vert_{L^2(p+q)}^2 - \epsilon} \Bigg).
\end{align*}

\begin{theorem}[Minimum Detection Time -- Population]\label{thm:u_H1}
    Assume \Cref{C1} and \Cref{C4:stat_Monotone} hold and let $\epsilon > 0$.  Then for
    \begin{align*}
        \min\bigg( \tau, \Big( \frac{R  }{ \Vert f^* \Vert_{L^2(p+q)} } \Big)^2 \bigg) \geq t \geq t_1^*(\epsilon),
    \end{align*}
    where $t_1^*(\epsilon)$ is defined in \Cref{timeAnalysisCor}, we get
    \begin{align*}
        \vert T(t) \vert \geq \epsilon - \frac{8 \sqrt{2}}{3} \Vert f^* \Vert_{L^2(p+q)}^2 L_1 L_2 \big( t_1^*(\epsilon) \big)^{3/2}.
    \end{align*}
    If $\epsilon$ is not large enough to make the right-hand sides of the inequalities positive or $t_1^*(\epsilon)$ is not smaller than $\tau$ or $(R / \Vert f^* \Vert_{L^2(p+q)} )^2$, the bound is vacuous.
\end{theorem}

We prove \Cref{thm:u_H1} in \Cref{pf:u_H1}.  Now, the following theorem is useful in showing the null hypothesis and necessarily needs the time to be smaller the maximum time needed to detect $\epsilon$ as well as the time needed to stay in $B_R$ (so that \Cref{prop:uubarBND} holds).

\begin{theorem}[Maximum Undetectable Time -- Population]\label{thm:u_H0}
    Let $\epsilon > 0$.  Under \Cref{C1} and for
    \begin{align*}
        t \leq \min\Bigg\{ \Big( \frac{R  }{ \Vert f^* \Vert_{L^2(p+q)} } \Big)^2,  t_2^*(\epsilon) \Bigg\},
    \end{align*}
    where $t_2^*(\epsilon)$ is defined in \Cref{timeAnalysisCor}, we have
    \begin{align*}
        \vert T(t) \vert \leq \epsilon + \frac{8 \sqrt{2}}{3} \Vert f^* \Vert_{L^2(p+q)}^2 L_1 L_2 \big( t_2^*(\epsilon) \big)^{3/2}.
    \end{align*}
\end{theorem}
The proof of \Cref{thm:u_H0} is given in \Cref{pf:u_H0}.

\subsection{Realistic Time Analysis}\label{sec:timeReal}

Following the approximation theorems of \Cref{sec:approxReal}, we can first get the realistic dynamics extension of the zero-time NTK time-analysis theorems and second conduct test power analysis.  Specifically for the minimum detection analysis, we need the following assumption, which will ensure monotonicity of the two-sample statistic in the realistic dynamics case.
\begin{assumption}[Realistic Monotonicity Condition]\label{C3:stat_hat_Monotone}
    Assume that there is an interval $[0,\what{\tau}]$ such that $\vert \what{u}(x,s) \vert \leq 1$ for $s \in [0,\what{\tau}]$ and $x \in \text{supp}(\what{p} + \what{q})$.
\end{assumption}
\begin{remark}
    Note that the assumption $\vert \what{u}(x,s) \vert \leq 1$ definitely holds for at least small time intervals since the training dynamics are smooth and $\what{u}(x,0) = 0$.  Moreover, we crucially use the fact that the training loss is decreasing for the proof of \Cref{lem:stat_hat_Monotone}.
\end{remark}
Recall from \Cref{timeAnalysisCor} the $\epsilon$-detection time thresholds for the zero-time NTK two-sample test given by
\begin{align*}
    t_1^*(\epsilon) &:= \min_{S \in \mathcal{S}_1(\epsilon) } \lambda_{\min}(S) \log \Bigg( \frac{\Vert \pif \Vert_S^2 }{\Vert \pif \Vert_S^2 - \epsilon} \Bigg), \\
    t_2^*(\epsilon) &:= \max_{S \in \mathcal{S}_2 (\epsilon)} \lambda_{\max}(S) \log\Bigg(  \frac{\Vert \pif \Vert_{S}^2}{\Vert \pif \Vert_{L^2(p+q)}^2 - \epsilon} \Bigg).
\end{align*}
From a birds-eye view, note that the following realistic time analysis theorems assume
\begin{enumerate}
    \item \Cref{C1} and \Cref{C2}

    \item For the minimum $\epsilon$-detection time analysis, $t_1^*(\epsilon) \leq t \leq \min(R^2,\what{\tau})$, where $\what{\tau}$ comes from \Cref{C3:stat_hat_Monotone}. For the maximum $\epsilon$-undetectability time analysis, $t \leq \min\{ R^2, t_2^*(\epsilon) \}$.
\end{enumerate}
To counteract the time-valued approximation error of $\bar{u}$ with $\what{u}$, we must assume that the error detection level is greater than the approximation error.

\begin{theorem}[Minimum Detection Time -- Realistic]\label{thm:uhat_H1}
    Let $\epsilon > 0$.  Assume that \Cref{C1}, \Cref{C2}, and \Cref{C3:stat_hat_Monotone} hold.
    If $\max(R^2,\what{\tau}) \geq t \geq t_1^*(\epsilon)$, then
    \begin{enumerate}
        \item  with probability $\geq 1 - 2(n_p^{-A} + n_q^{-A})$,
            \begin{align*}
                \vert \what{T}_{train}(t) \vert &\geq \epsilon - \delta_{train}(t_1^*(\epsilon)),
            \end{align*}

        \item with probability $\geq 1 - (n_p^{-A} + n_q^{-A} + m_p^{-A} + m_q^{-A})$,
        \begin{align*}
            \vert \what{T}_{test}(t) \vert &\geq \epsilon - \delta_{train}(t_1^*(\epsilon)) - L_1^2 t \sqrt{2} \bigg(\sqrt{ A \log(m_p) / m_p } + \sqrt{ A \log(m_q) / m_q } \\
            &+ \sqrt{ A \log(n_p) / n_p } + \sqrt{ A \log(n_q) / n_q }\bigg),
        \end{align*}

        \item with probability $\geq 1 - (n_p^{-A} + n_q^{-A})$,
        \begin{align*}
            \vert \what{T}_{pop}(t) \vert &\geq \epsilon - \delta_{train}(t_1^*(\epsilon)) - L_1^2 t \sqrt{2} \bigg( \sqrt{ A \log(n_p) / n_p } + \sqrt{ A \log(n_q) / n_q } \bigg),
        \end{align*}

    \end{enumerate}
    where the approximation error function
    \begin{align*}
        \delta_{train}(t) &= C(L_1, A, n_p, n_q) t + C_2 t^{3/2} + C_3 t^2 + C_4 t^{5/2}
    \end{align*}
    comes from \Cref{prop:testStatUhat_Ubar_bnd_samp} with training samples and $t_1^*(\epsilon)$ from \Cref{timeAnalysisCor}.  Moreover, if $\epsilon$ is not large enough to make the right-hand sides of the inequalities positive, the bounds are vacuous.
\end{theorem}

We don't need monotonicity for the maximum undetectibility time case in \Cref{timeAnalysisCor} because we use the regular triangle inequality.  Thus, the following theorem holds for each of $\what{T}_{train}, \what{T}_{test},$ and $\what{T}_{pop}$ with their respective time-approximation error functions 
\begin{align*}
    \delta_{test}(t) &= C(L_1, A, m_p, m_q) t + C_2 t^{3/2} + C_3 t^2 + C_4 t^{5/2} \\
    \delta_{pop}(t) &= C_1 t + C_2 t^{3/2} + C_3 t^2 + C_4 t^{5/2},
\end{align*}
where $C_1, C_2, C_3, C_4$ are constants from \Cref{prop:uhatubarBND_informal} and $C(L_1, A, m_p, m_q)$ comes from \Cref{prop:testStatUhat_Ubar_bnd_samp}.

\begin{theorem}[Maximum Undetectability Time -- Realistic]\label{thm:uhat_H0}
    Let $\epsilon > 0$ and assume
    \begin{align*}
        t \leq \min\Big\{ R^2,  t_2^*(\epsilon) \Big\},
    \end{align*}
    where $t_2^*(\epsilon)$ is defined in \Cref{timeAnalysisCor}.  Then
    \begin{enumerate}
        \item with probability $\geq 1 - 2 (n_p^{-A} + n_q^{-A})$, we have
            \begin{align*}
                \vert \what{T}_{train}(t) \vert &\leq \epsilon + \delta_{train}(t_2^*(\epsilon)) ,
            \end{align*}
        
        \item with probability $\geq 1 - (n_p^{-A} + n_q^{-A} + m_p^{-A} + m_q^{-A})$, we have
            \begin{align*}
                \vert \what{T}_{test}(t) \vert &\leq \epsilon + \delta_{test}(t_2^*(\epsilon)) ,
            \end{align*}

        \item with probability $\geq 1 - (n_p^{-A} + n_q^{-A})$, we have
            \begin{align*}
                \vert \what{T}_{pop}(t) \vert &\leq \epsilon + \delta_{pop}(t_2^*(\epsilon)).
            \end{align*}
    \end{enumerate}
    where $\delta_{train}, \delta_{test}, \delta_{pop}$ are time approximation error functions coming from \Cref{rem:error_hat} and \Cref{prop:testStatUhat_pop_Ubar_bnd}.
\end{theorem}

We include the proofs of both \Cref{thm:uhat_H1} and \Cref{thm:uhat_H0} in \Cref{pf:uhat_H1}.  These theorems are used to perform the power analysis in \Cref{sec:powerReal}.

\subsection{Realistic Statistical Power Analysis}\label{sec:powerReal}

Assuming we are in the realistic dynamics scenario, we want to perform some amount of test power analysis by showing that the amount of time it takes to detect a desired deviation level $\epsilon > 0$ is much faster in the alternative hypothesis case versus the null hypothesis case. To do such quantitative analysis, however, we need a more concrete setting. To this end, consider the more concrete setting of $f^*$ lying on the first $k$ eigenfunctions of $K_0$. Now, we want to see if the time to detect $\epsilon$ is larger whether we are in the null hypothesis or in the first $k$ eigenfunction assumption. Solving this problem becomes slightly complex since there is a time-approximate error term in the deviation that comes from \Cref{prop:testStatUhat_Ubar_bnd_samp}.  Since this assumption is not exactly the logical complement of the null, we define the setting more concretely.

\begin{definition}[Projected Target Function]
    If $\pif$ nontrivially projects \textbf{only} onto the first $k$ eigenfunctions of $K_0$ holds true, we denote the projected target function on the first $k$ eigenfunctions as $\pif = f^*_k$.  We denote the test statistic when $f^*_k \neq 0$ by $\what{T}_{train,k}(t), \what{T}_{test,k}(t), \what{T}_{pop, k}(t)$ evaluated on the training set, test set, and population, respectively, and when the evaluation set is understood from the context, we use $\what{T}_k(t)$.  If $p = q$, we say the null hypothesis holds.  We denote the test statistic under this null hypothesis by $\what{T}_{train, null}(t), \what{T}_{test, null}(t)$, and $\what{T}_{pop, null}(t)$ depending on the evaluation set, and when the evaluation set is understood from the context, we use $\what{T}_{null}(t)$.
\end{definition}

In this definition, note $f^*$ is not supported on just the first $k$ eigenfunctions, but rather only the projection via the zero-time kernel $\pif$ is supported on the first $k$ eigenfunctions.  This means that $f^*$ may have a nonzero component that is orthogonal to $\pif$.  Note that we have three two-sample test situations since the two-sample test depends on which dataset it is evaluated on.  In particular, we will combine the results for $\what{T}_{pop}, \what{T}_{test},$ and $\what{T}_{train}$ into the following corollary since the only difference is given by a difference in constants.

\begin{corollary}[Realistic Test Time Analysis]\label{cor:genMomRes_uhat}
    Let $\Vert \pif \Vert_{L^2(p+q)}^2 / 2 > \epsilon > 0$ be a detection level.  Now consider the evaluation dependent constant
    \begin{align*}
        C^+ = \begin{cases}
            C_{pop}^+ = \sqrt{2}L_1^2 4 & \what{T}_{pop}(t) \text{ evaluation} \\
            C_{train}^+ = \sqrt{2}L_1^2 \Bigg(4 + \sqrt{ A\frac{\log(n_p)}{n_p} } + \sqrt{ A \frac{\log(n_q)}{n_q} } \Bigg) & \what{T}_{train}(t) \text{ evaluation} \\
            C_{test}^+ = \sqrt{2}L_1^2 \Bigg(4 + \sqrt{ A\frac{\log(m_p)}{m_p} } + \sqrt{ A \frac{\log(m_q)}{m_q} } \Bigg) & \what{T}_{test}(t) \text{ evaluation}
        \end{cases},
    \end{align*}
    coming from \Cref{prop:testStatUhat_Ubar_bnd_samp} and \Cref{prop:testStatUhat_pop_Ubar_bnd} and consider a time separation level $\epsilon/C^+ \geq \gamma > 0$. Let $t^-(\epsilon)$ be such that for $t \geq t^-(\epsilon)$, we have $\what{T}_{k}(t) \geq \epsilon$ (for our different evaluation settings).  Similarly, let $t^+(\epsilon)$ be such that for $t \geq t^+(\epsilon)$, we have $\what{T}_{null}(t) \geq \epsilon$. If we assume
    \begin{align*}
        \Vert f_k^* \Vert_{L^2(p+q)}^{2} > \max\Bigg\{ \frac{2 \epsilon \exp\Big( (\epsilon/C^+ - \gamma) / \lambda_k \Big) }{ \exp\Big( (\epsilon/C^+ - \gamma ) / \lambda_k \Big) - 1 },  \max_{a \in \{1, 5/2\} } \frac{2 \epsilon \exp\Big( (\epsilon/C^-)^{1/a} / \lambda_k \Big) }{ \exp\Big( (\epsilon/C^-)^{1/a} / \lambda_k \Big) - 1 } \Bigg\}  ,
    \end{align*}
    where $C^- = C^+ + C_2 + C_3 + C_4$ and the constants $C_2, C_3, C_4$ coming from \Cref{prop:uhatubarBND_informal}, then
    \begin{align*}
        t^+(\epsilon) - t^-(\epsilon) \geq \gamma > 0
    \end{align*}
    with probability
    \begin{align*}
        \geq \begin{cases}
            1 - ( n_p^{-A} + n_q^{-A}) & \what{T}_{pop}(t) \text{ evaluation} \\
            1 - ( n_p^{-A} + n_q^{-A} + m_p^{-A} + m_q^{-A}) & \what{T}_{test}(t) \text{ evaluation} \\
            1 - 2( n_p^{-A} + n_q^{-A}) & \what{T}_{train}(t) \text{ evaluation}
        \end{cases}.
    \end{align*}
\end{corollary}

The proof of \Cref{cor:genMomRes_uhat} is in \Cref{pf:genMomRes_uhat}.  To discuss statistical power, we use the following statement to show that for a particular type 1 error
\begin{align*}
    \P( \what{T}_{null}(t) > \tau \vert H_0) = \alpha
\end{align*}
the statistical power goes to 1, i.e. we have
\begin{align*}
    \P \Big( \what{T}_{test} (t) > \tau \big\vert \pif = f_k^* \Big) \xrightarrow{n_p, n_q, m_p, m_q \to \infty} 1.
\end{align*}
This idea is formalized in the following corollary.

\begin{corollary}[Realistic Test Power Analysis]\label{cor:genMomPower_uhat}
    Let $\Vert \pif \Vert_{L^2(p+q)}^2 / 2 > \epsilon > 0$ be a detection level. Given training and test sample sizes $n_p, n_q, m_p, m_q$, let 
    \begin{align*}
        \widetilde{C} = \sqrt{2}L_1^2 \Bigg( \sqrt{ \frac{A \log(m_p)}{m_p} } + \sqrt{ \frac{A \log(m_q)}{m_q} } + \sqrt{ \frac{A \log(n_p)}{n_p} } + \sqrt{ \frac{A \log(n_q)}{n_q} } \bigg)
    \end{align*}
    and $C^-_{test} = C^+_{test} + C_2 + C_3 + C_4$, where constants $C^+_{test}$ come from \Cref{cor:genMomRes_uhat} and $C_2, C_3, C_4$ come from \Cref{prop:uhatubarBND_informal}.  If we assume
    \begin{align*}
        \Vert f_k^* \Vert_{L^2(p+q)}^{2} > \max_{a \in \{1, 5/2\} } \frac{2 \epsilon \exp\Big( (\epsilon/(C^-_{test} + \widetilde{C}) )^{1/a} / \lambda_k \Big) }{ \exp\Big( (\epsilon/(C^-_{test} + \widetilde{C}) )^{1/a} / \lambda_k \Big) - 1 }  ,
    \end{align*}
    then for
    \begin{align*}
        t_1^*(2\epsilon) \leq t \leq \min_{a \in \{1, 5/2\} } \bigg( \frac{\epsilon}{ C_{test}^- + \widetilde{C} } \bigg)^{1/a},
    \end{align*}
    we have $\what{T}_{test,k}(t) > \epsilon$ and $\what{T}_{test, null}(t) < \epsilon$ with probability $\geq 1 - ( n_p^{-A} + n_q^{-A} + m_p^{-A} + m_q^{-A})$.  Moreover, for a two-sample level $0 < \alpha < 1$ (typically $\alpha = 0.05$) with $\alpha \geq n_p^{-A} + n_q^{-A} + m_p^{-A} + m_q^{-A}$ and test threshold $\tau \leq \epsilon$, we have
    \begin{align*}
        \P \Big( \what{T}_{test} (t) > \tau \big\vert \pif = f_k^* \Big) &\geq 1 - n_p^{-A} + n_q^{-A} + m_p^{-A} + m_q^{-A}.
    \end{align*}
\end{corollary}
The proof is included in \Cref{pf:genMomPower_uhat}.

\begin{remark}
    From the proof of \Cref{cor:genMomRes_uhat}, we can see that the maximum time separation level is governed by
    \begin{align*}
        \frac{\epsilon}{C^+} - \lambda_k \log\bigg( \frac{ \Vert f^*_k \Vert_{L^2(p+q)}^2 }{\Vert f^*_k \Vert_{L^2(p+q)}^2 - 2\epsilon } \bigg).
    \end{align*}
    Notice that if $\epsilon = \Vert f_k^* \Vert_{L^2(p+q)}^2 x$ for some fraction $0 < x < 1$, then we can simplify this expression.  In particular, we see that our expression changes to
    \begin{align*}
        \gamma(x) &= \frac{\Vert f_k^* \Vert_{L^2(p+q)}^2 x }{C^+} - \lambda_k \log\bigg( \frac{ 1 }{1 - 2 x } \bigg) \\
        &= \frac{\Vert f_k^* \Vert_{L^2(p+q)}^2 x }{C^+} - \lambda_k \log\bigg( \frac{ 1/2 }{1/2 - x } \bigg).
    \end{align*}
    From this, we can see that it is necessary that $0 < x < \frac{1}{2}$.  Note that as $x \to 0$, we get $\gamma(x) \to 0$; but as $x \to \frac{1}{2}$, we get $\gamma(x) \to -\infty$.  Since $\gamma(x)$ is not decreasing, we can find a maximum for the time separation $\gamma(x)$.  In particular, we see that
    \begin{align*}
        \gamma'(x) = \frac{\Vert f_k^* \Vert_{L^2(p+q)}^2 }{C^+} - \lambda_k \frac{1/2 - x }{ 1/2 }  \frac{ 1/2 }{(1/2 - x)^2 } = \frac{\Vert f_k^* \Vert_{L^2(p+q)}^2 }{ C^+ } - \lambda_k \frac{ 1}{1/2 - x }.
    \end{align*}
    Notice that $\gamma'(x) = 0$ implies that
    \begin{align*}
        x = \frac{1}{2} - \frac{\lambda_k C^+ }{ \Vert f_k^* \Vert_{L^2(p+q)}^2 }.
    \end{align*}
    Moreover, we note that this is a maximum since
    \begin{align*}
        \gamma''(x) = - \lambda_k \frac{1}{(1/2 - x)^2} < 0.
    \end{align*}
    Obviously, this only makes sense if
    \begin{align*}
        \frac{1}{2} > \frac{\lambda_k C^+ }{ \Vert f_k^* \Vert_{L^2(p+q)}^2 } \iff \Vert f_k^* \Vert_{L^2(p+q)}^2 > 2 \lambda_k C^+.
    \end{align*}
    This means that as long as $f_k^*$ has large enough norm, our neural network two-sample test should be most trustworthy when we observe deviation $\epsilon = \frac{\Vert f_k^* \Vert_{L^2(p+q)}^2 }{2} - \lambda_k C^+$ since that is the deviation level with the maximum time separation between the assumption $\pif = f_k^*$ and the null hypothesis $p = q$.
\end{remark}

\begin{remark}
    It is instructive to note what are fixed parameters versus parameters to be chosen in \Cref{cor:genMomRes_uhat}.  First, notice that the complexity of our neural network determines not only the constants $C^+$, $C^-$, and $\lambda_k$ but also whether or not the assumption $\pif = f^*_k$ holds.  Although $f^* = \frac{p-q}{p+q}$ is fixed inherently from the two-sample test problem, we assume that the complexity of the neural network is fixed at initialization, which fixes these constants, the hypothesis, and how large $\Vert f_k^* \Vert_{L^2(p+q)}^2$ is.  This means that the only choosable parameters are $\epsilon$ and $\gamma$ (which is upper bounded by $\epsilon$).  Moreover, note that the upper bound $\Vert \pif \Vert_{L^2(p+q)}^2 / 2 \geq \epsilon > 0$ is an artifact of side-stepping the time-approximation error from \Cref{prop:testStatUhat_Ubar_bnd_samp}.  In particular, playing around with the proof of \Cref{cor:genMomRes_uhat}, it is possible to get a different bound for $\Vert f_k^* \Vert_{L^2(p+q)}^2$ albeit with the deviation level given by $\epsilon - \delta(t)$ (depending on the evaluation set).
\end{remark}

\section{Experiments}\label{sec:experiments}

We run our neural network two-sample test on two different data-generating processes.  One of the data-generating processes is a characteristically hard two-sample test problem where the datasets $\what{P}$ and $\what{Q}$ come from a Gaussian mixture model.  The second data-generating process only aims to differentiate two multivariate Gaussians from each other.  We scale the neural network complexity in terms of a ratio with respect to the number of samples in the training set.  Additionally, we run permutation tests to find the threshold $\tau$ at the $95$-percentile.  We run around 500 different tests and check whether the test statistic is larger than the threshold found from the 95th percentile.  We now showcase specifics of the data generating process and how the neural network is constructed.

\subsection{Data Generating Process}
Our hard two-sample testing problem is given by setting $P$ and $Q$ both to be Gaussian mixture models given by
\begin{align*}
    P &= \sum_{i=1}^2 \frac{1}{2} \mathcal{N}( \mu_i^h , I_d ) \\
    Q &= \sum_{i=1}^2 \frac{1}{2} \mathcal{N}\Bigg( \mu_i^h , \begin{bmatrix}
    1 & \Delta_i^h & 0_{d-2} \\ \Delta_i^h & 1 & 0_{d-2} \\ 0_{d-2}^\top & 0_{d-2}^\top & I_{d-2}
\end{bmatrix} \Bigg),
\end{align*}
where $\mu_1^h = 0_d$, $\mu_2^h = 0.5 * \mathbf{1}_d$, $\Delta_1^h = 0.5$, and $\Delta_2^h = -0.5$.  For the purposes of testing, we assume that we have balanced sampling of $N$ from each distribution $P$ and $Q$ so that the total number of samples is $2N$.  The number of test samples is typically set to be $M < N$.

\subsection{Neural Network Architecture}

We use a neural network architecture of $L$ layers and a layer-width size of $k$.  We choose the number of parameters in the neural network $kL$ as different ratios of the number of training samples.  For example, if we consider the ratio $0.01$ and $N = 1000$, then  $kL = 10$.  To adhere closely to the setup of the theory, we initialize our neural network symmetrically to make sure at time 0 our neural network returns 0.  We make sure to initialize the neural network weights with the He initialization introduced in \cite{he2015delving} so that the weights are initialized from a random normal distribution with variance $\frac{2}{k}$.  This initialization ensures that our neural network training doesn't result in any exploding or vanishing gradients.  We train the neural network with a learning rate of 0.1.

\subsection{Test Results}

We have attached below a heatmap of the statistical power as a function of the number of epochs as well as the ratio of parameters to samples.\footnote{All the code for producing these plots is on Github at \url{https://github.com/varunkhuran/NTK_Logit}. }  Additionally we attach the evolution of the neural network two-sample test for a particular setting for reference.  The hyperparameters for these tests essentially used a learning rate $\eta = 0.1$, $100$ permutation tests, dimensionality of $d = 20$, training sample size of $N = 6000$ from each of $P$ and $Q$, a testing sample size of $M = 1000$ from each of $P$ and $Q$, and $L = 2$ layers.  We train for a maximum of 15 epochs and use a batch size of 50.  Moreover, our significance level $\alpha$ is the $95$th percentile.  We calculate the power of our neural network two-sample test by checking which of our 1000 tests lie past the 95th percentile of their respective permutation test and calculate the power by the ratio of all tests that lie past the 95th percentile divided by the total number of tests 1000.  We try this experiment with ratios of parameters to number of training samples to see any double descent type of behavior in how well the statistical power performs.

\begin{figure}[h!]
    \centering
    \includegraphics[scale=0.5]{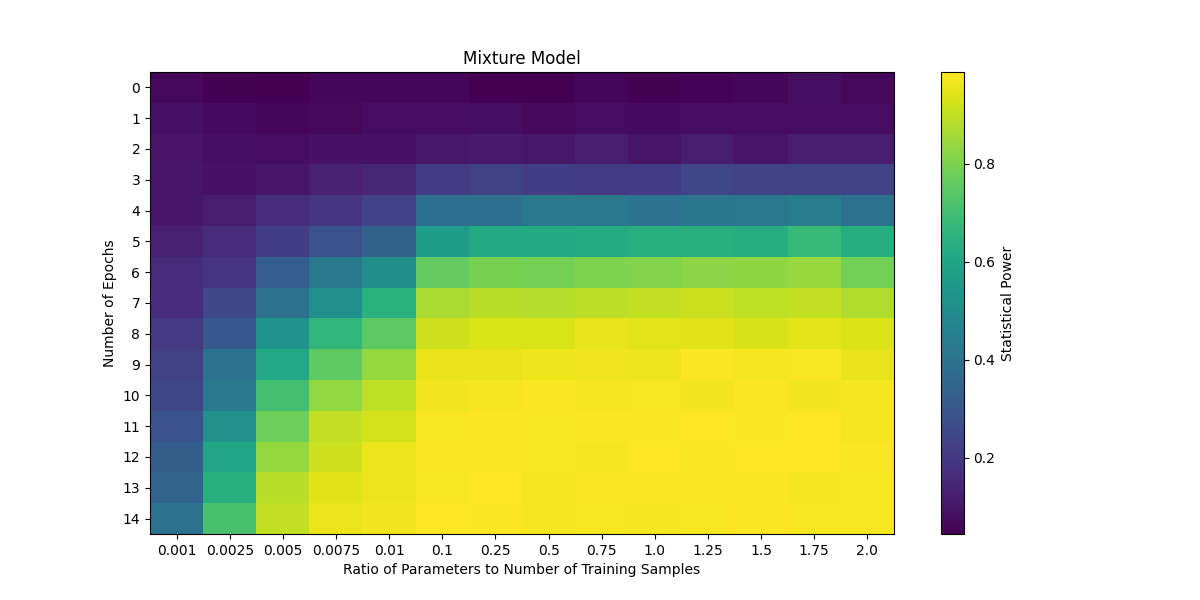}
    \caption{Plots statistical power for each epoch and ratio of parameters-to-samples when we are in the alternative hypothesis case with $P$ and $Q$ as above.}
    \label{fig:power_plot}
\end{figure}

Observing \Cref{fig:power_plot}, we notice that as the number of epochs increases the statistical power increases as well.  On the ratio of parameters to training samples axis, however, we note that the smaller sized neural networks still produce fairly good statistical power with enough neural network training. 
 The case when we are in the null hypothesis with samples drawn from $P$ is shown in \Cref{fig:power_null_plot}.  In the null hypothesis, we see that the statistical power is much lower than in the alternative hypothesis case, as expected.

 \begin{figure}[h!]
    \centering
    \includegraphics[scale=0.5]{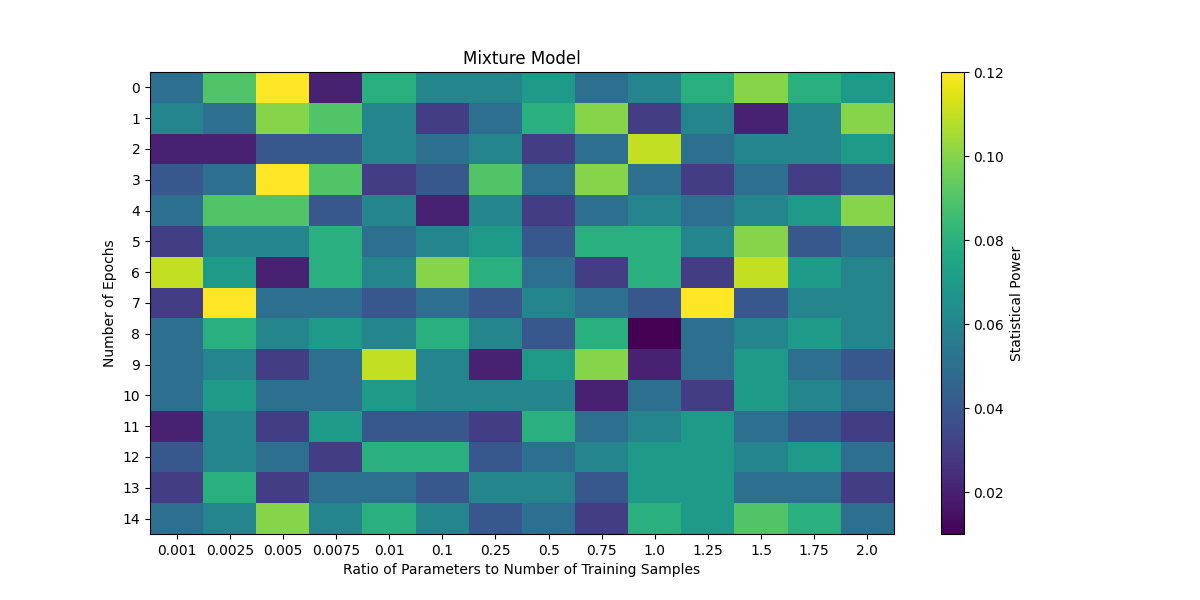}
    \caption{Plots statistical power for each epoch and ratio of parameters-to-samples when we are in the null hypothesis case with samples from $P$.}
    \label{fig:power_null_plot}
\end{figure}

\section*{Acknowledgements}

AC was supported by NSF DMS 2012266, NSF CISE 2403452, and a gift from Intel.
XC is partially supported by 
NSF DMS-2237842, 
NSF DMS-2134037, 
NSF DMS-2007040
and Simons Foundation.

\bibliographystyle{plain}
\bibliography{lit}

\appendix

\section{Proofs for \Cref{subsec:pop0}}

\begin{proof}[Proof of \Cref{prop:diffeqSol}]\label{pf:diffeqSol}
    We simply just need to take the derivative of the solution and show that it is exactly the differential equation, then the uniqueness of solutions of differential equations implies that our ansatz is indeed the solution.  To see this, let us call the solution 
    \begin{align*}
        e^*(\cdot, t) = \sum_{\ell = 1}^M e^{-t \lambda_\ell} \< u_\ell, \bar{e}(\cdot, 0) \>_{L^2(p+q)} u_\ell.
    \end{align*}
    Now notice that
    \begin{align*}
        \partial_t e^*(\cdot, t) = - \sum_{\ell = 1}^M \lambda_\ell e^{-t \lambda_\ell } \< u_\ell, \bar{e}(\cdot, 0) \>_{L^2(p+q)} u_\ell.
    \end{align*}
    On the other hand, let us plug in the ansatz $e^*$ into the differential equation and see what we get.  In particular,
    \begin{align*}
        - \sum_{j = 1}^M \lambda_j \< u_j, e^*(\cdot, t) \>_{L^2(p+q)} u_j &= - \sum_{j = 1}^M \lambda_j \< u_j, \sum_{\ell = 1}^M e^{-t \lambda_\ell} \< u_\ell, \bar{e}(\cdot, 0) \>_{L^2(p+q)} u_\ell \>_{L^2( p+q) }  u_j \\
        &= - \sum_{j, \ell = 1}^M \lambda_j e^{-t \lambda_\ell}  \<u_\ell, \bar{e}(\cdot,0) \>_{L^2(p+q)} \underbrace{ \<u_j, u_\ell \>_{L^2(p+q)} }_{\delta_{j \ell }} u_j \\
        &= - \sum_{\ell = 1}^M \lambda_\ell e^{-t\lambda_\ell} \< u_\ell, \bar{e}(\cdot, 0) \>_{L^2(p+q)} u_\ell \\
        &= \partial_t e^*(\cdot, t).
    \end{align*}
    This shows the result.
\end{proof}

\begin{proof}[Proof of \Cref{lem:T0_form}]\label{pf:T0_form}
    Now our two-sample test statistic becomes
    \begin{align*}
        \overline{T}(t) &= \overline{\mu_P}(t) - \overline{\mu_Q}(t) \\ 
        &= \int_{\R^d} \pif(x) (p-q)(x) dx + \int_{\R^d} \bar{e}(x,t) d(p-q)(x) \\
        &= \int_{\R^d} \pif(x) \frac{p-q}{p+q}(x) d(p+q)(x) + \int_{\R^d} \bar{e}(x,t) d(p-q)(x) \\
        &= \< \pif, f^* \>_{L^2(p+q)} + \int_{\R^d} \bar{e}(x,t) d(p - q)(x).
    \end{align*}
    Extending the eigenfunctions $\{u_\ell\}_{\ell=1}^M$ to a full basis for $L^2(p+q)$ given by $\{u_\ell\}_{\ell=1}^\infty$, we can see that the term with $f^*$ reduces to
    \begin{align*}
        \< \pif, f^* \>_{L^2(p+q)} &= \bigg\< \sum_{\ell=1}^\infty \< u_\ell, \Pi_{K_0} (f^*) \>_{L^2(p+q)} u_\ell , \sum_{\ell'=1}^\infty \<u_\ell', f^* \>_{L^2(p+q)} u_\ell' \bigg\>_{L^2(p+q)} \\
        &= \sum_{\ell=1}^\infty \< u_\ell, \Pi_{K_0} (f^*) \>_{L^2(p+q)} \< u_\ell, f^* \>_{L^2(p+q)} \\
        &= \sum_{\ell=1}^\infty \underbrace{ \< \Pi_{K_0} (u_\ell), f^* \>_{L^2(p+q)}}_{\ell > M \implies 0} \< u_\ell, f^* \>_{L^2(p+q)} \\
        &= \sum_{\ell=1}^M \< \Pi_{K_0} (u_\ell), f^* \>_{L^2(p+q)} \< u_\ell, f^* \>_{L^2(p+q)} \\
        &= \sum_{\ell=1}^M \< u_\ell, \pif \>_{L^2(p+q)} \< u_\ell, f^* \>_{L^2(p+q)}.
    \end{align*}
    Now since $\< u_\ell, \pif \>_{L^2(p+q)} = \< u_\ell,f^* \>_{L^2(p+q)}$ for $\ell \leq M$, we see that
    \begin{align*}
        \< \pif, f^* \>_{L^2(p+q)} &= \sum_{\ell=1}^M \< u_\ell, \pif \>_{L^2(p+q)}^2 \\
        &= \< \pif, \pif \>_{L^2(p+q)} = \Vert \pif \Vert_{L^2(p+q)}^2.
    \end{align*}
    At this point, we plug in our ansatz and get
    \begin{align*}
        \overline{T}(t) &= \Vert \pif \Vert_{L^2(p+q)}^2 + \int_{\R^d} \sum_{\ell = 1}^M e^{-t \lambda_\ell} \<u_\ell, \bar{e}(\cdot,0) \>_{L^2(p+q)} u_\ell(x) d(p - q)(x) \\
        &= \Vert  \pif \Vert_{L^2(p+q)}^2 + \sum_{\ell = 1}^M e^{-t \lambda_\ell} \<u_\ell, \bar{e}(\cdot,0) \>_{L^2(p+q)} \int_{\R^d} u_\ell(x) d(p - q)(x).
    \end{align*}
    At this point, recall that we used the initialization $\theta_0$ such that $\bar{e}(\cdot, 0) = \bar{u}(x,0) -  \pif = f(x,\theta_0) -  \pif = - \pif = \Pi_{K_0}\big( \frac{q-p}{p+q} \big) (x)$.  Along with the fact that $\< u_\ell, \pif \>_{L^2(p+q)} = \< u_\ell,f^* \>_{L^2(p+q)}$ for $\ell \leq M$, we use this fact to see that
    \begin{align*}
        \< u_\ell, \bar{e}(\cdot,0) \>_{L^2(p+q)} &= \< u_\ell, -\pif \>_{L^2(p+q)} \\
        &= \int_{\R^d} u_\ell(x) \Pi_{K_0} \Big(\frac{q-p}{p+q} \Big) (x) ( p(x) + q(x)) dx \\
        &= \int_{\R^d} u_\ell(x) \frac{q-p}{p+q} (x) ( p(x) + q(x)) dx \\
        &= \int_{\R^d} u_\ell(x) d(q-p)(x) \\
        \implies \<u_\ell, \bar{e}(\cdot,0) \>_{L^2(p+q)} &= - \int_{\R^d} u_\ell(x) d(p-q)(x).
    \end{align*}
    This means that
    \begin{align*}
        \overline{T}(t) &= \Vert \pif \Vert_{L^2(p+q)}^2 + \sum_{\ell = 1}^M e^{-t \lambda_\ell} \<u_\ell, \bar{e}(\cdot,0) \>_{L^2(p+q)} \underbrace{\int_{\R^d} u_\ell(x) d(p - q)(x)}_{- \<u_\ell, \bar{e}(\cdot, 0)\>_{L^2(p+q)} } \\
        &= \Vert \pif \Vert_{L^2(p+q)}^2 - \sum_{\ell \geq 1} e^{-t \lambda_\ell} \<u_\ell, \bar{e}(\cdot,0) \>_{L^2(p+q)}^2.
    \end{align*}
    We get the result by seeing that $\bar{e}(\cdot, 0 ) = -\pif $ and that applying the square gets rid of the negative sign. So we're done.
\end{proof}

\begin{proof}[Proof of \Cref{prop:H1thm}]\label{pf:H1thm}
    We want to find the smallest time $t$ so that
    \begin{align*}
        \overline{T}(t) &\geq \epsilon \\
        \Vert \pif \Vert_{L^2(p+q)}^2 - \sum_{\ell = 1}^M e^{-t \lambda_\ell} \<u_\ell, \pif \>_{L^2(p+q)}^2 &\geq \epsilon \\
        \Vert \pif \Vert_{L^2(p+q)}^2 - \epsilon &\geq \sum_{\ell = 1}^M e^{-t \lambda_\ell} \<u_\ell, \pif \>_{L^2(p+q)}^2.
    \end{align*}
    Now using our specific subset $S \subseteq \{1, \dotsc, M\}$, so that
    \begin{align*}
        \<u_\ell, f^*\>_{L^2(p+q)} = \<u_\ell, \pif\>_{L^2(p+q)},
    \end{align*}
    allows us to consider the following analysis.
    \begin{align*}
        \sum_{\ell = 1}^M e^{-t \lambda_\ell} \<u_\ell, f^* \>_{L^2(p+q)}^2 &= \sum_{\ell \in S } \underbrace{e^{-t \lambda_\ell}}_{\leq e^{-t \lambda_{\min}(S) }} \<u_\ell, f^* \>_{L^2(p+q)}^2 + \sum_{\ell \not\in S} \underbrace{e^{-t \lambda_\ell}}_{\leq 1} \<u_\ell, f^* \>_{L^2(p+q)}^2 \\
        &\leq e^{-t\lambda_{\min}(S)} \underbrace{ \sum_{\ell \in S} \< u_\ell, f^* \>_{L^2(p+q)}}_{\Vert f^* \Vert_S^2 } + \underbrace{ \sum_{\ell \not\in S} \<u_\ell, f^*\>_{L^2(p+q)}^2}_{\Vert \pif \Vert_{L^2(p+q)}^2 - \Vert f^* \Vert_{S}^2 } \\
        &= e^{-t \lambda_{\min}(S) } \Vert f^* \Vert_S^2 + \Vert \pif \Vert_{L^2(p+q)}^2 - \Vert f^* \Vert_S^2.
    \end{align*}
    We want this quantity to still be less than $\Vert \pif \Vert_{L^2(p+q)}^2 - \epsilon$ and to ensure this, we get
    \begin{align*}
        e^{-t \lambda_{\min}(S) } \Vert f^* \Vert_S^2 + \Vert \pif \Vert_{L^2(p+q)}^2 - \Vert f^* \Vert_S^2 &\leq \Vert \pif \Vert_{L^2(p+q)}^2 - \epsilon \\
        e^{-t \lambda_{\min}(S) } \Vert f^* \Vert_S^2 - \Vert f^* \Vert_S^2 &\leq - \epsilon \\
        e^{-t \lambda_{\min}(S) } - 1 &\leq - \frac{\epsilon}{\Vert f^* \Vert_S^2 } \\
        e^{-t \lambda_{\min}(S)} &\leq 1 - \frac{\epsilon}{\Vert f^* \Vert_S^2 } \\
        -t \lambda_{\min}(S) &\leq \log\Big( 1 - \frac{\epsilon}{\Vert f^* \Vert_S^2 } \Big) \\
        t &\geq \log\bigg( \Big( 1 - \frac{\epsilon}{\Vert f^* \Vert_S^2 } \Big)^{-\lambda_{\min}(S) } \bigg). 
    \end{align*}
    Rearranging the right-hand side and noticing that on $S$ we have $\Vert \pif \Vert_S^2 = \Vert f^* \Vert_S^2 $, we get the result.
\end{proof}

\begin{proof}[Proof of \Cref{prop:H0thm}]\label{pf:H0thm}
    We want to find the largest time $t$ so that
    \begin{align*}
        \overline{T}(t) &\leq \epsilon \\
        \Vert \pif \Vert_{L^2(p+q)}^2 - \sum_{\ell = 1}^M e^{-t \lambda_\ell} \<u_\ell, \pif \>_{L^2(p+q)}^2 &\leq \epsilon \\
        \Vert \pif \Vert_{L^2(p+q)}^2 - \epsilon &\leq \sum_{\ell \geq 1} e^{-t \lambda_\ell} \<u_\ell, \pif \>_{L^2(p+q)}^2.
    \end{align*}
    Now using our specific subset $S$ and that $\<u_\ell, \pif\>_{L^2(p+q)} = \<u_\ell, f^*\>_{L^2(p+q)}$ for $\ell \leq M$ allows us to consider the following analysis.
    \begin{align*}
        \sum_{\ell = 1}^M e^{-t \lambda_\ell} \<u_\ell, f^* \>_{L^2(p+q)}^2 &= \sum_{\ell \in S } \underbrace{e^{-t \lambda_\ell}}_{\geq e^{-t \lambda_{\max}(S) }} \<u_\ell, f^* \>_{L^2(p+q)}^2 + \sum_{\ell \not\in S} \underbrace{e^{-t \lambda_\ell}}_{\geq 0} \<u_\ell, f^* \>_{L^2(p+q)}^2 \\
        &\geq e^{-t\lambda_{\max}(S)} \underbrace{ \sum_{\ell \in S} \< u_\ell, f^* \>_{L^2(p+q)}}_{\Vert \pif \Vert_S^2 } \\
        &= e^{-t \lambda_{\max}(S) } \Vert \pif \Vert_S^2 .
    \end{align*}
    We want our lower bound found above to still be greater than $\Vert \pif \Vert_{L^2(p+q)}^2 - \epsilon$ and to ensure this, we get
    \begin{align*}
        e^{-t \lambda_{\max}(S) } \Vert \pif \Vert_S^2  &\geq \Vert \pif \Vert_{L^2(p+q)}^2 - \epsilon \\
        e^{-t \lambda_{\max}(S) } &\geq \frac{\Vert \pif \Vert_{L^2(p+q)}^2}{\Vert \pif \Vert_{S}^2} - \frac{\epsilon}{\Vert \pif \Vert_{S}^2} \\
        -t \lambda_{\max}(S) &\geq \log\bigg( \frac{\Vert \pif \Vert_{L^2(p+q)}^2}{\Vert \pif \Vert_{S}^2} - \frac{\epsilon}{\Vert \pif \Vert_{S}^2} \bigg) \\
        t &\leq \log\Bigg\{  \bigg( \frac{\Vert \pif \Vert_{L^2(p+q)}^2}{\Vert \pif \Vert_{S}^2} - \frac{\epsilon}{\Vert \pif \Vert_{S}^2} \bigg)^{-\lambda_{\max}(S)} \Bigg\}. 
    \end{align*}
    Rearranging the right-hand side, we get the result.
\end{proof}

\section{Proofs for Population Dynamics}\label{apdx:pop}

\begin{proof}[Proof of \Cref{lem:paramApprox_pop}]\label{pf:paramApprox_pop}
    Recall that the dynamics of $\theta$ can be written as
    \begin{align*}
        \Dot{\theta}(t) = - \nabla_{\theta} \L(\theta(t)) = - \E_{x \sim p+q} \big[ \nabla_{\theta} f(x,\theta(t)) e(x,t) \big].
    \end{align*}
    Moreover, note that we can write
    \begin{align*}
        \Vert \theta(t) - \theta(0) \Vert &\leq \int_0^t \Vert \Dot{\theta}(s) \Vert ds \leq \int_0^t \Vert \nabla_\theta \L(\theta(s)) \Vert ds \\
        &\leq \sqrt{t} \bigg( \int_0^t \Vert \nabla_\theta \L(\theta(s)) \Vert^2 ds \bigg)^{1/2},
    \end{align*}
    where the last inequality comes from the basic $L_p$-$L_q$ inclusion inequality. Additionally, we know that 
    \begin{align*}
        \frac{d}{dt} \L(\theta(t)) = \< \nabla_\theta \L(\theta(t)), \Dot{\theta}(t) \> = - \Vert \nabla_\theta \L(\theta(t)) \Vert^2 \leq 0.
    \end{align*}
    This not only implies that $\L(\theta(t))$ is decreasing but also allows us to write
    \begin{align*}
        \Vert \theta(t) - \theta(0) \Vert &\leq \sqrt{t} \bigg( \int_0^t \Vert \nabla_\theta \L(\theta(s)) \Vert^2 ds \bigg)^{1/2} \\
        &= \sqrt{t} \Big( \L(\theta(0)) - \L(\theta(t)) \Big)^{1/2} \\
        &\leq \sqrt{t} \sqrt{\L(\theta(0))}.
    \end{align*}
    At this point, we can notice that $\L(\theta(0)) = \Vert u(x,0) - f^* \Vert_{L^2(p+q)}^2 = \Vert f^* \Vert_{L^2(p+q)}^2$.  This finally gives us the result
    \begin{align*}
        \Vert \theta(t) - \theta(0) \Vert &\leq \sqrt{t} \Vert f^* \Vert_{L^2(p+q)}.
    \end{align*}
    We need $\theta(t) \in B_R$ and one way to ensure this is
    \begin{align*}
        \Vert \theta(t) - \theta(0) \Vert &\leq \sqrt{t} \Vert f^* \Vert_{L^2(p+q)} \leq R \\
        \implies t &\leq \bigg( \frac{R }{\Vert f^* \Vert_{L^2(p+q)} } \bigg)^2
    \end{align*}
\end{proof}

\begin{proof}[Proof of \Cref{lem:KtK0bd}]\label{pf:KtK0bd}
    Let $w \in L^2(p+q)$, then note that for our kernel integral operator $K_t$, we have $ \< w, \E_{x'\sim p+q} K_t(\cdot, x') w(x') \>_{L^2(p+q)}$ equals
    \begin{align*}
        \E_{x \sim p+q } \E_{x' \sim p+q} w(x) \< \nabla_{\theta} f(x,\theta(t)), \nabla_{\theta} f(x', \theta(t)) \>_{\Theta} w(x') \\
        = \< \E_{x \sim p+q } \nabla_{\theta} f(x,\theta(t)) w(x) ,  \E_{x' \sim p+q} \nabla_{\theta} f(x', \theta(t)) w(x') \>_{\Theta} \\
        = \Vert \E_{x \sim p+q } \nabla_{\theta} f(x,\theta(t)) w(x) \Vert_{\Theta}^2.
    \end{align*}
    This means that $\<w, \E_{x'\sim p+q} (K_t - K_0)(\cdot, x') w(x') \>_{L^2(p+q)}$ equals
    \begin{align*}
         \Vert \E_{x \sim p+q } \nabla_{\theta} f(x,\theta(t)) w(x) \Vert_{\Theta}^2 - \Vert \E_{x \sim p+q } \nabla_{\theta} f(x,\theta(0)) w(x) \Vert_{\Theta}^2 \\
        = (\Vert \E_{x \sim p+q } \nabla_{\theta} f(x,\theta(t)) w(x) \Vert_{\Theta} + \Vert \E_{x \sim p+q } \nabla_{\theta} f(x,\theta(0)) w(x) \Vert_{\Theta}) \\
        \cdot (\Vert \E_{x \sim p+q } \nabla_{\theta} f(x,\theta(t)) w(x) \Vert_{\Theta} - \Vert \E_{x \sim p+q } \nabla_{\theta} f(x,\theta(0)) w(x) \Vert_{\Theta}).
    \end{align*}
    Now using Minkowski's integral inequality and \Cref{C1}(1), we get 
    \begin{align*}
        \Vert \E_{x \sim p+q } \nabla_{\theta} f(x,\theta(t)) w(x) \Vert_{\Theta} &\leq \E_{x \sim p+q } \Vert \nabla_{\theta} f(x,\theta(t)) w(x) \Vert_{\Theta} \\
        &\leq \E_{x \sim p+q } \Vert \nabla_{\theta} f(x,\theta(t)) \Vert_{\Theta} \Vert w(x) \Vert \\
        &\leq L_1 \E_{x \sim p+q } \Vert w(x) \Vert \\
        &\leq L_1 \Vert w \Vert_{L^2(p+q)}.
    \end{align*}
    Notice that this implies
    \begin{align*}
        \Vert \E_{x \sim p+q } \nabla_{\theta} f(x,\theta(t)) w(x) \Vert_{\Theta} + \Vert \E_{x \sim p+q } \nabla_{\theta} f(x,\theta(0)) w(x) \Vert_{\Theta} \leq 2L_1 \Vert w \Vert_{L^2(p+q)}.
    \end{align*}
    Now using \Cref{C1}(2), we have
    \begin{align*}
        \Vert \E_{x \sim p+q } \nabla_{\theta} f(x,\theta(t)) w(x) \Vert_{\Theta} - \Vert \E_{x \sim p+q } \nabla_{\theta} f(x,\theta(0)) w(x) \Vert_{\Theta} \\
        \leq \Vert \E_{x \sim p+q } (\nabla_{\theta} f(x,\theta(t)) - \nabla_{\theta} f(x,\theta(0)) ) w(x) \Vert_{\Theta} \\
        \leq \E_{x \sim p+q } \Vert \nabla_{\theta} f(x,\theta(t)) - \nabla_{\theta} f(x,\theta(0)) \Vert_{\Theta} \Vert w(x) \Vert \\
        \leq L_2 \Vert \theta(t) - \theta(0) \Vert_{\Theta} \Vert w \Vert_{L^2(p+q)} \\
        \leq L_2 \sqrt{t} \Vert f^* \Vert_{L^2(p+q)} \Vert w \Vert_{L^2(p+q)}.
    \end{align*}
    This means that
    \begin{align*}
        \< w, \E_{x'\sim p+q} K_t(\cdot, x') w(x') \>_{L^2(p+q)} \leq 2 L_1 L_2 \sqrt{t} \Vert f^* \Vert_{L^2(p+q)} \Vert w \Vert_{L^2(p+q)}^2.
    \end{align*}
    Finally this proves that $\Vert K_t - K_0 \Vert_{L^2(p+q)} \leq 2 L_1 L_2 \sqrt{t} \Vert f^* \Vert_{L^2(p+q)}$.
\end{proof}

\begin{proof}[Proof of \Cref{prop:uubarBND}]\label{pf:uubarBND}
    Note that
    \begin{align*}
        \partial_t (u - \Bar{u})(x,t) = \partial_t (e - \Bar{e})(x,t) = \E_{x' \sim p+q}\big[ K_0(x,x') \Bar{e}(x',t) - K_t(x,x') e(x',t) \big] \\
        = -\E_{x' \sim p+q}\big[ (K_t(x,x') - K_0(x,x')) \Bar{e}(x',t) + K_t(x,x')( e - \Bar{e})(x',t) \big].
    \end{align*}
    Notice here that because $K_t(x,x') = \< \nabla_\theta f(x,\theta(t)), \nabla_\theta f(x', \theta(t)) \>$, we have that $K_t$ is a positive semi-definite operator (we will use this later). Now if we take an inner product with $e -\Bar{e}$ on both sides of the equation, we get
    \begin{align*}
        \frac{d}{dt} \frac{1}{2} \Vert (e - \Bar{e})(\cdot, t) \Vert_{L^2(p+q)}^2 &= \< (e - \Bar{e})(\cdot, t), \partial_t (e - \Bar{e})(\cdot, t) \>_{L^2(p+q)} \\
        &= \< (e - \Bar{e})(\cdot,t), -\E_{x' \sim p+q}\big[ (K_t(x,x') - K_0(x,x')) \Bar{e}(x',t) \\
        &+ K_t(x,x')( e - \Bar{e})(x',t) \big]\>_{L^2(p+q)} \\
        &\leq \vert \< (e - \Bar{e})(\cdot, t), -\E_{x' \sim p+q}\big[ (K_t(x,x') - K_0(x,x')) \Bar{e}(x',t)\big]\>_{L^2(p+q)} \vert \\
        &\leq \Vert (e - \Bar{e})(\cdot,t) \Vert_{L^2(p+q)} \Vert K_t - K_0 \Vert_{L^2(p+q)} \Vert \Bar{e}(\cdot,t) \Vert_{L^2(p+q)} \\
        &\leq \Vert (e - \Bar{e})(\cdot,t) \Vert_{L^2(p+q)} \Vert K_t - K_0 \Vert_{L^2(p+q)} \Vert \pif \Vert_{L^2(p+q)},
    \end{align*}
    where the first inequality comes from the fact that $K_t$ is a positive semi-definite operator as well as using absolute values whilst the second inequality comes from using the Cauchy-Schwartz-Bunyakovsky inequality along with the kernel integral operator norm bound of $\Vert K_t - K_0 \Vert_{L^2(p+q)}$.  Now recalling that $\bar{e}(\cdot, 0) = \pif$ and using Parseval's identity, the last inequality comes from the fact that
    \begin{align*}
        \Vert \bar{e}(\cdot, t) \Vert_{L^2(p+q)}^2 &=  \sum_{\ell = 1}^M \underbrace{ e^{-2t \lambda_\ell } }_{ \leq 1} \vert \< u_\ell,  \bar{e}(\cdot,0) \> \vert^2  \leq   \sum_{\ell = 1}^M \vert \< u_\ell, \pif \> \vert^2 = \Vert \pif \Vert_{L^2(p+q)}^2.
    \end{align*}
    From \Cref{lem:KtK0bd}, we get that
    \begin{align*}
        \frac{d}{dt} \frac{1}{2} \Vert (e - \Bar{e})(\cdot,t) \Vert_{L^2(p+q)}^2 &\leq 2 L_1 L_2 \sqrt{t} \Vert f^* \Vert_{L^2(p+q)} \cdot \Vert \pif \Vert_{L^2(p+q)} \cdot \Vert (e - \Bar{e})(\cdot,t) \Vert_{L^2(p+q)} .
    \end{align*}
    Now, finally notice that
    \begin{align*}
        \frac{1}{2} \Vert (e - \Bar{e})(\cdot, t) \Vert_{L^2(p+q)}^2 &= \int_0^t \Big( \frac{d}{dt} \frac{1}{2} \Vert (e - \Bar{e})(\cdot,s) \Vert_{L^2(p+q)}^2 \Big) ds \\
        &\leq 2 L_1 L_2 \Vert f^* \Vert_{L^2(p+q)} \Vert \pif \Vert_{L^2(p+q)}  \int_0^t \sqrt{s} \Vert (e - \Bar{e})(\cdot,s) \Vert_{L^2(p+q)} ds.
    \end{align*}
    Let $t^* \leq t$ be the time such that
    \begin{align*}
        \sup_{s \in [0,t]} \Vert (e-\Bar{e})(\cdot, s) \Vert_{L^2(p+q)} = \Vert (e-\Bar{e})(\cdot, t^*) \Vert_{L^2(p+q)}.
    \end{align*}
    Then, we find that
    \begin{align*}
        \frac{1}{2} \Vert (e - \Bar{e})(\cdot, t^*) \Vert_{L^2(p+q)}^2 &\leq 2 L_1 L_2 \Vert f^* \Vert_{L^2(p+q)} \Vert \pif \Vert_{L^2(p+q)}  \Vert (e - \Bar{e})(\cdot,t^*) \Vert_{L^2(p+q)} \int_0^{t^*} \sqrt{s} ds,
    \end{align*}
    but this implies
    \begin{align*}
        \frac{1}{2} \Vert (e - \Bar{e})(\cdot, t^*) \Vert_{L^2(p+q)} &\leq 2 L_1 L_2 \Vert f^* \Vert_{L^2(p+q)} \Vert \pif \Vert_{L^2(p+q)} \frac{2}{3} (t^*)^{3/2} \\
        \implies \Vert (e - \Bar{e})(\cdot, t) \Vert_{L^2(p+q)} &\leq 4 L_1 L_2 \Vert f^* \Vert_{L^2(p+q)} \Vert \pif \Vert_{L^2(p+q)}  \frac{2}{3} (t^*)^{3/2} \\
        &\leq (8/3) L_1 L_2 \Vert f^* \Vert_{L^2(p+q)} \Vert \pif \Vert_{L^2(p+q)}  (t)^{3/2}, 
    \end{align*}
    where we use the fact that $\Vert (e-\Bar{e})(\cdot, t) \Vert_{L^2(p+q)} \leq \Vert (e-\Bar{e})(\cdot, t^*) \Vert_{L^2(p+q)}$ and $t^* \leq t$.  Finally using that $\Vert \pif \Vert_{L^2(p+q)} \leq \Vert f^* \Vert_{L^2(p+q)}$ gives the result, so we're done.
\end{proof}

\begin{proof}[Proof of \Cref{cor:testStatU_Ubar_bnd}]\label{pf:testStatU_Ubar_bnd}
    Notice that
    \begin{align*}
         \Big\vert T(t) - \overline{T}(t) \Big\vert &= \bigg\vert \int_{\R^d} ( u - \bar{u} )( x, t) d(p-q)(x) \bigg\vert \\
         &\leq \bigg\vert \int_{\R^d} ( u - \bar{u} )( x, t) d(p)(x) \bigg\vert + \bigg\vert \int_{\R^d} ( u - \bar{u} )( x, t) d(q)(x) \bigg\vert \\
         &\leq \int_{\R^d} \vert ( u - \bar{u} )( x, t) \vert d(p)(x) + \int_{\R^d} \vert ( u - \bar{u} )( x, t) \vert  d(q)(x) \\
         &= \int_{\R^d} \vert ( u - \bar{u} )( x, t) \vert d(p+q)(x) \\
         &\leq \sqrt{2} \Vert u - \bar{u} \Vert_{L^2(p+q)},
    \end{align*}
    where the last inequality comes from using a basic $L_1$-$L_2$ inclusion inequality. Now using \Cref{prop:uubarBND}, we get the result, and we're done.
\end{proof}

\begin{proof}[Proof of \Cref{lem:stat_Monotone}]\label{pf:stat_Monotone}
    Note that the loss $\L(\theta(t)) = \Vert f(\cdot, \theta(t)) - f^* \Vert_{L^2(p+q)}^2$ is monotonically decreasing since
    \begin{align*}
        \frac{d}{dt} \L(\theta(t)) = \< \nabla_\theta \L(\theta(t)), \Dot{\theta}(t) \> = - \Vert \nabla_\theta \L(\theta(t)) \Vert^2 \leq 0.
    \end{align*}
    So we have that $\L(\theta(s)) \geq \L(\theta(t))$ if $0 \leq s \leq t \leq \tau$.  Writing out the loss as $\L(\theta(s)) = \Vert f(\cdot,\theta(s)) \Vert_{L^2(p+q)}^2 - 2 \< f(\cdot, \theta(s)), f^* \rangle_{L^2(p+q)} + \Vert f^* \Vert_{L^2(p+q)}^2$, we can see
    \begin{align*}
        \L(\theta(s)) &\geq \L(\theta(t)) \\
        \Vert f(\cdot,\theta(s)) \Vert_{L^2(p+q)}^2 - 2 \< f(\cdot, \theta(s)), f^* \>_{L^2(p+q)} &\geq \Vert f(\cdot,\theta(t)) \Vert_{L^2(p+q)}^2 - 2 \< f(\cdot, \theta(t)), f^* \>_{L^2(p+q)} \\
        \< f(\cdot, \theta(t)) - f(\cdot, \theta(s)), f^* \>_{L^2(p+q)} & \geq \vert f(\cdot, \theta(t)) \Vert_{L^2(p+q)}^2 - \Vert f(\cdot, \theta(s)) \Vert_{L^2(p+q)}.
    \end{align*}
    Notice that because $u(x,t) = f(x, \theta(t))$ and we assume that $\Vert u(x,t) \Vert_{L^2(p+q)}^2$ is increasing in time, we know that
    \begin{align*}
        \< f(\cdot, \theta(t)) - f(\cdot, \theta(s)), f^* \>_{L^2(p+q)} & \geq \underbrace{\vert f(\cdot, \theta(t)) \Vert_{L^2(p+q)}^2 - \Vert f(\cdot, \theta(s)) \Vert_{L^2(p+q)}}_{\geq 0} \\
        \int_{\R^d} \big( u(x,t) - u(x,s) \big) \frac{p-q}{p+q}(x) d(p+q)(x) &\geq 0 \\
        \int_{\R^d} \big( u(x,t) - u(x,s) \big) d(p-q)(x) &\geq 0 \\
        \int_{\R^d}  u(x,t) d(p-q)(x) &\geq \int_{\R^d} u(x,s)  d(p-q)(x).
    \end{align*}
    So we see that on the interval $[0,T]$, the two-sample test statistic $T(t)$ is monotonically increasing.
\end{proof}

\begin{proof}[Proof of \Cref{thm:u_H1}]\label{pf:u_H1}
    Using the assumption
    \begin{align*}
        \min\bigg( \tau, \Big( \frac{R }{ \Vert f^* \Vert_{L^2(p+q)} } \Big)^2 \bigg) \geq t \geq t_{1}^*(\epsilon)
    \end{align*}
    allows us to use \Cref{cor:testStatU_Ubar_bnd}, \Cref{timeAnalysisCor}, and \Cref{lem:stat_Monotone} simultaneously.  Using monotonicity and the reverse triangle inequality shows that
    \begin{align*}
        \vert T(t) \vert &\geq \vert T(t_1^*(\epsilon)) \vert \geq \Big\vert \overline{T}(t_1^*(\epsilon)) \Big\vert - \Big\vert T(t_1^*(\epsilon)) - \overline{T}(t_1^*(\epsilon)) \Big\vert  \bigg\vert \\
        &\geq \epsilon - \frac{8 \sqrt{2}}{3} \Vert f^* \Vert_{L^2(p+q)}^2 L_1 L_2 \big( t_1^*(\epsilon) \big)^{3/2}
    \end{align*}
    where we can get rid of the absolute values by assumption.  So we're done.
\end{proof}

\begin{proof}[Proof of \Cref{thm:u_H0}]\label{pf:u_H0}
    Because of the assumption on $t$, we can use both \Cref{timeAnalysisCor} as well as \Cref{cor:testStatU_Ubar_bnd}.  Using the triangle inequality gives us
    \begin{align*}
        \vert T(t) \vert &\leq \Big\vert \overline{T}(t) \Big\vert + \Big\vert T(t) - \overline{T}(t) \Big\vert \leq \epsilon + \frac{8 \sqrt{2}}{3} \Vert f^* \Vert_{L^2(p+q)}^2 L_1 L_2 t \\
        &\leq \epsilon + \frac{8 \sqrt{2}}{3} \Vert f^* \Vert_{L^2(p+q)}^2 L_1 L_2 \big( t_2^*(\epsilon) \big)^{3/2} .
    \end{align*}
    So we're done.
\end{proof}

\section{Proofs for Realistic Dynamics}\label{apdx:real}

\begin{proof}[Proof of \Cref{lem:thetahatBD}]\label{pf:thetahatBD}
    Similar to the proof of \Cref{lem:paramApprox_pop}, we recall that
    \begin{align*}
        \Dot{\what{\theta}}(t) &= - \frac{1}{2} \bigg( \int_{\R^d}  \nabla_\theta f(x,\what{\theta}(t)) \Big( \big( f(x,\what{\theta}(t)) - 1 \big) \what{p}(x) + \big( f(x,\what{\theta}(t)) + 1 \big) \what{q}(x) \Big) dx \bigg) \\
        &= - \nabla_\theta \what{L}( \what{\theta}(t) ).
    \end{align*}
    Moreover, recall that
    \begin{align*}
        \frac{d}{dt} \what{L}(\what{\theta}(t)) = \< \nabla_\theta \what{L}(\what{\theta}(t)), \Dot{\what{\theta}}(t) \>_{\Theta} = - \Vert \nabla_\theta \what{L}(\what{\theta}(t)) \Vert_{\Theta} \leq 0.
    \end{align*}
    This already shows that $\what{L}(\what{\theta}(t)) \leq \what{L}(\theta(0))$.  Now, notice that
    \begin{align*}
        \Vert \what{\theta}(t) - \theta(0) \Vert_{\Theta} &\leq \int_0^t \Vert \Dot{\what{\theta}}(s) \Vert_\Theta ds \leq \sqrt{t} \bigg( \int_0^t \Vert \Dot{\what{\theta}}(s) \Vert_\Theta^2 ds \bigg)^{1/2} \\
        &\leq \sqrt{t} \bigg( \int_0^t \Vert \nabla_\theta \what{L}(\what{\theta}(s)) \Vert_\Theta^2 ds \bigg)^{1/2} \\
        &= \sqrt{t} \sqrt{ \what{L}(\what{\theta}(0)) - \what{L}(\what{\theta}(t)) } \leq \sqrt{t} \sqrt{ \what{L}(\what{\theta}(0))}.
    \end{align*}
    Now because $f(\cdot, \theta(0)) = f(\cdot, \what{\theta}(0)) = 0$, we get that
    \begin{align*}
        \what{L}(\theta(0)) &= \frac{1}{2} \bigg( \int_{\R^d} \big( f(x,\theta(0)) - 1 \big)^2 \what{p}(x) dx + \int_{\R^d} \big( f(x,\theta(0)) + 1 \big)^2 \what{q}(x) dx \bigg) \\
        &= \frac{1}{2} \bigg( \int_{\R^d} \what{p}(x) dx + \int_{\R^d} \what{q}(x) dx \bigg) = 1.
    \end{align*}
    So this implies that
    \begin{align*}
        \Vert \what{\theta}(t) - \theta(0) \Vert_\Theta \leq \sqrt{t}.
    \end{align*}
    For the second statement, just notice that we want to ensure $\Vert \what{\theta}(t) - \theta(0) \Vert_\Theta \leq R$.  With our bounds, this is ensured if
    \begin{align*}
        \sqrt{t} \leq R .
    \end{align*}
    Readjusting this expression gives us the result.
\end{proof}

\begin{proof}[Proof of \Cref{lem:matBernNN}]\label{pf:matBernNN}
      Notice that $\E X_i = 0$ and 
      \begin{align*}
          \Vert X_i \Vert \leq \Vert \nabla_\theta f(x_i, \theta) \nabla_\theta f(x_i, \theta)^\top \Vert + \E_{x\sim p} \Vert \nabla_\theta f(x, \theta) \nabla_\theta f(x, \theta)^\top \Vert  \leq 3 L_1^2.
      \end{align*}
      Moreover, notice that
      \begin{align*}
          \Vert \frac{1}{n} \sum_{i=1}^n \E X_i X_i^\top \Vert \leq \frac{1}{n} \sum_{i=1}^{n} \Vert \underbrace{ \E X_i X_i^\top}_{I} \Vert.
      \end{align*}
      Simplifying $I$, we see
      \begin{align*}
          I &= \E_{x_i \sim p} \nabla_\theta f(x_i,\theta) \nabla_\theta f(x_i,\theta)^\top  \Vert \nabla_\theta f(x_i, \theta) \Vert^2 \\
          &- 2(\E_{x_i \sim p} \nabla_\theta f(x_i,\theta) \nabla_\theta f(x_i,\theta)^\top)^2 \\
          &+ (\E_{x_i \sim p} \nabla_\theta f(x_i,\theta) \nabla_\theta f(x_i,\theta)^\top)^2.
      \end{align*}
      This means that
      \begin{align*}
          I &= \E_{x_i \sim p} \nabla_\theta f(x_i,\theta) \nabla_\theta f(x_i,\theta)^\top  \Vert \nabla_\theta f(x_i, \theta) \Vert^2 - (\E_{x_i \sim p} \nabla_\theta f(x_i,\theta) \nabla_\theta f(x_i,\theta)^\top)^2,
      \end{align*}
      which implies
      \begin{align*}
          \Vert I \Vert &\leq \E_{x_i \sim p} \underbrace{\Vert \nabla_\theta f(x_i,\theta) \nabla_\theta f(x_i,\theta)^\top \Vert  \Vert \nabla_\theta f(x_i, \theta) \Vert^2}_{\leq L_1^4} \\
          &+ \underbrace{\Vert \E_{x_i \sim p} \nabla_\theta f(x_i,\theta) \nabla_\theta f(x_i,\theta)^\top \Vert^2}_{L_1^4} \\
          &\leq 2 L_1^4.
      \end{align*}
      Since $X_i$ is symmetric, we know the same bound holds for $X_i^\top X_i$ terms.  This means that $\nu = 2L_1^4$.  Finally, using \Cref{matBern} and cleaning some terms, we get that
      \begin{align*}
          \text{Pr}\Big[ \Vert \frac{1}{n} \sum_{i=1}^n X_i \Vert \geq t \Big] \leq 2 M_\Theta \exp\bigg\{ - \frac{n t^2}{2L_1^2( 2 L_1^2 + t) } \bigg\}.
      \end{align*}
      Let us consider when
      \begin{align*}
          t = \sqrt{ 2L_1^2 (2L_1^2 + 3/2)  \frac{  A \log(n) + \log(2 M_\Theta)   }{n} }.
      \end{align*}
      Moreover, we can choose $n$ large enough such that
      \begin{align*}
          \sqrt{ 2L_1^2 (2L_1^2 + 3/2)  \frac{  A \log(n) + \log(2 M_\Theta)   }{n} } < \frac{3}{2},
      \end{align*}
      then we have that
      \begin{align*}
          \frac{ nt^2 }{ 2L_1^2(2L_1^2 + t) } &> \frac{ nt^2 }{ 2L_1^2(2L_1^2 + (3/2) ) } \\
         - \frac{ nt^2 }{ 2L_1^2(2L_1^2 + t) } &<  -\frac{ nt^2 }{ 2L_1^2(2L_1^2 + (3/2) ) } \\
        \exp\bigg\{ - \frac{ nt^2 }{ 2L_1^2(2L_1^2 + t) } \bigg\} &< \exp\bigg\{ - \frac{ nt^2 }{ 2L_1^2(2L_1^2 + (3/2) ) }\bigg\} \\
        \implies 2 M_\Theta \exp\bigg\{ - \frac{n t^2}{2L_1^2( 2 L_1^2 + t) } \bigg\} &\leq 2 M_\Theta \exp\bigg\{ - \frac{ nt^2 }{ 2L_1^2(2L_1^2 + (3/2) ) }\bigg\}.
      \end{align*}
      So with our choice of $t$, we actually get that
      \begin{align*}
        2 M_\Theta \exp\bigg\{ - \frac{ nt^2 }{ 2L_1^2(2L_1^2 + (3/2) ) }\bigg\} \\
        = 2 M_\Theta \exp\bigg\{ - \frac{ (2L_1^2(2L_1^2 + (3/2) ))(A \log(n) + \log(2M_{\Theta})) }{ 2L_1^2(2L_1^2 + (3/2) ) }\bigg\} \\
        = 2 M_\Theta \exp\bigg\{ - (A \log(n) + \log(2M_{\Theta})) \bigg\} \\
        = n^{-A}.
      \end{align*}
      Taking the compliment of this event, we get that with probability greater than $1 - n^{-A}$
      \begin{align*}
          \Vert \frac{1}{n} \sum_{i=1}^n X_i \Vert \leq \sqrt{ 2L_1^2 (2L_1^2 + 3/2)  \frac{  A \log(n) + \log(2 M_\Theta)   }{n} }.
      \end{align*}
      So we're done.
\end{proof}

\begin{proposition}\label{uhatubarBND}
    Assume that $t \leq R^2$ (so that $\theta(t) \in B_R$) as well as \Cref{C1} and \Cref{C2}, then with probability $\geq 1 - n_p^{-A} - n_q^{-A}$, we have $\Vert (\what{u}-\Bar{u})(\cdot,t) \Vert_{L^2(p+q)}$ is less than or equal to
    \begin{align*}
        4 L_1^2 t &+  4 L_1 L_2 t^{3/2}  \Vert f^* \Vert_{L^2(p+q)} \Big( 1 + 2 L_1^3 t^2 \sqrt{2(2L_1^2 + 3/2) \frac{A \log(n_p) + \log(2 M_\Theta)}{n_p} } \Big)^{1/2} \\
        &+ 4 L_1 L_2 t^{3/2}  \Vert f^* \Vert_{L^2(p+q)} \Big( 1 + 2 L_1^3 t^2 \sqrt{2(2L_1^2 + 3/2) \frac{A \log(n_q) + \log(2 M_\Theta)}{n_q} } \Big)^{1/2} \\
        &+ t^2 \cdot \sqrt{2}L_1^3 \Vert f^* \Vert_{L^2(p+q)} \sqrt{ 2L_1^2 (2L_1^2 + 3/2)  \frac{  A \log(n_p) + \log(2 M_\Theta)   }{n_p} }\\
        &+ t^2 \cdot \sqrt{2}L_1^3 \Vert f^* \Vert_{L^2(p+q)} \sqrt{ 2L_1^2 (2L_1^2 + 3/2)  \frac{  A \log(n_q) + \log(2 M_\Theta)   }{n_p} }.
    \end{align*}
    So that $\Vert (\what{u}-\Bar{u})(\cdot,t) \Vert_{L^2(p+q)}$ is $O(t^{5/2})$.
\end{proposition}
\begin{proof}[Proof of \Cref{uhatubarBND}]\label{pf:uhatubarBND}
    Inspecting $\partial_t(\what{u} - \Bar{u})$ more closely, we see that
    \begin{align*}
        2 \partial_t( \what{u} - \Bar{u})(\cdot, t) &= -\E_{x' \sim \what{p}} \what{K}_t(\cdot, x') \what{e}_p(x', t) - \E_{x' \sim \what{q}} \what{K}_t(\cdot, x') \what{e}_q(x',t) \\
        &+ \E_{x' \sim p} K_0(\cdot,x') \Bar{e}(x',t) + \E_{x' \sim q} K_0(\cdot,x') \Bar{e}(x',t).
    \end{align*}
    Notice that
    \begin{align*}
        -\E_{x' \sim \what{p}} \what{K}_t(\cdot, x') \what{e}_p(x', t) + \E_{x' \sim p} K_0(\cdot,x') \Bar{e}(x',t) \\
        = -\bigg\{ \E_{x' \sim \what{p}} \what{K}_t(\cdot, x')(\what{e}_p(x', t) - \bar{e}(x',t)) + \E_{x' \sim \what{p}} (\what{K}_t - K_0)(\cdot,x') \Bar{e}(x',t) \\
        + \Big( \E_{x' \sim \what{p}} - \E_{x' \sim p} \Big) K_0(\cdot, x') \bar{e}(x', t) \bigg\}.
    \end{align*}
    For $q$, we get a similar form
    \begin{align*}
        -\E_{x' \sim \what{q}} \what{K}_t(\cdot, x') \what{e}_q(x', t) + \E_{x' \sim q} K_0(\cdot,x') \Bar{e}(x',t) \\
        = -\bigg\{ \E_{x' \sim \what{q}} \what{K}_t(\cdot, x')(\what{e}_q(x', t) - \bar{e}(x',t)) + \E_{x' \sim \what{q}} (\what{K}_t - K_0)(\cdot,x') \Bar{e}(x',t) \\
        + \Big( \E_{x' \sim \what{q}} - \E_{x' \sim q} \Big) K_0(\cdot, x') \bar{e}(x', t) \bigg\}.
    \end{align*}
    Putting this together, we get
    \begin{align*}
        2 \partial_t( \what{u} - \Bar{u})(\cdot, t) = \underbrace{- \E_{x' \sim \what{p}} \what{K}_t(\cdot, x')(\what{e}_p(x', t) - \bar{e}(x',t))}_{I_{1,p}} \\
        + \underbrace{ \E_{x' \sim \what{p}} (K_0-\what{K}_t)(\cdot,x') \Bar{e}(x',t) }_{I_{2,p}} + \underbrace{ \Big( \E_{x' \sim p} - \E_{x' \sim \what{p}} \Big) K_0(\cdot, x') \bar{e}(x', t) }_{I_{3,p} } \\
        \underbrace{ - \E_{x' \sim \what{q}} \what{K}_t(\cdot, x')(\what{e}_q(x', t) - \bar{e}(x',t)) }_{I_{1,q} } + \underbrace{ \E_{x' \sim \what{q}} (K_0 - \what{K}_t )(\cdot,x') \Bar{e}(x',t) }_{I_{2,q} } \\
        + \underbrace{ \Big( \E_{x' \sim q} - \E_{x' \sim \what{q}} \Big) K_0(\cdot, x') \bar{e}(x', t) }_{ I_{3,q} }.
    \end{align*}
    Similar to the proof of \Cref{prop:uubarBND}, we will consider
    \begin{align*}
        \frac{d}{dt} \Vert (\what{u}-\Bar{u})(\cdot,t) \Vert_{L^2(p+q)}^2 &= \< (\what{u} - \Bar{u})(\cdot,t), 2 \partial_t (\what{u} - \Bar{u})(\cdot,t) \>_{L^2(p+q)} \\
        &= \< (\what{u} - \Bar{u})(\cdot,t), I_{1,p} + I_{1,q} \>_{L^2(p+q)} + \< (\what{u} - \Bar{u})(\cdot,t), I_{2,p} + I_{2,q} \>_{L^2(p+q)} \\ 
        &+ \< (\what{u} - \Bar{u})(\cdot,t), I_{3,p} + I_{3,q} \>_{L^2(p+q)} \\
        &\leq \< (\what{u} - \Bar{u})(\cdot,t), I_{1,p} + I_{1,q} \>_{L^2(p+q)} \\
        &+ \Vert (\what{u}-\Bar{u})(\cdot,t) \Vert_{L^2(p+q)} \Big( \Vert I_{2,p}+I_{2,q} \Vert_{L^2(p+q)} + \Vert I_{3,p} + I_{3,q} \Vert_{L^2(p+q)} \Big).
    \end{align*}
    So we'll need to bound $I_{2,p}, I_{2,q}, I_{3,p}$, and $I_{3,q}$ and will deal with the $I_{1,p}$ and $I_{1,q}$ terms at the end.

    Before starting, let $A_P$ be the event that
    \begin{align*}
        \Vert (\E_{x'\sim p} - \E_{x'\sim \what{p}}) \nabla_\theta u(x',\theta(0)) \nabla_\theta u(x',\theta(0))^\top \Vert \leq \sqrt{ 2L_1^2 (2L_1^2 + 3/2)  \frac{  A \log(n_p) + \log(2 M_\Theta)   }{n_p} }
    \end{align*}
    and let $A_Q$ be the event that
    \begin{align*}
        \Vert (\E_{x'\sim q} - \E_{x'\sim \what{q}}) \nabla_\theta u(x',\theta(0)) \nabla_\theta u(x',\theta(0))^\top \Vert \leq \sqrt{ 2L_1^2 (2L_1^2 + 3/2)  \frac{  A \log(n_q) + \log(2 M_\Theta)   }{n_q} }.
    \end{align*}
    Note that from \Cref{lem:matBernNN}, we know that $A_P$ occurs with probability $\geq 1 - n_p^{-A}$ and $A_Q$ occurs with probability $\geq 1 - n_q^{-A}$.  Since these events are disjoint, notice that
    \begin{align*}
        \text{Pr}\big( A_P \cap A_Q \big) = 1 - \text{Pr}\big( A_P^c \cup A_Q^c \big) = 1 - n_p^{-A} - n_q^{-A}
    \end{align*}
    where $A_P^c$ and $A_Q^c$ are the complements of $A_P$ and $A_Q$ respectively.  We work in the regime that both $A_P$ and $A_Q$ occur.

    \textbf{Bounding $I_{3,p}$ and $I_{3,q}$:} We will first work with just $I_{3,p}$ and will notice that the method of bounding $I_{3,q}$ is the same.  Then using the triangle inequality, we will get our bounds. Notice that 
    \begin{align*}
        I_{3,p} &= \< \nabla_\theta u(\cdot, \theta(0)), (\E_{x'\sim p} - \E_{x'\sim \what{p}}) \nabla_\theta u(x',\theta(0)) \bar{e}(x',t) \>_{\Theta} \\
        \implies \Vert I_{3,p} \Vert_{L^2(p+q)} &\leq \Big\Vert \underbrace{\Vert \nabla_\theta u(\cdot, \theta(0)) \Vert_{\Theta}}_{\leq L_1} \Big\Vert_{L^2(p+q)} \Vert (\E_{x'\sim p} - \E_{x'\sim \what{p}}) \nabla_\theta u(x',\theta(0)) \bar{e}(x',t) \Vert_{\Theta} \\
        &\leq \sqrt{2} L_1 \Vert \underbrace{(\E_{x'\sim p} - \E_{x'\sim \what{p}}) \nabla_\theta u(x',\theta(0)) \bar{e}(x',t)}_{a_3} \Vert_{\Theta}.
    \end{align*}
    Now we can use the fact that
    \begin{align*}
        \bar{e}(x,t) &= -\int_0^t \E_{y \sim p+q} K_0(x,y) \bar{e}(y,s) ds \\
        &= -\int_0^t \< \nabla_{\theta} u(x,\theta(0)), \E_{y \sim p+q} \nabla_\theta u(y,\theta(0)) \bar{e}(y,s) \>_{\Theta} ds.
    \end{align*}
    This means that we can rewrite $a_3$ as
    \begin{align*}
        a_3 = -\int_0^t \big[(\E_{x'\sim p} - \E_{x'\sim \what{p}}) \nabla_\theta u(x',\theta(0)) \nabla_\theta u(x',\theta(0))^\top  \big] \E_{y\sim p+q} \nabla_\theta u(y,\theta(0)) \bar{e}(y,s) ds.
    \end{align*}
    This would mean that $\Vert a_3 \Vert_\Theta$ is bounded by
    \begin{align*}
        \int_0^t \Vert (\E_{x'\sim p} - \E_{x'\sim \what{p}}) \nabla_\theta u(x',\theta(0)) \nabla_\theta u(x',\theta(0))^\top \Vert \underbrace{\Vert \nabla_\theta u(y, \theta(0)) \Vert_\Theta}_{\leq L_1} \E_{y \sim p+q} \Vert \bar{e}(y,s) \Vert ds \\
        \leq t L_1 \Vert \pif \Vert_{L^2(p+q)} \Vert \Vert (\E_{x'\sim p} - \E_{x'\sim \what{p}}) \nabla_\theta u(x',\theta(0)) \nabla_\theta u(x',\theta(0))^\top \Vert.
    \end{align*}
    Now recalling that $\bar{e}(\cdot, 0) = \pif$ and using Parseval's identity, the last inequality comes from the fact that
    \begin{align*}
        \Vert \bar{e}(\cdot, t) \Vert_{L^2(p+q)}^2 &=  \sum_{\ell = 1}^M \underbrace{ e^{-2t \lambda_\ell } }_{ \leq 1} \vert \< u_\ell,  \bar{e}(\cdot,0) \> \vert^2  \leq   \sum_{\ell = 1}^M \vert \< u_\ell, \pif \> \vert^2 = \Vert \pif \Vert_{L^2(p+q)}^2.
    \end{align*}
    So we only need to bound the operator norm of 
    \begin{align*}
        \Vert (\E_{x'\sim p} - \E_{x'\sim \what{p}}) \nabla_\theta u(x',\theta(0)) \nabla_\theta u(x',\theta(0))^\top \Vert.
    \end{align*}
    To this end, since we assume that we are working under event $A_P \cap A_Q$, we can again use \Cref{lem:matBernNN} and get that with probability greater than $1 - {n_p}^{-A} - n_q^{-A}$,
    \begin{align*}
        \Vert (\E_{x'\sim p} - \E_{x'\sim \what{p}}) \nabla_\theta u(x',\theta(0)) \nabla_\theta u(x',\theta(0))^\top \Vert \leq \sqrt{ 2L_1^2 (2L_1^2 + 3/2)  \frac{  A \log(n_p) + \log(2 M_\Theta)   }{n_p} }.
    \end{align*}
    Now putting all these bounds together, we get that
    \begin{align*}
        \Vert I_{3,p} \Vert_{L^2(p+q)} \leq t \cdot \sqrt{2}L_1^3 \Vert \pif \Vert_{L^2(p+q)} \sqrt{ 2L_1^2 (2L_1^2 + 3/2)  \frac{  A \log(n_p) + \log(2 M_\Theta)   }{n_p} }.
    \end{align*}
    For ease later on, let us define
    \begin{align*}
        g_3(t, n) = t \cdot \sqrt{2}L_1^3 \Vert \pif \Vert_{L^2(p+q)} \sqrt{ 2L_1^2 (2L_1^2 + 3/2)  \frac{  A \log(n) + \log(2 M_\Theta)   }{n} }.
    \end{align*}
    Note that because we are working under event $A_P \cap A_Q$, we know that with probability greater than $1 - n_p^{-A} - {n_q}^{-A}$
    \begin{align*}
        \Vert I_{3,q} \Vert_{L^2(p+q)} \leq t \cdot \sqrt{2}L_1^3 \Vert \pif \Vert_{L^2(p+q)} \sqrt{ 2L_1^2 (2L_1^2 + 3/2)  \frac{  A \log(n_q) + \log(2 M_\Theta)   }{n_q} } = g_3(t,n_q).
    \end{align*}
    Now let us bound $I_{2,p}$ and $I_{2,q}$.

    \textbf{Bounding $I_{2,p}$ and $I_{2,q}$:}  We will again bound for $I_{2,p}$ and essentially use the same logic for bounding $I_{2,q}$.  Note that
    \begin{align*}
        \Vert I_{2,p} \Vert_{L^2(p+q)}^2 &= \E_{x \sim p+q} \Big\vert \E_{x' \sim \what{p} } ( \what{K}_t - K_0)(x, x') \bar{e}(x',t) \Big\vert^2 \\
        &\leq \E_{x \sim p+q} \Big( \E_{x' \sim \what{p} } \vert  (\what{K}_t - K_0)(x, x') \vert \vert \bar{e}(x',t) \vert \Big)^2.
    \end{align*}
    Using \Cref{lem:thetahatBD}, note that
    \begin{align*}
        \vert \what{K}_t(x,x') - K_0(x,x') \vert &= \vert \< \nabla_\theta u(x,\what{\theta}(t)), \nabla_\theta u(x',\what{\theta}(t)) \> - \< \nabla_\theta u(x,\theta(0)), \nabla_\theta u(x',\theta(0)) \> \vert \\
        &\leq \Vert \nabla_\theta u(x,\what{\theta}(t)) \Vert_{\Theta} \Vert \Vert \nabla_\theta u(x', \what{\theta}(t)) - \nabla_{\theta} u(x',\theta(0)) \Vert_{\Theta} \\
        &+ \Vert \nabla_\theta u(x,\what{\theta}(t)) - \nabla_\theta u(x,\theta(0)) \Vert_{\Theta} \Vert \nabla_\theta u(x',\theta(0)) \Vert_\Theta \\
        &\leq 2 L_1 L_2 \Vert \what{\theta}(t) - \theta(0) \Vert_{\Theta} \leq 2 L_1 L_2  \sqrt{t}.
    \end{align*}
    Now the only thing left to bound is $\E_{x' \sim \what{p} } \vert \bar{e}(x',t) \vert$.  To do this, recalling the time-integrated form of $\bar{e}$, we have that 
    \begin{align*}
        \E_{x \sim \what{p} } \vert \bar{e}(x,t) \vert = \Vert \bar{e}(\cdot, t) \Vert_{L^1(\what{p})} \leq \Big( \Vert \bar{e}(\cdot, t) \Vert_{L^2(\what{p})}^2 \Big)^{1/2}.
    \end{align*}
    Now notice that
    \begin{align*}
        \Vert \bar{e}(\cdot, t) \Vert_{L^2(\what{p})}^2 = \E_{x \sim \what{p}} \vert \bar{e}(x,t) \vert^2 = ( \E_{x \sim \what{p}} - \E_{x \sim p}) \vert \bar{e}(x,t) \vert^2 + \E_{x \sim p} \vert \bar{e}(x,t) \vert^2.
    \end{align*}
    Because integrating a positive function over both $p$ and $q$ is an upper bound of just integrating over $p$, we know that
    \begin{align*}
        \E_{x \sim p} \vert \bar{e}(x,t) \vert^2 = \Vert \bar{e}(\cdot,t) \Vert_{L^2(p)}^2 \leq \Vert \bar{e}(\cdot,0) \Vert_{L^2(p+q)}^2 = \Vert \pif \Vert_{L^2(p+q)}^2,
    \end{align*}
    so we only need to deal with the first term.  In particular, using the time-integrated form of $\vert \bar{e}(x,t) \vert^2$, we get that equals
    \begin{align*}
        \Big\vert ( \E_{x \sim \what{p}} - \E_{x \sim p}) \vert \bar{e}(x,t) \vert^2 \Big\vert  = \bigg\vert \int_0^t \int_0^t \E_{y_1, y_2 \sim p+q} \bar{e}(y_2,s_2) \nabla_\theta u(y_2, \theta(0))^\top \\
        \cdot \Big[ ( \E_{x \sim \what{p}} - \E_{x \sim p}) \nabla_\theta u(x,\theta(0)), \nabla_\theta u(x,\theta(0))^\top \Big] \nabla_\theta u(y_1,\theta(0)) \bar{e}(y_1, s_1) ds_1 ds_2 \bigg\vert \\
        \leq L_1^2 \Vert ( \E_{x \sim \what{p}} - \E_{x \sim p}) \nabla_\theta u(x,\theta(0)), \nabla_\theta u(x,\theta(0))^\top \Vert \underbrace{\bigg( \int_0^t \E_{x \sim p+q} \vert \bar{e}(x,s) \vert ds \bigg)^{2}}_{\leq (t \sqrt{2}\Vert \pif \Vert_{L^2(p+q)} )^2  } \\
        \leq 2 L_1^2 t^2 \Vert \pif \Vert_{L^2(p+q)}^2 \Big\Vert ( \E_{x \sim \what{p}} - \E_{x \sim p}) \nabla_\theta u(x,\theta(0)), \nabla_\theta u(x,\theta(0))^\top \Big\Vert .
    \end{align*}
    Again, since we are under event $A_P \cap A_Q$ we can use \Cref{lem:matBernNN} and get that with probability greater than $1 - n_p^{-A} - n_q^{-A}$
    \begin{align*}
        \big\vert ( \E_{x \sim \what{p}} - \E_{x \sim p}) \vert \bar{e}(x,t) \vert^2 \big\vert \leq 2 L_1^2 t^2 \Vert \pif \Vert_{L^2(p+q)}^2 \sqrt{ 2L_1^2 (2L_1^2 + 3/2)  \frac{  A \log(n_p) + \log(2 M_\Theta)   }{n_p} }.
    \end{align*}
    Plugging this back, we get
    \begin{align*}
        \E_{x \sim \what{p} } \vert \bar{e}(x,t) \vert \leq \sqrt{2} \Vert \pif \Vert_{L^2(p+q)} \Big( 1 + 2 L_1^3 t^2 \sqrt{2(2L_1^2 + 3/2) \frac{A \log(n_p) + \log(2 M_\Theta)}{n_p} } \Big)^{1/2}.
    \end{align*}
    Plugging back to our original expression for $I_{2,p}$ and using the fact that $\E_{x\sim p+q} 1 = 2$, we get that
    \begin{align*}
        \Vert I_{2,p} \Vert_{L^2(p+q)} \leq 4 L_1 L_2 \sqrt{t}  \Vert \pif \Vert_{L^2(p+q)} \Big( 1 + 2 L_1^3 t^2 \sqrt{2(2L_1^2 + 3/2) \frac{A \log(n_p) + \log(2 M_\Theta)}{n_p} } \Big)^{1/2}
    \end{align*}
    with probability $\geq 1 - n_p^{-A} - n_q^{-A}$.  Similar to before, we define
    \begin{align*}
        g_2(t,n) = 4 L_1 L_2 \sqrt{t}  \Vert \pif \Vert_{L^2(p+q)} \Big( 1 + 2 L_1^3 t^2 \sqrt{2(2L_1^2 + 3/2) \frac{A \log(n) + \log(2 M_\Theta)}{n} } \Big)^{1/2}.
    \end{align*}
    Using the same logical reasoning of being in the event $A_P \cap A_Q$, we get that with probability $\geq 1 - n_p^{-A} - n_q^{-A}$
    \begin{align*}
        \Vert I_{2,q} \Vert_{L^2(p+q)} &\leq 4 L_1 L_2 \sqrt{t}  \Vert f^* \Vert_{L^2(p+q)} \Big( 1 + 2 L_1^3 t^2 \sqrt{2(2L_1^2 + 3/2) \frac{A \log(n_q) + \log(2 M_\Theta)}{n_q} } \Big)^{1/2} \\
        &= g_2(t,n_{Q}).
    \end{align*}

    \textbf{Working with the $I_1$ terms:}  Let us again work with $I_{1,p}$ and use the same logic for $I_{1,q}$ later.  In particular, note that
    \begin{align*}
        \< (\what{u} - \Bar{u})(\cdot,t), I_{1,p} \>_{L^2(p+q)} &=  \< (\what{u} - \Bar{u})(\cdot,t), -\E_{x' \sim \what{p}} \what{K}_t(\cdot, x') ( \what{e}_P(x', t) - \Bar{e}(x',t)) \>_{L^2(p+q)} \\
         &= -\< (\what{u} - \Bar{u})(\cdot,t), \E_{x' \sim \what{p}} \what{K}_t(\cdot, x') ( \what{u}(x', t) - \Bar{u}(x',t)) \>_{L^2(p+q)} \\
         &- \< (\what{u} - \Bar{u})(\cdot,t), \E_{x' \sim \what{p}} \what{K}_t(\cdot, x') ( f^*(x') - 1)  \>_{L^2(p+q)} \\
         &\leq - \< (\what{u} - \Bar{u})(\cdot,t), \E_{x' \sim \what{p+q}} \what{K}_t(\cdot, x') ( f^*(x') - 1)  \>_{L^2(p+q)},
    \end{align*}
    where we get the inequality because $\what{K}_t$ is a positive semi-definite operator so the first term is less than $0$.  Now we can bound by the following
    \begin{align*}
        \vert \< (\what{u} - \Bar{u})(\cdot,t), I_{1,p} \>_{L^2(p+q)} \vert &\leq \Vert (\what{u}-\bar{u})(\cdot,t) \Vert_{L^2(p+q)} \Vert  \E_{x' \sim \what{p}} \what{K}_t(\cdot, x') ( f^*(x') - 1) \Vert_{L^2(p+q)}.
    \end{align*}
    Here note that
    \begin{align*}
        \Vert  \E_{x' \sim \what{p}} \what{K}_t(\cdot, x') ( f^*(x') - 1) \Vert_{L^2(p+q)} &= \Vert  \E_{x' \sim \what{p}} \< \nabla_\theta u(\cdot, \what{\theta}(t)), \nabla_\theta u(x',\what{\theta}(t))\>_{\Theta} ( f^*(x') - 1) \Vert_{L^2(p+q)} \\
        &\leq L_1^2 \E_{x' \sim \what{p}} \vert f^*(x') - 1 \vert \\
        &\leq L_1^2 \int_{\R^d} \bigg\vert \frac{p-q}{p+q}(x) - 1 \bigg\vert d\what{p}(x) \\
        &= L_1^2 \int_{\R^d} \bigg\vert \frac{p(x)-q(x) - p(x) - q(x)}{p(x)+q(x)} \bigg\vert d\what{p}(x) \\
        &= L_1^2 \int_{\R^d} 2 \bigg\vert \frac{q(x)}{p(x)+q(x)} \bigg\vert d\what{p}(x) \\
        &= 2 L_1^2 \frac{1}{n_p} \sum_{i=1}^{n_p} \underbrace{\bigg\vert \frac{q(x_i)}{p(x_i)+q(x_i)}}_{\leq 1} \bigg\vert \\
        &\leq 2 L_1^2 \frac{1}{n_p} \sum_{i=1}^{n_p} 1 = 2 L_1^2.
    \end{align*}
    This means that
    \begin{align*}
        \vert \< (\what{u} - \Bar{u})(\cdot,t), I_{1,p} \>_{L^2(p+q)} \vert \leq 2 L_1^2 \Vert (\what{u}-\bar{u})(\cdot,t) \Vert_{L^2(p+q)}.
    \end{align*}
    Using the same logic (but with the term $\frac{p(x)}{p(x) + q(x)}$), we can show that
    \begin{align*}
        \vert \< (\what{u} - \Bar{u})(\cdot,t), I_{1,q} \>_{L^2(p+q)} \vert \leq 2 L_1^2 \Vert (\what{u}-\bar{u})(\cdot,t) \Vert_{L^2(p+q)}.
    \end{align*}
    Putting this altogether, we see that
    \begin{align*}
        \frac{d}{dt} \Vert (\what{u}-\Bar{u})(\cdot,t) \Vert_{L^2(p+q)}^2 &\leq \vert \< (\what{u} - \Bar{u})(\cdot,t), I_{1,p} \>_{L^2(p+q)}\vert + \vert \< (\what{u} - \Bar{u})(\cdot,t), I_{1,q} \>_{L^2(p+q)} \vert \\
        &+ \Vert (\what{u}-\Bar{u})(\cdot,t) \Vert_{L^2(p+q)} \Big( \Vert I_{2,p} \Vert_{L^2(p+q)} + \Vert I_{2,q} \Vert_{L^2(p+q)} \\
        &+ \Vert I_{3,p} \Vert_{L^2(p+q)} + \Vert I_{3,q} \Vert_{L^2(p+q)} \Big) \\
        &= \bigg( 4 L_1^2 +  g_2(t,n_p) + g_3(t,n_p) \\
        &+ g_2(t,n_q) + g_3(t,n_q) \bigg) \Vert (\what{u}-\bar{u})(\cdot,t) \Vert_{L^2(p+q)}.
    \end{align*}
    Using the same argument in \Cref{prop:uubarBND}, let $t^* \in [0,t]$ be such that
    \begin{align*}
        \sup_{s \in [0,t]} \Vert (\what{u}-\Bar{u})(\cdot,s) \Vert_{L^2(p+q)} = \Vert (\what{u}-\Bar{u})(\cdot,t^*) \Vert_{L^2(p+q)},
    \end{align*}
    then we know that
    \begin{align*}
        \Vert (\what{u}-\Bar{u})(\cdot,t^*) \Vert_{L^2(p+q)}^2 &\leq \int_0^{t^*} \bigg( 4 L_1^2 +  g_2(s,n_p) + g_3(s,n_p) \\
        &+ g_2(s,n_q) + g_3(s,n_q) \bigg) \Vert (\what{u}-\bar{u})(\cdot,s) \Vert_{L^2(p+q)} ds \\
        &\leq \int_0^{t^*} \bigg( 4 L_1^2 +  g_2(s,n_p) + g_3(s,n_p) \\
        &+ g_2(s,n_q) + g_3(s,n_q) \bigg) \Vert (\what{u}-\bar{u})(\cdot,t^*) \Vert_{L^2(p+q)} ds \\
        &\leq \Vert (\what{u}-\bar{u})(\cdot,t^*) \Vert_{L^2(p+q)} \int_0^{t^*} \bigg( 4 L_1^2 +  g_2(s,n_p) + g_3(s,n_p) \\
        &+ g_2(s,n_q) + g_3(s,n_q) \bigg) ds \\
        \Vert (\what{u}-\Bar{u})(\cdot,t^*) \Vert_{L^2(p+q)} &\leq \int_0^{t^*} \bigg( 4 L_1^2 +  g_2(s,n_p) + g_3(s,n_p) + g_2(s,n_q) + g_3(s,n_q) \bigg) ds.
    \end{align*}
    Now since $t^* \leq t$ and
    \begin{align*}
        \Vert (\what{u}-\Bar{u})(\cdot,t) \Vert_{L^2(p+q)} \leq \Vert (\what{u}-\Bar{u})(\cdot,t^*) \Vert_{L^2(p+q)},
    \end{align*}
    we know that
    \begin{align*}
        \Vert (\what{u}-\Bar{u})(\cdot,t) \Vert_{L^2(p+q)} &\leq \int_0^{t^*} \bigg( 4 L_1^2 +  g_2(s,n_p) + g_3(s,n_p) + g_2(s,n_q) + g_3(s,n_q) \bigg) ds.
    \end{align*}
    Moreover, by inspection, we can see that
    \begin{align*}
        4 L_1^2 +  g_2(s,n_p) + g_3(s,n_p) + g_2(s,n_q) + g_3(s,n_q)
    \end{align*}
    is monotone in $s$, which means that
    \begin{align*}
        \Vert (\what{u}-\Bar{u})(\cdot,t) \Vert_{L^2(p+q)} &\leq \bigg( 4 L_1^2 +  g_2(s,n_p) + g_3(s,n_p) + g_2(s,n_q) + g_3(s,n_q) \bigg) t.
    \end{align*}
    Putting this altogether and using the fact that we are working under the regime of event $A_P \cap A_Q$, we can use \Cref{lem:matBernNN} for the zero-time NTK for samples from $p$ and $q$ to get that with probability $\geq 1 - n_p^{-A} - n_q^{-A}$, we have
    \begin{align*}
        \Vert (\what{u}-\Bar{u})(\cdot,t) \Vert_{L^2(p+q)} &\leq 4 L_1^2 t +  g_2(s,n_p) t + g_3(s,n_p) t + g_2(s,n_q) t + g_3(s,n_q) t,
    \end{align*}
    but the right-hand side of the inequality is just
    \begin{align*}
        4 L_1^2 t &+  4 L_1 L_2 t^{3/2}  \Vert \pif \Vert_{L^2(p+q)} \Big( 1 + 2 L_1^3 t^2 \sqrt{2(2L_1^2 + 3/2) \frac{A \log(n_p) + \log(2 M_\Theta)}{n_p} } \Big)^{1/2} \\
        &+ 4 L_1 L_2 t^{3/2}  \Vert \pif \Vert_{L^2(p+q)} \Big( 1 + 2 L_1^3 t^2 \sqrt{2(2L_1^2 + 3/2) \frac{A \log(n_q) + \log(2 M_\Theta)}{n_q} } \Big)^{1/2} \\
        &+ t^2 \cdot \sqrt{2}L_1^3 \Vert \pif \Vert_{L^2(p+q)} \sqrt{ 2L_1^2 (2L_1^2 + 3/2)  \frac{  A \log(n_p) + \log(2 M_\Theta)   }{n_p} }\\
        &+ t^2 \cdot \sqrt{2}L_1^3 \Vert \pif \Vert_{L^2(p+q)} \sqrt{ 2L_1^2 (2L_1^2 + 3/2)  \frac{  A \log(n_q) + \log(2 M_\Theta)   }{n_p} }.
    \end{align*}
    This means that 
    \[\Vert (\what{u}-\Bar{u})(\cdot,t) \Vert_{L^2(p+q)} = O(t^{5/2}).\]
    Moreover, we get the result using the fact that \[\Vert \pif \Vert_{L^2(p+q)} \leq \Vert f^* \Vert_{L^2(p+q)}.\qedhere\]
\end{proof}

\begin{proof}[Proof of \Cref{prop:testStatUhat_Ubar_bnd_samp}]\label{pf:testStatUhat_Ubar_bnd_samp}
    Consider following calculation
    \begin{align*}
        \big\vert \what{T}_{test}(t) - \overline{T}(t) \big\vert &= \bigg\vert \int_{\R^d} \what{u}(x,t) d(\ptest-\qtest)(x) - \int_{\R^d} \bar{u}(x,t) d(p-q)(x) \bigg\vert \\
        &= \bigg\vert \int_{\R^d} \what{u}(x,t) d(\ptest-\qtest)(x) - \int_{\R^d} \what{u}(x,t) d(p-q)(x)  \\
        &+ \int_{\R^d} \what{u}(x,t) d(p-q)(x) - \int_{\R^d} \bar{u}(x,t) d(p-q)(x) \bigg\vert \\
        &\leq \underbrace{ \bigg\vert \int_{\R^d} \what{u}(x,t) d\big[ (\ptest - p) + (q - \qtest) \big](x) \bigg\vert}_{A_1} \\
        &+ \underbrace{ \bigg\vert \int_{\R^d} ( \what{u} - \bar{u} )(x,t) d(p-q)(x) \bigg\vert}_{A_2}.
    \end{align*}
    Let us deal with $A_2$ first and then with $A_1$.  Note that
    \begin{align*}
        A_2 &= \bigg\vert \int_{\R^d} ( \what{u} - \bar{u} )(x,t) dp(x) - \int_{\R^d} ( \what{u} - \bar{u} )(x,t) dq(x)  \bigg\vert \\
        &\leq \bigg\vert \int_{\R^d} ( \what{u} - \bar{u} )(x,t) dp(x) \bigg\vert + \bigg\vert \int_{\R^d} ( \what{u} - \bar{u} )(x,t) dq(x)  \bigg\vert \\
        &\leq \int_{\R^d} \vert ( \what{u} - \bar{u} )(x,t) \vert d(p+q)(x) \\
        &\leq \sqrt{2} \Vert \what{u} - \bar{u}(\cdot, t) \Vert_{L^2(p+q)}.
    \end{align*}
    So we can use \Cref{uhatubarBND} for $A_2$ and will use this as part of the final bound.  Notice that $A_1$ is actually bounds $\vert \what{T}_{test}(t) - \what{T}_{pop}(t) \vert$. Now let us bound $A_1$.  First note that
    \begin{align*}
        A_1 &\leq \underbrace{\bigg\vert \int_{\R^d}  \what{u}(x,t)  d( \ptest - p) \bigg\vert }_{A_{1,p}} + \underbrace{  \bigg\vert \int_{\R^d} \what{u}(x,t) d( q - \qtest) \bigg\vert  }_{A_{1,q}}. 
    \end{align*}
    To bound $A_{1,p}$ and $A_{1,q}$, we will aim to use Hoeffding's inequality, but we must first show that $\vert \what{u}(x,t) \vert$ is bounded.  To this end, consider the time-integrated form of $\what{u}(x,t) = f(x,\what{\theta}(t))$.  Recalling the density-specific residuals
    \begin{align*}
        \what{e}_p(x,t) = \big( f(x',\what{\theta}(t)) - 1 \big) \\
        \what{e}_q(x,t) = \big( f(x',\what{\theta}(t)) + 1 \big),
    \end{align*}
    and using \Cref{C1}, we have
    \begin{align*}
        \vert f(x,\what{\theta}(t)) \vert &= \bigg\vert - \frac{1}{2} \int_0^t  \bigg( \E_{x' \sim \what{p} } \< \nabla_\theta f(x,\what{\theta}(s)), \nabla_\theta f(x',\what{\theta}(s))\>_{\Theta} \what{e}_p(x',s) \\
        &+\E_{x' \sim \what{q} } \< \nabla_\theta f(x,\what{\theta}(s)), \nabla_\theta f(x',\what{\theta}(s))\>_{\Theta} \what{e}_q(x',s) \bigg) ds \bigg\vert.
    \end{align*}
    It is important to note that in the equation above, $\what{p}$ and $\what{q}$ are training datasets (\textbf{not} $\ptest$ and $\qtest$), and with this in mind, we continue as
    \begin{align*}
        \vert f(x,\what{\theta}(t)) \vert &\leq \frac{1}{2}\int_0^t  \E_{x' \sim \what{p} } \underbrace{ \vert  \< \nabla_\theta f(x,\what{\theta}(s)), \nabla_\theta f(x',\what{\theta}(s))\>_{\Theta} \vert}_{\leq L_1^2} \vert \what{e}_p(x',s) \vert \\
        &+\E_{x' \sim \what{q} } \underbrace{ \vert \< \nabla_\theta f(x,\what{\theta}(s)), \nabla_\theta f(x',\what{\theta}(s))\>_{\Theta} \vert}_{\leq L_1^2} \vert \what{e}_q(x',s) \vert  ds \\
        &\leq \frac{1}{2} L_1^2 \int_0^t  \E_{x' \sim \what{p} }\vert \what{e}_p(x',s) \vert +\E_{x' \sim \what{q} } \vert \what{e}_q(x',s) \vert  ds \\
        &= \frac{1}{2} L_1^2 \int_0^t  \bigg( \int_{\R^d} \vert f(x,\what{\theta}(s)) - 1 \vert d\what{p}(x) + \int_{\R^d} \vert f(x,\what{\theta}(s)) + 1 \vert d\what{q}(x) \bigg) ds.
    \end{align*}
    Using \Cref{monotoneLemma} with $a(x,t) =  f(x,\what{\theta}(s)) - 1 $ and $b(x,t) =  f(x,\what{\theta}(s)) + 1 $, we know that the right hand side of the equation above is decreasing if $\what{L}(\what{\theta}(s))$ is decreasing.  Indeed, recall
    \begin{align*}
        \frac{d}{dt} \what{L}(\what{\theta}(t)) = \< \nabla_\theta \what{L}(\what{\theta}(t)), \Dot{\what{\theta}}(t) \>_{\Theta} = - \Vert \nabla_\theta \what{L}(\what{\theta}(t)) \Vert_{\Theta} \leq 0.
    \end{align*}
    This means that
    \begin{align*}
        \int_{\R^d} \vert f(x,\what{\theta}(s)) - 1 \vert d\what{p}(x) + \int_{\R^d} \vert f(x,\what{\theta}(s)) + 1 \vert d\what{q}(x) \\
        \leq \int_{\R^d} \vert f(x,\what{\theta}(0)) - 1 \vert d\what{p}(x) + \int_{\R^d} \vert f(x,\what{\theta}(0)) + 1 \vert d\what{q}(x) = 2.
    \end{align*}
    Plugging this back in, we get that
    \begin{align*}
        \vert f(x,\what{\theta}(t)) \vert \leq L_1^2 t.
    \end{align*}
    Because we have boundedness, we can use \Cref{hoeffding}.  Reworking the probability and lower bound in Hoeffding's inequality, we see that
    \begin{align*}
        A_{1,p} \leq L_1^2 t \sqrt{  2 \frac{A \log(m_p) }{m_p}  }
    \end{align*}
    with probability $\geq 1 - m_p^{-A}$.  Similarly, we get that
    \begin{align*}
        A_{1,q} \leq L_1^2 t \sqrt{  2 \frac{A \log(m_q) }{m_q}  }
    \end{align*}
    with probability $\geq 1 - m_q^{-A}$.  So for both these events to occur together, we can use a probability intersection bound to get that
    \begin{align*}
        A_1 \leq L_1^2 t \sqrt{2} \bigg( \sqrt{ \frac{A \log(m_p) }{m_p}  } +   \sqrt{  \frac{A \log(m_q) }{m_q}  } \bigg)
    \end{align*}
    with probability $\geq 1 - m_p^{-A} - m_q^{-A}$.  Coming back to $A_2$, we know the bound from \Cref{uhatubarBND} occurs with probability $\geq 1 - n_p^{-A} - n_q^{-A}$ (the finite-sample training dataset size); thus, to have the bound for $A_1$ and $A_2$ simultaneously, we again use an intersection probability bound to get that both events occur simultaneously with probability $\geq 1 - (m_p^{-A} + m_q^{-A} + n_p^{-A} + n_q^{-A})$.  Putting this altogether, we see that with probability $\geq 1 - ( m_p^{-A} + m_q^{-A} + n_p^{-A} + n_q^{-A} )$ we have
    \begin{align*}
        \big\vert \what{T}_{test}(t) - \overline{T}(t) \big\vert &\leq C_{1 L_1, A, m_p, m_q } t + C_{ L_1, L_2, f^* } t^{3/2} \\
        &+ C_{L_1, f^*, n_p, n_q, M_\Theta, A } t^2 \\
        &+ C_{ L_1, L_2, f^*, n_p, n_q, M_\Theta, A } t^{5/2},
    \end{align*}
    where the constants can be recovered by putting the bound for $A_1$ together with \Cref{uhatubarBND}.  So we're done.
\end{proof}

\begin{proof}[Proof of \Cref{lem:stat_hat_Monotone}]\label{pf:stat_hat_Monotone}
    Recall that the loss $\what{L}(\what{\theta}(s))$ is monotonically decreasing because
    \begin{align*}
        \frac{d}{dt}\what{L}(\what{\theta}(t)) = \< \nabla_\theta \what{L}(\what{\theta}(t)), \what{\theta}(t) \>_{\Theta} = - \Vert \nabla_\theta \what{L}(\what{\theta}(t)) \Vert_{\Theta} \leq 0.
    \end{align*}
    Now since the loss
    \begin{align*}
        \what{L}(\what{\theta}(s)) = \int_{\R^d} \vert \what{u}(x,s) - 1 \vert^2 d\what{p}(x) + \int_{\R^d} \vert \what{u}(x,s) + 1 \vert^2 d\what{q}(x)
    \end{align*}
    is decreasing, we can use \Cref{monotoneLemma} applied to $\what{L}(\what{\theta}(s))$ to see that
    \begin{align*}
        \int_{\R^d} \vert \what{u}(x,s) - 1 \vert d\what{p}(x) + \int_{\R^d} \vert \what{u}(x,s) + 1 \vert d\what{q}(x)
    \end{align*}
    is actually monotonically decreasing.  Notice that because $\vert \what{u}(x, s) \vert \leq 1$ on $[0,\what{\tau}]$, we have
    \begin{align*}
        \int_{\R^d} \vert \what{u}(x,s) - 1 \vert d\what{p}(x) &= \int_{\R^d} (1 - \what{u}(x,s) ) \\
        \int_{\R^d} \vert \what{u}(x,s) + 1 \vert d\what{q}(x) &= \int_{\R^d} ( \what{u}(x,s) + 1 ) d\what{q}(x).
    \end{align*}
    So putting this back into the definition of monotonically decreasing loss, we see that
    \begin{align*}
        \int_{\R^d} 1 - \what{u}(x,s) d\what{p}(x) + \int_{\R^d}  \what{u}(x,s) + 1  d\what{q}(x) \\
        \geq \int_{\R^d} 1 - \what{u}(x,t) d\what{p}(x) + \int_{\R^d}  \what{u}(x,t) + 1  d\what{q}(x) \\
        \implies \int_{\R^d} \what{u}(x,t) d(\what{p}-\what{q})(x) \geq \int_{\R^d}  \what{u}(x,s)d( \what{p} - \what{q})(x).
    \end{align*}
    This implies that $\what{T}_{train}(t)$ is monotonically increasing.  So we're done.
\end{proof}

\begin{proof}[Proof of \Cref{thm:uhat_H1}]\label{pf:uhat_H1}
    Note that the assumptions we have are essentially the assumptions of \Cref{prop:testStatUhat_Ubar_bnd_samp} and \Cref{lem:stat_hat_Monotone}.  Moreover, because we have 
    \begin{align*}
        \max(R^2,\what{\tau}) \geq t \geq t_1^*(\epsilon),
    \end{align*}
    we can use \Cref{prop:testStatUhat_Ubar_bnd_samp}, \Cref{timeAnalysisCor}, and \Cref{lem:stat_hat_Monotone} simultaneously.  With probability $\geq 1 - 2(n_p^{-A} + n_q^{-A})$, using the reverse triangle inequality and montonicity gives us
    \begin{align*}
        \vert \what{T}_{train}(t) \vert &\geq \vert \what{T}_{train}(t_1^*(\epsilon)) \vert \\
        &\geq \bigg\vert \big\vert \overline{T}(t_1^*(\epsilon)) \big\vert - \big\vert \what{T}_{train}(t_1^*(\epsilon)) - \overline{T}(t_1^*(\epsilon)) \big\vert \bigg\vert \\
        &\geq \epsilon - \delta_{train}(t_1^*(\epsilon))
    \end{align*}
    where we can rid of the absolute values by assumption.  Now, note that if we assumed that
    \begin{align*}
        \epsilon &> \delta_{train}(t_1^*(\epsilon)) + L_1^2 t \sqrt{2} \bigg( \sqrt{ A \log(m_p) / m_p } + \sqrt{ A \log(m_q) / m_q } \\
        &\sqrt{ A \log(n_p) / n_p } + \sqrt{ A \log(n_q) / n_q } \bigg),
    \end{align*}
    then we would have
    \begin{align*}
        \vert \what{T}_{test}(t) \vert &\geq \bigg\vert \vert \what{T}_{train}(t) \vert - \vert \what{T}_{test}(t) - \what{T}_{train}(t) \vert \bigg\vert \\
        &\geq \epsilon - \delta_{train}(t_1^*(\epsilon)) - L_1^2 t \sqrt{2} \bigg( \sqrt{ A \log(m_p) / m_p } + \sqrt{ A \log(m_q) / m_q } \\
        &+ \sqrt{ A \log(n_p) / n_p } + \sqrt{ A \log(n_q) / n_q } \bigg)
    \end{align*}
    Similarly, if we assume that
    \begin{align*}
        \epsilon &> \delta_{train}(t_1^*(\epsilon)) + L_1^2 t \sqrt{2} \bigg( \sqrt{ A \log(n_p) / n_p } + \sqrt{ A \log(n_q) / n_q } \bigg),
    \end{align*}
    then we have
    \begin{align*}
        \vert \what{T}_{pop}(t) \vert &\geq \bigg\vert \vert \what{T}_{train}(t) \vert - \vert \what{T}_{pop}(t) - \what{T}_{train}(t) \vert \bigg\vert \\
        &\geq \epsilon - \delta_{train}(t_1^*(\epsilon)) - L_1^2 t \sqrt{2} \bigg( \sqrt{ A \log(n_p) / n_p } + \sqrt{ A \log(n_q) / n_q } \bigg).
    \end{align*}
    So we're done.
\end{proof}

\begin{proof}[Proof of \Cref{thm:uhat_H0}]\label{pf:uhat_H0}    
    Because of the conditions on $t$, we can use all of \Cref{timeAnalysisCor}, \Cref{prop:testStatUhat_Ubar_bnd_samp}, and \Cref{prop:testStatUhat_pop_Ubar_bnd} simultaneously.  So essentially, we can use the triangle inequality to get
    \begin{align*}
        \vert \what{T}(t) \vert &\leq \big\vert \overline{T}(t) \big\vert + \big\vert \what{T}(t) - \overline{T}(t) \big\vert  \leq \epsilon + \delta(t) \leq \epsilon + \delta(t_2^*(\epsilon)),
    \end{align*}
    where, in general, $\what{T}(t)$ and $\delta(t)$ can be replaced by $\what{T}_{train}, \what{T}_{test}(t), \what{T}_{pop}(t)$ and $\delta_{train}(t), \delta_{test}(t), \delta_{pop}(t)$, respectively.  These situations happen with probability $\geq 1 - 2( n_p^{-A} + n_q^{-A})$, $\geq 1 - (n_{p}^{-A} + n_q^{-A} + m_{p}^{-A} + m_q^{-A})$, and $\geq 1 - (n_p^{-A} + n_{q}^{-A})$, respectively.  So we're done.
\end{proof}

\begin{proof}[Proof of \Cref{cor:genMomRes_uhat}]\label{pf:genMomRes_uhat}
    We will first work with the time associated with detecting deviation $\epsilon$ under the null hypothesis, and then we consider time associated with detecting $\epsilon$ under the assumption that $f^*$ lies on the first $k$ eigenfunctions of $K_0$.  After both these detection times are studied, we study when they are well-separated.

    \textbf{Null Hypothesis:} We first note that if we are in the null hypothesis so that $p = q$, then $f^* = 0$, which implies that $\Vert f^* \Vert_{L^2(p+q)} = \Vert \pif \Vert_{L^2(p+q)} = 0$.  Looking into the proof of \Cref{prop:testStatUhat_Ubar_bnd_samp} and \Cref{prop:testStatUhat_pop_Ubar_bnd}, we see that the only term that does not depend on $f^*$ is of the form $C^+ t$ but $C^+$ changes depending on which dataset the two-sample test is evaluated on.  In particular, we specify
    \begin{align*}
        C^+ = \begin{cases}
            \sqrt{2}L_1^2 4 & \what{T}_{pop}(t) \text{ evaluation} \\
            \sqrt{2}L_1^2 \Bigg(4 + \sqrt{ A\frac{\log(n_p)}{n_p} } + \sqrt{ A \frac{\log(n_q)}{n_q} } \Bigg) & \what{T}_{train}(t) \text{ evaluation} \\
            \sqrt{2}L_1^2 \Bigg(4 + \sqrt{ A\frac{\log(m_p)}{m_p} } + \sqrt{ A \frac{\log(m_q)}{m_q} } \Bigg) & \what{T}_{test}(t) \text{ evaluation}.
        \end{cases}
    \end{align*}
    This means that under the null hypothesis $p=q$ and with either $\what{T}_{pop}, \what{T}_{test},$ or $\what{T}_{train}$ determining $C^+$, if
    \begin{align*}
        t^+(\epsilon)  &\geq \frac{\epsilon}{C^+},
    \end{align*}
    then we cannot trust the neural network two-sample test statistic past the time threshold $t^+(\epsilon)$.  Note that as $n_p, n_q, m_p, m_q \to \infty$, the threshold for $t$ to cross becomes $\frac{\epsilon}{4 \sqrt{2}L_1^2 }$ and reverts back to the constant $C^+$ in the case we use $\what{T}_{pop}(t)$.

    \textbf{Assumption $\pif = f^*_k$:}  Recall that we are dealing with the case that $\pif = f^*_k$ so that $\pif$ nontrivially projects onto \textit{only} the first $k$ eigenfunctions.  To deal with the time-approximation error $\delta(t)$, we will consider the detection time needed for $2\epsilon$ and conduct analysis for this case.  If we are in the assumption $\pif = f^*_k$, notice that the minimum time needed for the zero-time NTK dynamics to detect a deviation $2\epsilon$ from \Cref{timeAnalysisCor} is given by
    \begin{align*}
        t_1^*(2\epsilon) = \min_{S \in \mathcal{S}_1(\epsilon)} \lambda_{\min}(S) \log\bigg( \frac{\Vert f^*_k \Vert_S^2 }{ \Vert f^*_k \Vert_S^2 - 2\epsilon } \bigg).
    \end{align*}
    Importantly, if we want to counteract the approximation error from \Cref{prop:testStatUhat_Ubar_bnd_samp} and \Cref{prop:testStatUhat_pop_Ubar_bnd}, we simply need to make sure $\delta( t_1^*(2\epsilon) ) < \epsilon$ so that the total detection will be $2 \epsilon - \delta( t_1^*(2\epsilon) ) > \epsilon$, where $\delta$ will be $\delta_{pop}, \delta_{test},$ or $\delta_{train}$.  Notice from the form of time-approximation error function $\delta(t)$, we have
    \begin{align*}
        C^- \min\{ t, t^{5/2} \} \leq \delta(t) \leq C^- \max\{ t, t^{5/2} \},
    \end{align*}
    where $C^- = C^+ + C_2 + C_3 + C_4$ with the constants coming from \Cref{prop:testStatUhat_Ubar_bnd_samp}, \Cref{prop:testStatUhat_pop_Ubar_bnd}, and $C^+$ defined above.  Thus, note that $C^-$ depends on whether we use the two-sample test $\what{T}_{pop}, \what{T}_{test},$ or $\what{T}_{train}$. Assuming the specific assumption that $f^*$ nontrivially projects only on the first $k$ eigenfunctions of $K_0$ so that $\pif = f_k^*$, notice that
    \begin{align*}
        t_1^*(2\epsilon) &= \min_{S \in \mathcal{S}_1(2\epsilon)} \lambda_{\min}(S) \log\bigg( \frac{ \Vert f^*_k \Vert_{S}^2 }{\Vert f^*_k \Vert_{S}^2 - 2\epsilon } \bigg) \\
        &\leq \lambda_k \log\bigg( \frac{ \Vert f^*_k \Vert_{L^2(p+q)}^2 }{\Vert f^*_k \Vert_{L^2(p+q)}^2 - 2\epsilon } \bigg) := t^-_k (2\epsilon).
    \end{align*}
    With this in mind, notice that
    \begin{align*}
        \delta( t_1^*(2\epsilon) ) \leq C^- \max \{ t_1^*(2\epsilon), ( t_1^*(2\epsilon) )^{5/2} \} \leq C^-  \max \{ t^-_k (2\epsilon), (t^-_k (2\epsilon))^{5/2} \}
    \end{align*}
    so we only need to ensure 
    \begin{align*}
        C^-  \max \{ t^-_k (2\epsilon), (t^-_k (2\epsilon))^{5/2} \} \leq \epsilon.
    \end{align*}
    Rearranging this formula and plugging in the expression for $t^-_k(2\epsilon)$, we see that our condition above is ensured if
    \begin{align*}
        \Vert f^*_k \Vert_{L^2(p+q)}^2 \geq \max_{a \in \{1, 5/2\} } \frac{2 \epsilon \exp\Big( (\epsilon/C^-)^{1/a} / \lambda_k \Big) }{ \exp\Big( (\epsilon/C^-)^{1/a} / \lambda_k \Big) - 1 },
    \end{align*}
    which is our assumption.

    \textbf{Separation of null and assumption $\pif = f^*_k$ times:}  Finally, we want to ensure that the time needed  $t^+(\epsilon) - t^-(\epsilon) \geq \gamma > 0$ for some.  Noting the lower and upper bounds on $t^+(\epsilon)$ and $t^-(\epsilon)$, respectively, we find that our condition will be satisfied if
    \begin{align*}
         t^+(\epsilon) - t^-( \epsilon) &\geq \frac{\epsilon}{C^+} - \lambda_k \log\bigg( \frac{ \Vert f^*_k \Vert_{L^2(p+q)}^2 }{\Vert f^*_k \Vert_{L^2(p+q)}^2 - 2\epsilon } \bigg) \geq \gamma.
    \end{align*}
    Rewriting this inequality, we see that it is satisfied when
    \begin{align*}
        \Vert f^*_k \Vert_{L^2(p+q)}^2 \geq \frac{2 \epsilon \exp\Big( (\epsilon/C^+ - \gamma) / \lambda_k \Big) }{ \exp\Big( (\epsilon/C^+ - \gamma ) / \lambda_k \Big) - 1 }.
    \end{align*}
    As this is an assumption, we see that we are done.
\end{proof}

\begin{proof}[Proof of \Cref{cor:genMomPower_uhat}]\label{pf:genMomPower_uhat}
    We need to show that indeed $\what{T}_{test,k}(t) > \epsilon$ and $\what{T}_{test, null}(t) < \epsilon$ with high probability.  First we deal with the $\pif = f^*_k$ hypothesis and then with the null.
    
    \textbf{Assumption $\pif = f^*_k$:}  We want to make sure that $\what{T}_{test,k}(t) > \epsilon$.  With minor modifications to the proof of \Cref{thm:uhat_H1}, we see that
    \begin{align*}
        \vert \what{T}_{test,k}(t) \vert &\geq 2\epsilon - \delta_{train}(t) - \widetilde{C} t ,
    \end{align*}
    then to ensure $\what{T}_{test,k}(t) > \epsilon$, we need 
    \begin{align*}
        \epsilon &> \delta_{train}(t) + \widetilde{C} t  \geq \delta_{train}(t_1^*(2\epsilon)) + \widetilde{C} t_1^*(2\epsilon).
    \end{align*}
    Recalling from the proof of \Cref{cor:genMomRes_uhat}, we know that
    \begin{align*}
        C^-_{test} \min\{ t, t^{5/2} \} \leq \delta_{test} (t) \leq C^-_{test} \max\{ t, t^{5/2} \}.
    \end{align*}
    Thus, we have the result if we can show
    \begin{align*}
        (C_{test}^- + \widetilde{C})\max\{ t, t^{5/2} \} < \epsilon,
    \end{align*}
    which occurs if
    \begin{align*}
        t < \min_{a \in \{1, 5/2\} } \bigg( \frac{\epsilon}{ C_{test}^- + \widetilde{C} } \bigg)^{1/a}.
    \end{align*}
    We still want to make sure that the NTK dynamics can detect $2\epsilon$ with such a $t$.  This is enforced if
    \begin{align*}
        (C_{test}^- + \widetilde{C})\max\{ t_1^*(2\epsilon), {t_1^*(2\epsilon)}^{5/2} \} < \epsilon.
    \end{align*}
    Assuming the specific assumption that $f^*$ nontrivially projects only on the first $k$ eigenfunctions of $K_0$ so that $\pif = f_k^*$, notice that
    \begin{align*}
        t_1^*(2\epsilon) &= \min_{S \in \mathcal{S}_1(\epsilon)} \lambda_{\min}(S) \log\bigg( \frac{ \Vert f^*_k \Vert_{S}^2 }{\Vert f^*_k \Vert_{S}^2 - 2\epsilon } \bigg) \leq \lambda_k \log\bigg( \frac{ \Vert f^*_k \Vert_{L^2(p+q)}^2 }{\Vert f^*_k \Vert_{L^2(p+q)}^2 - 2\epsilon } \bigg) := t^-_k (2\epsilon).
    \end{align*}
    If we enforce $(C_{test}^- + \widetilde{C})\max\{ t_k^-(2\epsilon), {t_k^-(2\epsilon)}^{5/2} \} < \epsilon$, we have the result, but rearranging this expression gives us exactly our assumption on the norm $\Vert f_k^* \Vert_{L^2(p+q)}^2$.  So we know that $\what{T}_{test,k}(t) > \epsilon$.

    \textbf{Null Hypothesis Analysis:}  Note that similar to the case with the proof of \Cref{cor:genMomRes_uhat}, we have
    \begin{align*}
        \vert \what{T}_{test,null}(t) \vert \leq \vert \overline{T}_{test, null}(t) \vert + \vert \what{T}_{test, null}(t) - \overline{T}_{null}(t) \vert .
    \end{align*}
    Notice that $\overline{T}_{null}(t) = 0$ for all $t$ if $f^* = 0$ (i.e. the null hypothesis holds).  This means that the only term we care about is
    \begin{align*}
        \vert \what{T}_{test, null}(t) - \overline{T}_{null}(t) \vert \leq \delta_{test}(t) = C(L_1, A, m_p, m_q) t + C_2 t^{3/2} + C_3 t^2 + C_4 t^{5/2}.
    \end{align*}
    Note, however, that the constants $C_2 = C_3 = C_4 = 0$ since $f^* = 0$ as was the case in the proof of \Cref{cor:genMomRes_uhat}.  This means that $\what{T}_{test,null}(t) \leq \epsilon$ if 
    \begin{align*}
        t \leq \frac{\epsilon}{C(L_1, A, m_p, m_q)}.
    \end{align*}
    Notice, however, that $C_{test}^+ = C(L_1, A, m_p, m_q) \leq C_{test}^-$ for the assumption $\pif = f_k$.  This means that 
    \begin{align*}
        \min_{a \in \{1, 5/2\} } \bigg( \frac{\epsilon}{ C_{test}^- + \widetilde{C} } \bigg)^{1/a} < \frac{\epsilon}{C(L_1, A, m_p, m_q)}
    \end{align*}
    and so under our assumptions, we know that $\what{T}_{test,null}(t) < \epsilon$ with probability $\geq 1 - ( n_p^{-A} + n_q^{-A} + m_p^{-A} + m_q^{-A})$.

    \textbf{Test Power Analysis:}  For this analysis assume that 
    \begin{align*}
        \delta = n_p^{-A} + n_q^{-A} + m_p^{-A} + m_q^{-A}.
    \end{align*}
    Now consider our two-sample level $0 < \alpha < 1$ (typically $\alpha = 0.05$) with $\alpha \geq \delta$ and our test threshold $\tau$, where $\tau$ comes from
    \begin{align*}
        \P \big( \what{T}_{test, null}(t) > \tau \vert H_0 \big) =  \alpha > \P \big( \what{T}_{test, null}(t) > \epsilon \geq \tau) \geq \delta.
    \end{align*}
    Then notice that
    \begin{align*}
        \P \Big( \what{T}_{test} (t) > \tau \big\vert \pif = f_k^* \Big) &\geq \P\Big( \what{T}_{test} (t) > \epsilon \big\vert \pif = f_k^* \Big) \geq 1 - \delta.
    \end{align*}
\end{proof}

\section{Helper Lemmas}

\begin{lemma}\label{monotoneLemma}
    Let $a(x,t), b(x,t) : \R^d \times [0,\infty) \to \R$ be differentiable functions in $t$ and let $d\what{p}(x)$ and $d\what{q}(x)$ be discrete probability measures supported only on a finite number of Dirac masses. Then
    \begin{align*}
        g(t) = \int_{\R^d} \vert a(x,t) \vert^2 d\what{p}(x) + \int_{\R^d} \vert b(x,t) \vert^2 d\what{q}(x)
    \end{align*}
    is decreasing if and only if
    \begin{align*}
        h(t) = \int_{\R^d} \vert a(x,t) \vert d\what{p}(x) + \int_{\R^d} \vert b(x,t) \vert d\what{q}(x)
    \end{align*}
    is decreasing.
\end{lemma}
\begin{proof}
    We will take the derivatives of both $g(t)$ and $h(t)$ with respect to time and compare them, but we will restrict the integrals to $\text{supp}(\what{p})_+ = \{ x \in \text{supp}(\what{p}) :  \vert a(x,t) \vert > 0 \}$ and $\text{supp}(\what{q})_+ = \{ x \in \text{supp}(\what{q}) :  \vert b(x,t) \vert > 0 \}$.  In particular, consider
    \begin{align*}
        \frac{d}{dt} g(t) &= \frac{d}{dt} \Bigg(  \int_{\R^d} \vert a(x,t) \vert^2 d\what{p}(x) + \int_{\R^d } \vert b(x,t) \vert^2 d\what{q}(x) \Bigg) \\
        &= \int_{\text{supp}(\what{p})_+} \partial_t \vert a(x,t) \vert^2 d\what{p}(x) + \int_{\text{supp}(\what{q})_+} \partial_t \vert b(x,t) \vert^2 d\what{q}(x) \\
        &= \int_{\text{supp}(\what{p})_+} 2 \vert a(x,t) \vert \text{sgn}( a(x,t)) \partial_t a(x,t)  d\what{p}(x) \\
        &+ \int_{\text{supp}(\what{q})_+} 2 \vert b(x,t) \vert \text{sgn}( b(x,t)) \partial_t b(x,t) d\what{q}(x) .
    \end{align*}
    For $h(t)$, we get
    \begin{align*}
        \frac{d}{dt}h(t) &= \frac{d}{dt} \Bigg(  \int_{\R^d} \vert a(x,t) \vert d\what{p}(x) + \int_{\R^d } \vert b(x,t) \vert d\what{q}(x) \Bigg) \\
        &= \int_{\text{supp}(\what{p})_+} \partial_t \vert a(x,t) \vert d\what{p}(x) + \int_{\text{supp}(\what{q})_+} \partial_t \vert b(x,t) \vert d\what{q}(x) \\
        &= \int_{\text{supp}(\what{p})_+} \text{sgn}( a(x,t)) \partial_t a(x,t)  d\what{p}(x) + \int_{\text{supp}(\what{q})_+} \text{sgn}( b(x,t)) \partial_t b(x,t) d\what{q}(x) .
    \end{align*}
    Because we are using points only in $\text{supp}(\what{p})_+$ and $\text{supp}(\what{q})_+$ and since the supports of $d\what{p}$ and $d\what{q}$ are discrete measures, we can define
    \begin{align*}
        C(t) &= 2 \max\bigg\{ \max_{x \in \text{supp}(\what{p})_+} \vert a(x,t) \vert, \max_{x \in \text{supp}(\what{q})_+} \vert b(x,t) \vert \bigg\} > 0 \\
        c(t) &= 2 \min\bigg\{ \min_{x \in \text{supp}(\what{p})_+} \vert a(x,t) \vert, \min_{x \in \text{supp}(\what{q})_+} \vert b(x,t) \vert \bigg\} > 0.
    \end{align*}
    Notice that the assumption that $c(t) > 0$ heavily depends on that the measures $d\what{p}$ and $d\what{q}$ are composed of a finite number of Dirac measures. Now, notice that
    \begin{align*}
        \frac{d}{dt} g(t) &= \int_{\text{supp}(\what{p})_+} \underbrace{2 \vert a(x,t) \vert}_{\substack{ \leq C(t) \\ \geq c(t) } } \text{sgn}( a(x,t)) \partial_t a(x,t)  d\what{p}(x) \\
        &+ \int_{\text{supp}(\what{q})_+} \underbrace{ 2 \vert b(x,t) \vert}_{ \substack{ \leq C(t) \\ \geq c(t) } } \text{sgn}( b(x,t)) \partial_t b(x,t) d\what{q}(x) \\
        \implies c(t) \frac{d}{dt} h(t) &\leq \frac{d}{dt} g(t) \leq C(t) \frac{d}{dt} h(t).
    \end{align*}
    Since $\frac{d}{dt} h(t)$ and $\frac{d}{dt} g(t)$ are off by positive factors, we see that if one is decreasing, the other must also be decreasing.  This proves the lemma.
\end{proof}

\subsection{Concentration Inequalities}

\begin{lemma}[Hoeffding's Inequality]\label{hoeffding}
    Suppose $\{ X_i \}_{i=1}^n$ are independent random variables with $\vert X_i \vert \leq L$, then  for all $t \geq 0$
    \begin{align*}
        \text{Pr} \bigg\{ \Big\vert \frac{1}{n} \sum_{i=1}^n ( X_i - \E [X_i]  ) \Big\vert \geq t \bigg\} \leq 2 \exp\bigg( - \frac{n t^2 }{2 L^2  } \bigg).
    \end{align*}
\end{lemma}

\begin{lemma}[Hoeffding's Subgaussian Inequality]\label{hoeffdingSubGauss}
    Suppose $\{ X_i \}_{i=1}^n$ are independent $\sigma_i$-subgaussian random variables with $X_i$ having mean $\mu$, then  for all $t \geq 0$
    \begin{align*}
        \text{Pr} \bigg\{ \Big\vert \frac{1}{n} \sum_{i=1}^n ( X_i - \mu_i ) \Big\vert \geq t \bigg\} \leq 2 \exp\bigg( - \frac{t}{2 \sum_{i=1}^n \sigma_i^2 } \bigg).
    \end{align*}
\end{lemma}

\begin{lemma}[Matrix Bernstein]\label{matBern}
    Let $X_i$ be a sequence of $n$ independent, random, real-valued matrices of size $d_1$-by-$d_2$.  Assume that $\E X_i = 0$ and $\Vert X_i \Vert \leq L$ for each $i$ and $\nu > 0$ be such that
    \begin{align*}
        \Vert \frac{1}{n} \sum_{i=1}^n \E X_i X_i^\top \Vert, \Vert \frac{1}{n} \sum_{i=1}^n \E X_i^\top X_i \Vert \leq \nu.
    \end{align*}
    Then for any $t \geq 0$,
    \begin{align*}
        \text{Pr}\Big[ \Vert \frac{1}{n} \sum_{i=1}^n X_i \Vert \geq t \Big] \leq (d_1 + d_2) \exp\Big\{   - \frac{n t^2 }{  2 (\nu + Lt/3) } \Big\}.
    \end{align*}
\end{lemma}

\end{document}